\documentclass[11pt,3p,review,authoryear]{arXiv}

\journal{International Journal of Forecasting}
\biboptions{longnamesfirst}

\usepackage{bigints}
\usepackage{enumitem}
\usepackage{graphicx}
\usepackage{adjustbox}
\usepackage{scalerel}
\usepackage{url}
\usepackage{mathrsfs}
\usepackage{dsfont}
\usepackage{blkarray}
\usepackage{algorithm}
\usepackage[noend]{algorithmic}
\usepackage{array}
\usepackage{tikz,tkz-tab} 
\usetikzlibrary{arrows}
\usepackage{qtree}
\usepackage{amssymb}
\usepackage{amsthm}
\usepackage{float}
\usepackage{gensymb}
\usepackage[nameinlink, noabbrev]{cleveref}
\usepackage{booktabs}
\usepackage{color, colortbl}
\definecolor{Gray}{gray}{0.9}
\definecolor{Gray2}{gray}{0.7}

\usepackage{arydshln}

\newtheorem{theorem}{Theorem}
\newtheorem{proposition}{Proposition}
\newtheorem{lemma}{Lemma}

\newtheorem{notation}{Notation}
\newtheorem{hyp}{Assumption}

\newtheorem{rem}{Remark}
\expandafter\let\expandafter\oldproof\csname\string\proof\endcsname
\let\oldendproof\endproof
\renewenvironment{proof}[1][\proofname]{%
	\oldproof[\upshape \bfseries #1]%
}{\oldendproof}

\newcommand{\argmin}{\mathop{\mathrm{arg\,min}}}
\newcommand{\PreserveBackslash}[1]{\let\temp=\\#1\let\\=\temp}
\newcolumntype{C}[1]{>{\PreserveBackslash\centering}p{#1}}
\newcolumntype{R}[1]{>{\PreserveBackslash\raggedleft}p{#1}}
\newcolumntype{L}[1]{>{\PreserveBackslash\raggedright}p{#1}}

\floatstyle{ruled}
\newfloat{protocol}{htbp}{lok}
\floatname{protocol}{Meta-algorithm}
\crefname{protocol}{protocol}{meta-algorithms}
\Crefname{protocol}{Pernel}{Meta-algorithm}

\newtheorem{exam}{Example}

\newcounter{examplecounter}
\newenvironment{example}[1]{\refstepcounter{examplecounter} \textbf{Example \arabic{examplecounter}: {#1}.}~~}{}

\renewcommand{\leq}{\leqslant}
\renewcommand{\geq}{\geqslant}
\newcommand{\transp}{\mbox{\tiny \textup{T}}}
\newcommand{\inv}{\!\raisebox{2pt}{\tiny \textup{-1}}}
\renewcommand{\epsilon}{\varepsilon}

\newcommand{\defeq}{\stackrel{\mbox{\tiny \rm def}}{=}}
\renewcommand{\hat}{\widehat}
\newcommand{\norm}[1]{\Arrowvert #1 \Arrowvert}

\renewcommand{\leq}{\leqslant}
\renewcommand{\geq}{\geqslant}

\renewcommand{\phi}{\varphi}
\renewcommand{\epsilon}{\varepsilon}
\renewcommand{\hat}{\widehat}
\renewcommand{\tilde}{\widetilde}
\newcommand{\K}{\mathbf{\underline{K}}}
\newcommand{\E}{\mathbf{\underline{E}}}
\newcommand{\U}{\mathbf{\underline{U}}}

\newcommand{\A}{\mathbf{\underline{A}}}
\newcommand{\Y}{\mathbf{\underline{Y}}}
\newcommand{\W}{\mathbf{\underline{W}}}

\renewcommand{\H}{\mathbf{\underline{H}}}

\newcommand{\y}{\mathbf{y}}
\newcommand{\hy}{\hat{\mathbf{y}}}
\newcommand{\ty}{\tilde{\mathbf{y}}}
\renewcommand{\u}{\mathbf{u}}
\newcommand{\hu}{\hat{\mathbf{u}}}
\newcommand{\x}{\mathbf{x}}

\newcommand{\w}{\mathbf{w}}
\begin{document}

\begin{frontmatter}

\title{Online Hierarchical Forecasting for Power Consumption Data}

 \author[mb]{Margaux Br\'eg\`ere \fnref{mh}\corref{cor}}
\address[mb]{EDF R\&D, Palaiseau, France. \\
INRIA - D\'epartement d'Informatique de l'\'Ecole Normale Sup\'erieure, PSL Research University, Paris, France.}
 \cortext[cor]{Corresponding author}
\ead{margaux.bregere@math.u-psud.fr}
 \author[mh]{Malo Huard}
 \address[mh]{Laboratoire de math\'ematiques d'Orsay, CNRS, Universit\'e Paris-Saclay, Orsay, France.}

\begin{abstract}
We study the forecasting of the power consumptions of a population of households and of subpopulations thereof. These subpopulations are built according to location, to exogenous information and/or to profiles we determined from historical households consumption time series.  
Thus, we aim to forecast the electricity consumption time series at several levels of households aggregation. These time series are linked through some summation constraints which induce a hierarchy. Our approach consists in three steps: feature generation, aggregation and projection.
Firstly (feature generation step), we build, for each considering group for households, a benchmark forecast (called features), using random forests or generalized additive models. Secondly (aggregation step), aggregation algorithms, run in parallel, aggregate these forecasts and provide new predictions.  Finally (projection step), we use the summation constraints induced by the time series underlying hierarchy to re-conciliate the forecasts by projecting them in a well-chosen linear subspace. We provide some theoretical guaranties on the average prediction error of this methodology, through the minimization of a quantity called regret. We also test our approach on households power consumption data collected in Great Britain by multiple energy providers in the ‘\textit{Energy Demand Research Project}’ context. We build and compare various population segmentations for the evaluation of our approach performance.
\end{abstract}

\begin{keyword}
Adjusting forecasts \sep Combining forecasts  \sep Demand forecasting \sep Electricity \sep Time series
\end{keyword}

\end{frontmatter}

\section{Introduction}

\paragraph{Motivation: Electricity Forecasting} New opportunities come with the recent deployment of smart grids and the installation of meters: they record consumption quasi instantaneously in households. From these records, time series of demand are obtained at various levels of aggregation, such as consumption profiles and regions. 
For privacy reasons, household records may not be used directly. Moreover, consumption at individual level is erratic and unpredictable. 
This is why we focus on household aggregations. For demand management, it is useful to predict the global consumption.
 Furthermore, to dispatch correctly the electricity into the grid, forecasting demand at a regional level is also an important goal. 
Finally, a good estimation of the consumption of some groups of consumers (with the same profile) may be helpful for the electricity provider which may adapt its offer to perform effective demand side management. 
Thus, forecasts at various aggregated levels (entire population, geographical areas, groups of same consumption profiles) are useful for an efficient management of consumption.
In this work, we first build  at each aggregation level, and independently, benchmark forecasts (called features) using random forests or generalized additive models.
Noticing that these time series may be correlated (the consumption of a given region may be close to the one of a neighboring region)
and connected to each other through summation constraints (the global consumption is the sum of the region consumptions, e.g.), 
the problem considered falls under the umbrella of hierarchical time series forecasting.
Using these hierarchical relationships may improve the benchmark forecasts that were generated. Our approach consists in combining two methods: feature aggregation and and projection in a constrained space. Our aim is to improve forecasts both at the global and at the local levels.  

\paragraph{Literature Discussion for Hierarchical Forecasting}
Traditionally two types of methods have been used for hierarchical forecasting: bottom-up and top-down approaches. 
In the bottom-up approaches (see \citealp{Dunn1976}) forecasts are constructed for lower-level quantities and are then summed up to obtain forecasts at the upper levels. In contrast, top-down approaches (see  \citealp{Gross1990}) work by forecasting aggregated quantities and then by determining dis-aggregate proportions to compute lower level predictions.  \citet{ShliferAggregationProrationForecasting1979} compare these two families of methods and conclude that bottom-up approaches work better. 
Recently, it has indeed proven successful for load forecasting to improve the global consumption prediction error (see among others \citealp{auder2018scalable}).  
Other approaches (neither bottom-up nor top-down) were recently introduced, for example \citet{Hyndman2011} forecast all nodes in the hierarchy and reconcile them by orthogonal projection. 
Moreover, \citet{van2015game} introduce a game-theoretically optimal reconciliation method to improve a given set of forecasts. Firstly, one comes up with some forecasts for the time series without worrying about hierarchical constraints and then a reconciliation procedure is used to make the forecasts aggregate consistent. This generalizes the previous orthogonal projection to other possible projections in the constrained space (which ensures that the forecasts satisfy the hierarchy).
Finally, if we restrict here to mean forecasting, some follow-up works from \citet{pmlr-v70-taieb17a} allow to make probabilistic forecasting in this context of hierarchical prediction.\\

\paragraph{Literature Discussion for Aggregation Methods} Aggregation methods (also called ensemble methods) for individual sequences forecasting originate from theoretical works by \citet{Vovk:1990:AS:92571.92672}, \citet{Cover1991} and \citet{littlestone1994weighted}; their distinguishing feature with respect to classical ensemble methods is that they do not rely on any stochastic modeling of the observations and thus, are able to combine forecasts independently of their generating process. They have been  proved to be very effective to predict time series (see for instance \citealp{mallet2009ozone} and \citealp{devaine2013forecasting}) and those methods were used to win forecasting competitions (see \citealp{GaillardAdditivemodelsrobust2016a}). This aggregation approach has recently been extended to the hierarchical setting by \citet{goehry2019aggregation}; they used a bottom-up forecasting approach which consists in aggregating the consumption forecasts of small customers clusters. \\

In this article we combine the reconciliation approach based on orthogonal projection with various aggregation algorithms to provide new methods to which we were able to prove strong theoretical guaranties. 
We then illustrate the proposed methods using smart meter data collected in Great Britain by multiple energy providers  (see \citealp{schellong2011energy} and \citealp{AECOMEnergyDemandResearch2018}).
 ‘\textit{Energy Demand Research Project}’ data gathers multiple households power consumption data. 
We compare various population segmentations and evaluate the performance of four strategies for the forecasting of the electricity consumption time series at the several aggregation levels: features, aggregated features, projected features and finally aggregated and projected features.

\paragraph{Notation} Without further indications, 
 $\norm{\mathbf{x}}$ denotes the Euclidean norm of a vector $\mathbf{x}$.
For the other norms, there will be a subscript: e.g., the Frobenius norm of $\mathbf{x}$ is denoted by $\norm{\mathbf{x}}_F$.
Moreover, vectors will be in bold type and
unless stated otherwise, they are column vectors, 
while matrices will be  in bold underlined.
We denote the inner product of two vectors $\x$ and $\y$ of the same size by $\x \cdot \y=\x^{\transp}\y$.
Finally, 
the cardinal of a finite set $\mathcal{D}$ is denoted by $|\mathcal{D}|$.

\section{Methodology}

\label{sec:methodo}

We consider a set of time series  $\big\{(y_t^\gamma)_{t>0}, \gamma \in \Gamma\big\}$  connected to each other by some summation constraints: a few of them are equal to the sum of several others -- see further for a definition of $\Gamma$.
To forecast these time series, a set of features is generated. 
At any time step $t$, we want to forecast the vector of the values of the $|\Gamma|$ times series at $t$, denoted by $\y_t \defeq (y^\gamma_t)_{\gamma \in \Gamma}$.
We propose a three-step method to obtain relevant forecasts from these features.

\subsection{Modeling of the Hierarchical Relationships}

The relationships between the time series induce a hierarchy which should be exploited to improve forecasts.
These summation constraints may be represented by one or more trees, the value at each node being equal to the sum of the ones at its leaves. Let us denote by $\Gamma$ the set of the tree's nodes and $|\Gamma|$ its cardinal. 
There are as many summation constraints as there are nodes with leaves. 
Subsequently, we will introduce a matrix  $\K$ to encode these relationships. 
Each line of $\K$ is related to one of the summation constraints with $-1$ at the associated node and $1$ at its leaves.
Thus, for any instance $t$, the vector of the values of the $|\Gamma|$ times series at $t$, denoted by $\y_t$, is in the kernel of $\K$. 
Details on and examples of $\K$ are provided below.
Example~\ref{ex:simplehierarchy} treats a single summation constraint. 
Examples ~\ref{ex:twoleveltree} and~\ref{ex:twotrees} present more complex relationships between the time series, considering a hierarchy with two levels and two different partitions of the same time series, respectively. Finally, Example~\ref{ex:doublehierarchy} combines the two previous cases.
In our experiments of Section~\ref{sec:experiments}, the underlying hierarchies will be of the form of the ones of Examples~\ref{ex:simplehierarchy} and~\ref{ex:doublehierarchy}.

\begin{example}{Two-level Hierarchy}
\label{ex:simplehierarchy}
The simplest approach consists in considering a single equation connecting the time series. 
Here, $y^{\mathrm{\textsc tot}}$ stands for the one which is the sum of the $N$ others which are denoted by $y^1,\dots, y^N$.
The underlying hierarchy is represented in Figure~\ref{fig:ex1} by a tree with a single root directly connected to $N$ leaves.
For any instance $t$, the time series satisfy
$y_t^{\mathrm{\textsc tot}}=y_t^1+y_t^2+\dots + y_t^N$ and the vector $\y_t=\big(y_t^{\mathrm{\textsc tot}},y_t^1,\dots,y_t^N\big)^{\transp}$ respects the hierarchy if and only if $\K\y_t=0$ with $\K=\big(-1,1, 1,\dots,1\big)$.
\begin{figure}[h]
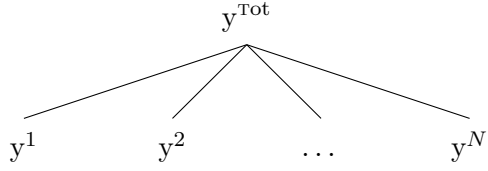

\label{fig:ex1}
\[
\Tree [.y^{\mathrm{\textsc tot}} y^1 y^2 {\dots}  !\qsetw{3cm} y^N ]
\]
\caption{Representation of a two-level hierarchy.}
\end{figure}
In Section~\ref{sec:experiments}, we consider the power consumption of a population of households which are distributed in $N$ regions. This setting will correspond to the present example.
\end{example}

\begin{example}{Three-level Hierarchy} 
\label{ex:twoleveltree}
\begin{figure}[h]
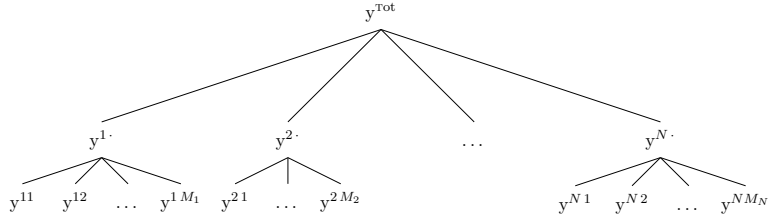

\label{fig:ex2}
\[
\begin{adjustbox}{width=10cm,center}
\Tree[!\qsetw{1cm} [.y^{1\,\cdot} y^{11} y^{12} {\dots} y^{1\,M_1}   ][.y^{2\,\cdot} y^{2\,1}   {\dots} y^{2\,M_2} ][.{\dots} ] [.y^{N\,\cdot} y^{N\,1}  y^{N\,2} {\dots} y^{NM_N}  ]].y^{\mathrm{\textsc tot}}
\end{adjustbox}
\]
\caption{Representation of a three-level hierarchy.}
\end{figure}
A few leaves of the tree of Example~\ref{ex:simplehierarchy} may be broken down into new time series and so on.
Figure~\ref{fig:ex2} represents a complete three-level hierarchy (	although we could consider any multilevel hierarchy) leading to the following summation equations, for each instance $t$,
\begin{align}
y^{\mathrm{\textsc tot}}_t &= \, y_t^{1\,\cdot}+y_t^{2\,\cdot}+\dots + y_t^{N\,\cdot}  \tag{1}\\ 
  y_t^{i \,\cdot} \, \, \, &=\, y_t^{i \,1}+y_t^{i \, 2}+\dots + y_t^{i\, M_i}, \quad \forall \, i=1,\dots, N. \tag{2i}
\end{align}
We order the time series in lexicographical order:
\[\y_t=(y_t^{\mathrm{\textsc tot}},y_t^{1 \, \cdot},y_t^{2 \, \cdot},\dots,y_t^{N \, \cdot} ,y_t^{11},y_t^{12},\dots,y_t^{1M_1}, y_t^{21},\dots,y_t^{NM_N} ),\] 
and define the constraint matrix $\K$ 
below; each line of $\K$ corresponds to one of the constraints mentioned above, either $(1)$ or one of the $N$ constraints $(2i)$ in a way that $\K\y_t=\mathbf{0}$ if and only if $\y_t$ respects the hierarchy.
\[
\K=\begin{blockarray}{rcccccc}
\begin{block}{(ccccc)cr}
 \text{-} 1 & \overbrace{\, 1 \quad \cdots \quad 1}^N   \vspace{-0.3cm} &&& & \leftarrow &\text{(1)}  \\
 & \text{-}1 \qquad \qquad \, \,   & \overbrace{1 \quad \cdots \quad 1}^{M_1} \vspace{-0.3cm}  \\
 & \ddots &  &\qquad  \ddots  \qquad &   \vspace{-0.3cm} & \leftarrow &\text{(2i)} \\
 & \qquad \qquad  \text{-}1 \quad  &&& \overbrace{1 \quad \cdots \quad 1}^{M_N}\\
\end{block}
\end{blockarray}
\]
This hierarchy corresponds to the concrete example above where, for $1 \leq n \leq N$, the $n$ region would be further divided into $M_1, \dots, M_n$ municipalities.
\end{example}

\begin{example}{Two Hierarchies of the Same Time Series}
\label{ex:twotrees} 
It is also possible to consider two partitions of the same time series $y_t^{\mathrm{\textsc tot}}$. 
For example, in our experiments of Section~\ref{sec:experiments}, in addition to the geographical clustering, we introduce a segmentation of the households based on their profiles.
Indeed, they are distributed in $N_1$ regions but also in $N_2$ groups depending on their consumption habits. 
These two different partitions induce the following two equations
 \begin{align*}
y_t^{\mathrm{\textsc tot}} &= \, y_t^{1\,\cdot}+y_t^{2\,\cdot}+\dots + y_t^{N_1\,\cdot}  \\
y_t^{\mathrm{\textsc tot}} &= \, y_t^{\cdot\,1}+y_t^{\cdot\,2}+\dots + y_t^{ \cdot \, N_2},
\end{align*}
and the two trees associated with these constraints which share the same root and are represented on Figure~\ref{fig:ex3}.
\begin{figure}[h]
\label{fig:ex3}
\[
\begin{tabular}{ccc}
\Tree [.y^{\mathrm{\textsc tot}} y^{1\,\cdot} y^{2\,\cdot} {\dots}  !\qsetw{3cm} y^{N_1\,\cdot} ]
&
=
\Tree [.y^{\mathrm{\textsc tot}} y^{\,  \cdot \, 1} y^{\, \cdot \, 2} {\dots}  !\qsetw{3cm} y^{\, \cdot \, N_2} ]
\end{tabular}%
\]
\caption{Representation of two two-level hierarchies.}
\end{figure}
For any instance $t$, the vector of times series $\y_t=( y_t^{\mathrm{\textsc tot}}, y_t^{1 \, \cdot},y_t^{2 \, \cdot},\dots, y_t^{N_1 \, \cdot} , y_t^{\cdot \, 1},y_t^{ \cdot \, 2},\dots, y_t^{\mathrm{\cdot \, \textsc{n}_2}}  )$ satisfies the above equations
 if and only if $\K\y_t=0$ with 
\[
\renewcommand\arraystretch{1.3}
\K=\begin{blockarray}{rccc}
\begin{block}{(cccc)} 
 \text{-} 1 & \overbrace{\, 1 \quad \cdots \quad 1}^{N_1}  & \vspace{-0.4cm} \\
 \text{-} 1 &             & \overbrace{\, 1 \quad \cdots \quad 1}^{N_2}  \vspace{0.1cm}\\
\end{block}
\end{blockarray} \, .
\]
The equality of the roots of the two trees is always satisfied in this model. Indeed there is a single time series $y_t^{\mathrm{\textsc tot}}$ to forecast and there are therefore only two summation constraints to take into account.
\end{example}

\begin{example} {Two Crossed Hierarchies}
\label{ex:doublehierarchy}
\begin{figure}[ht]
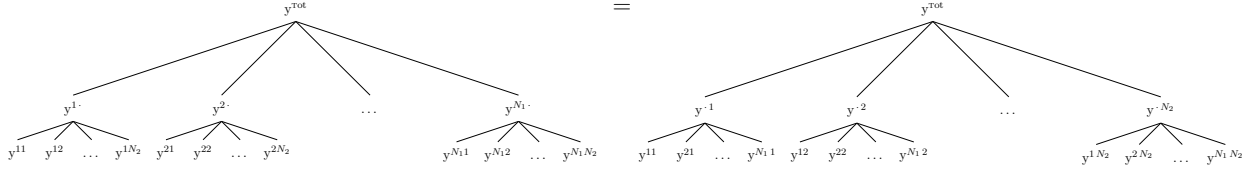

\label{fig:ex4}
\[
\begin{adjustbox}{width=\columnwidth,center}
\begin{tabular}{ccc}
\Tree[!\qsetw{1cm} [.y^{1\,\cdot} y^{11} y^{12} {\dots} y^{1N_2} ][.y^{2\,\cdot} y^{21}  y^{22} {\dots} y^{2N_2} ][.{\dots} ] [.y^{N_1\,\cdot} y^{N_11}  y^{N_12} {\dots} y^{N_1N_2} ]].y^{\mathrm{\textsc tot}}
&
\huge{=} \normalsize
\Tree[!\qsetw{1cm} [.y^{\, \cdot \,  1} y^{11} y^{21} {\dots} y^{N_1\,1}   ][.y^{\, \cdot \, 2} y^{12}  y^{22} {\dots} y^{N_1\,2} ][.{\dots} ] [.y^{\, \cdot \, N_2} y^{1\,N_2}  y^{2\,N_2} {\dots} y^{N_1\, N_2}  ]].y^{\mathrm{\textsc tot}}
\end{tabular}%
\end{adjustbox}
\]
\caption{Representation of two crossed hierarchies.}
\end{figure}
Considering two partitions, the time series can be represented with two three-level trees sharing the same root and leaves. 
Only the intermediate levels differs according to which partition is firstly taking into account.
The leaves of the trees form a $N_1 \times N_2$-matrix $\big(y^{ i\,j}\big)_{1\leq i\leq N_1, \, 1\leq j \leq N_2}$. 
An intermediate node of the first tree $y^{i \,\cdot}$ is the sum of the line $i$ while a node $y^{\cdot j}$ of the second tree is the sum of the column $j$. Whether we sum rows or columns first, the sum of all coefficients is $y_t^{\mathrm{\textsc tot}}$.
In the experiments of Section~\ref{sec:experiments}, one partition refers to a geographic distribution of the households while the other classifies them according to their consumption behaviours.
The first tree considers breaks down consumption firstly by the $N_1$ regions and then by the $N_2$ household profiles.
The second one divides the households according to their habits before splitting them geographically.
Both trees are represented in Figure~\ref{fig:ex4}. 
For any instance $t$, the time series satisfy the $2+N_1+N_2$ equations
\begin{align}
y_t^{\mathrm{\textsc tot}} &= \, y_t^{1\,\cdot}+y_t^{2\,\cdot}+\dots + y_t^{N_1\,\cdot}  \tag{1}\\
y_t^{i \,\cdot} \, \, \, &=\, y_t^{i \,1}+y_t^{i \, 2}+\dots + y_t^{i\, N_2}, \quad \forall \, i=1,\dots, N_1   \tag{2i} \\
y_t^{\mathrm{\textsc tot}} &=\, y_t^{\cdot \,1}+y_t^{\cdot \, 2}+\dots + y_t^{ \cdot \, N_2} \tag{3}\\
y_t^{\cdot \, j} \, \,\, &=\, y_t^{1 \,j}+y_t^{2\,j}+\dots + y_t^{N_1 \, j}, \quad \forall \, j=1,\dots, N_2. \,  \tag{4j}
\end{align}
Equations (1) and (3) refer to the first level of the trees while the $N_1+N_2$ Equations (2i) and (4j) refer to second levels.
At an instance $t$, by ordering  the time series in the vector as
\[\y_t=(y_t^{\mathrm{\textsc tot}}, y_t^{1 \,\cdot} ,\dots,y_t^{N_1 \,\cdot} , y_t^{\cdot \, 1} ,\dots,y_t^{\cdot \,N_1} , y_t^{1 \,1} , y_t^{1 \,2} ,\dots,y_t^{N_1N_2} ),\]
 it may be seen that they respect the hierarchy if and only if $\K\y_t=0$ with
\[
\K=\begin{blockarray}{rccccccc}
\begin{block}{(cccccc)cr } 
 \text{-} 1 & \overbrace{\, 1 \quad \cdots \quad 1}^{N_1}   &&&&         & \leftarrow &\text{(1)}  \vspace{-0.3cm} \\ 
      & \text{-}1 \qquad \qquad \quad && \overbrace{\, 1 \quad \cdots \quad 1}^{N_2} & &                                   &&   \\
  & \ddots &&  &\quad  \ddots  \qquad &                                                 & \leftarrow &\text{(2i)}\vspace{-0.3cm} \\
    & \qquad \quad \, \text{-}1 \quad  &&&& \overbrace{\, 1 \quad \cdots \quad 1}^{N_2}                                     &&\vspace{-0.3cm}\\
  \text{-} 1 && \overbrace{\, 1 \quad \cdots \quad 1}^{N_2}  &&&                                                                                                   & \leftarrow &\text{(3)} \\
&&   \text{-}1  \qquad \qquad \quad  & \, 1 \qquad \qquad  \quad  &   & \, 1 \qquad \qquad  \quad  &&\\
   && \ddots &  \ddots  & \quad  \cdots  \qquad  &  \ddots & \leftarrow &\text{(4j)}\\
 &&  \qquad \quad \, \, \text{-}1 \quad  &   \qquad \qquad 1 && \qquad \qquad 1\,\, &&\\
\end{block}
\end{blockarray} \,. 
\]
\end{example}

\subsection{A Three-step Forecast}
\label{subsec:threestepforecast}

Step $1$: For each node $\gamma \in \Gamma$, at each instance $t$, thanks to an historical data set of the $\gamma$ time series and to some exogenous variables proper to the node $\gamma$, a forecaster makes the prediction $x_t^\gamma$.
These $|\Gamma|$ benchmark forecasts are then collected into the feature vector $\x_t \defeq \big(x_t^\gamma\big)_{\gamma \in \Gamma}$.
We propose to use the knowledge of all of the features, namely the $|\Gamma|$  benchmark forecasts, and of the summation constraints to improve these $|\Gamma|$ predictions.
Step $2$: For each node $\gamma$ and each instance $t$, we form our prediction $\hat{y}_t^\gamma$ by linearly combining the components of the feature vector $\x_t$ thanks to a so-called aggregation algorithm (a copy $\mathcal{A}^\gamma$ of an aggregation algorithm $\mathcal{A}$ is run separately for each node $\gamma$). That is, we use all $|\Gamma|$ benchmark forecasts to predict $y_t^\gamma$, not only $x_t^\gamma$. We explain below why this is a good idea -- the main reason is given by correlations between time series.
The forecasts thus obtained are then gathered into a vector $\hy_t \defeq \big(\hat{y}_t^\gamma\big)_{\gamma \in \Gamma}$.
Step $3$: Finally, a re-conciliation step will update the forecast vector so that it is in the kernel $\K$.
Let us denote by $\ty_t$ the final vector of forecasts. We detail below each step of our procedure.
\small
\begin{center}
\tikzstyle{init} = [pin edge={to-,thin,black}]
\begin{tikzpicture}
    \node[draw,  pin={[init]above: Forecasters},text centered, minimum height=3em] (a) {Generation of features};
    \node[draw, pin={[init]above:$\mathcal{A}$},text centered,minimum height=3em] (c) [right of=a,node distance=4.8cm] {$|\Gamma|$ aggregations in parallel};
   \node[draw, pin={[init]above:$\K$},text centered, minimum height=3em] (d) [right of=c,node distance=3.8cm] {Projection};
    \node (b) [left of=a,node distance=5.3cm, coordinate] {a};
    \node [coordinate] (end) [right of=d, node distance=1.7cm]{};
    \path[->] (b) edge node [above] {historical data} (a);
    \path[->] (b) edge node [below] { exogeneous variables} (a);
    \path[->] (a) edge node [above]  {$\x_t$} (c);
    \path[->] (a) edge node [above]  { } (c);
    \path[->] (c) edge node [above] {$\hy_t$} (d);
    \path[->] (c) edge node [above] { } (d);
    \draw[->] (d) edge node [above] {$\ty_t$} (end) ;
     \draw[->] (d) edge node [above] { } (end) ;
\end{tikzpicture}
\end{center}
\normalsize

\paragraph{First Step: Generation of features}
At a fixed node  $\gamma \in \Gamma$, for any instance $t$, a forecasting method, which may depend on $\gamma$, predicts $x_t^\gamma$ with the historical data and the exogenous variables of the node $\gamma$.
The forecasting methods we use in the experiments of Section~\ref{sec:experiments} are described in Section~\ref{sec:features} and include non linear sequential ridge regression, fully adaptive Bernstein online aggregation and polynomially weighted average forecaster with multiple learning rates.
These benchmark forecasts are henceforth called features and are gathered in $\x_t=\big(x_t^\gamma \big)_{\gamma \in \Gamma}$.  
This feature vector is used in the aggregation step that comes next to predict again each time series; we discuss below and in Subsection~\ref{subsec:samefeatures} why we do so (the main reasons being that it is a good idea because of the correlations between the times series and also because it eases the description of our method).
We focus here on $|\Gamma|$ benchmark forecasts -- one for each of the nodes; however, we could also have considered several predictions per nodes. 

\paragraph{Second Step: Aggregation}
The above features are generated independently with different exogenous variables and possibly different methods.
Yet, the observations $ \big(y^\gamma_t\big)_{\gamma \in \Gamma}$ may be correlated.
For example, considering load forecasting, the consumptions associated with two nearby regions can be strongly similar.
Furthermore, the observations are related though the summations constraints (although we disregard these equations here). 
This is why linearly combining the features may refine some forecasts -- this is exactly what this step does. Formally, an aggregation algorithm outputs at each round a vector of weights $\hu_t^\gamma$ and returns the forecast $\hat{y}^\gamma_t\defeq \hu_t^\gamma \cdot \x_t$. It does so based on the information available, that is, the feature vector $\x_t$ and past data. We consider an aggregation rule  $\mathcal{A}$ (see Section~\ref{sec:aggregalgo}) and form a copy $\mathcal{A}^\gamma$  for each node $\gamma$, which we feed with an input parameter vector $\mathbf{s}_0^\gamma$. 
These predictions are then gathered into the vector $\hy_t=(y^\gamma_t)_{\gamma \in \Gamma}$.
This algorithm aims for the best linear combination of features and there are theoretical performance guaranties associated with these aggregation algorithms, see Section~\ref{sec:aggregalgo} for details. \\

Instead of this approach based on benchmark forecasting and aggregation node by node, we could have considered a meta-model to directly predict the time series vector $\big(y_t^\gamma\big)_{\gamma\in \Gamma}$ at each instant $t$ (with a common forecaster and therefore without any aggregation step).
Once this global forecast would have been obtained, we would have gone straight to the projection stage.
In such a model, the number of variables to be taken into account (the historical data of the time series but also the exogenous variables specific to each node) would have been considerable and getting relevant forecasts would have not been an easy task.
But actually, a practical choice motivated our method for the most.
Indeed, the forecasters may be black boxes proper to each node and the exogenous variables of a node $\gamma$ may be unknown at a node $\gamma'$. 
In our experiments, we followed this three-step approach. However, our method totally operates if, for each node $\gamma$ and at each instance $t$, an external expert provides the forecast $x_t^\gamma$. 
How these features have been obtained is no longer an issue and the aim is to improve these benchmark forecasts with aggregation and reconciliation steps.
Thus, at each instance $t$, only the features are reveal at time $t$ and by skipping the generation of features step, we go straight to the aggregation step.

\paragraph{Third Step: Projection}
As the $|\Gamma|$ executions of Algorithm $\mathcal{A}$ are run in parallel  and independently, the obtained forecast vector $\hy_t$ does not necessary respect hierarchical constraints.
To correct that, we consider the orthogonal projection of $\hy_t$ onto the kernel of $\K$, which we denote by $ \Pi_\K ( \hy_t)$. 
This updated forecast $\ty_t \defeq  \Pi_\K ( \hy_t) $ fulfills the hierarchical constraints. \\

To sum up, at each  instant $t$, we first generate benchmark forecasts -- also called features --  $\x_t$. These predictions are then aggregated to form a new vector of forecast $\hy_t$, which is itself updated in the projection step in $\ty_t$. This procedure is stated in Meta-algorithm~\ref{prot:1}. 
Moreover, we can also directly project the features, skipping the aggregation step; this leads to the forecasts $\Pi_{\K}(\x_t)$. Thus, we get four forecasts ($\x_t$, $\Pi_{\K}(\x_t)$, $\hy_t$ and $\ty_t$) for each node and each instant. The performance of our strategies is measured in mean squared error. In Section~\ref{sec:experiments}, we compare these four methods in the scope of power consumption forecasting.

\begin{protocol}[h]
\caption{\label{prot:1} Aggregation and projection of features with summation constraints}
\begin{algorithmic}
\STATE \textbf{Input}
\STATE \quad Set of nodes $\Gamma$ and constraint matrix $\K$ 
\STATE \quad Feature generation technique, see Section~\ref{sec:features}
\STATE \quad Aggregation algorithm $\mathcal{A}$ taking parameter vector $\mathbf{s}_0$, see Section~\ref{sec:aggregalgo}
\STATE Compute the orthogonal projection matrix $\Pi_\K=  \big(I_{|\Gamma|}-\K^{\transp}(\K\K^{\transp})^{\inv}\K\big)$
 \FOR{$\gamma \in \Gamma$}
\STATE Create a copy of $\mathcal{A}$ denoted by $\mathcal{A}^\gamma$ and run with $\mathbf{s}^\gamma_0$
\ENDFOR
\FOR{$t=1,\ldots$}
\STATE Generate features  $\x_t$ 
\FOR{$\gamma \in \Gamma$}
\STATE $\mathcal{A}^\gamma$ outputs $\hat{y}^\gamma_t=\u_t^\gamma \cdot \x_t$
\ENDFOR
\STATE Collect forecasts:  $\hy_t=(\hat{y}^\gamma_t)_{\gamma \in \Gamma}^{\transp}$
\STATE Project forecasts:
$\ty_t=\Pi_\K (\hy_t)$ 
\FOR{$\gamma \in \Gamma$}
\STATE  $\mathcal{A}^\gamma$ observes $y^\gamma_t$
\ENDFOR
\STATE Suffer a prediction error  $\frac{1}{|\Gamma|}\sum_{\gamma \in \Gamma} \big(y^\gamma_t-\tilde{y}^\gamma_t\big)^2$\\[3pt]
\ENDFOR 
\STATE  \textbf{aim}
\STATE \quad Minimize the average prediction error  $\displaystyle{\tilde{L}_T =\frac{1}{T} \sum_{t=1}^T \frac{1}{|\Gamma|} \big\|\y_t-\ty_t\big\|^2=\frac{1}{T\,|\Gamma|}\sum_{t=1}^T\sum_{\gamma \in \Gamma} \big(y^\gamma_t-\tilde{y}^\gamma_t\big)^2 }$.
\end{algorithmic}
\end{protocol}

\subsection{Assessment of the Forecasts -- Form of the Theoretical Guaranties Achieved}
\label{subsec:regret}

Our forecast are linear combinations of the features and are evaluated by the average prediction error
\begin{equation}
\label{eq:loss}
\tilde{L}_T \defeq \frac{1}{T} \sum_{t=1}^T \frac{1}{|\Gamma|} \sum_{\gamma \in \Gamma} \big(y^\gamma_t-\tilde{y}^\gamma_t\big)^2\,.
\end{equation}
We want to compare our method to constant linear combinations of features.
For example, recalling that, for $\gamma \in \Gamma$, $x_t^\gamma$ is the benchmark prediction of $y_t^\gamma$,  using  $\boldsymbol{\delta}^\gamma \defeq \mathbf{1}_{\{i=\gamma\}}$ (the standard basis vector that points in the $\gamma$ direction) as weights should be a good first choice to define a constant linear combination (for any $\gamma \in \Gamma$, this strategy provides $\boldsymbol{\delta}^\gamma \cdot \x_t = x_t^\gamma$ as forecast for $y_t^\gamma$).
Thus, the matrix $\big( \boldsymbol{\delta}^\gamma  \big)_{\gamma \in \Gamma}$ defines a constant benchmark strategy and
its cumulative prediction error is
\[ L_T \Big( \big( \boldsymbol{\delta}^\gamma  \big)_{\gamma \in \Gamma} \Big) \defeq
 \frac{1}{T}\sum_{t=1}^T\frac{1}{|\Gamma|}\sum_{\gamma \in \Gamma} \big(y^\gamma_t- x_t^\gamma \big)^2.\] 
As soon as the features $(x_t^\gamma)_{\gamma \in \Gamma}$ are well-chosen, this quantity is small. But, these benchmark predictions  do not satisfy the summation constraints \textit{a priori} and it won't be fair to compare our forecasts (which do respect to hierarchy -- projection step ensures it) to these benchmark forecasts -- or any other constant linear combinations of features. 
Thus, we introduce, in paragraph~\ref{subsub:class},
the set $\mathcal{C}$ which contains all the constant strategies which satisfy the hierarchical constraints and we also detail how a such strategy can be represented by a $|\Gamma|\times |\Gamma|$-matrix $\U \in \mathcal{C}$. 
 In paragraph~\ref{subsub:regret}, we decompose, for any $\U \in \mathcal{C}$, the average prediction error into an approximation error $L_T(\U)$ -- the average prediction error  of $\U$ -- and a sequential estimation error $\mathcal{E}_T(\U)$.
To achieve almost as well as the best constant combination of features, 
we want to obtain some guarantee of the form:
\begin{equation}
\label{eq:bound_form}
\tilde{L}_T\, \leq  \,\inf_{\U \in \mathcal{C}} \Big\{ L_T(\U) + \mathcal{E}_T(\U)\Big\}\,, \quad \mathrm{where} \quad \mathcal{E}_T(\U) = \mathcal{O}\Big(\frac{1}{\sqrt{T}}\Big)\,.
\end{equation}
Indeed, if $\mathcal{E}_T(\U) \underset{T \to +\infty}{\longrightarrow} 0$, the average prediction error of our strategy tends to $L_T(\U)$ -- and classical convergence rate are in $\frac{1}{\sqrt{T}}$ (see sections~\ref{subsub:regret} and~\ref{sec:aggregalgo}).
We will explain  how this aim is equivalent to minimizing the quantity called regret that we define below.

\subsubsection{Class of Comparison}
\label{subsub:class}

We consider here a constant strategy, namely $|\Gamma|$ linear combinations of the features.
More formally, let us denote by $\u^\gamma$ a constant weight vector which provides, for any instance $t$, the forecast $\u^\gamma \cdot \x_t$ for the time series $y^\gamma_t$.
By batching these $|\Gamma|$ vectors into a matrix $\U \defeq (\u^\gamma)_{\gamma \in \Gamma} \in \mathcal{M}_{|\Gamma|}$,
predictions satisfy the constraints for an instance $t$ if $\U^{\transp} \x_t \in \mathrm{Ker}(\K)$. 
For it to be true for any $t$ (except for a few particular case -- for instance if all features vector are null), this requires that the image of $\U^{\transp}$ is in the kernel of $\K$.
We introduce the following set of matrices, for which associated forecasts necessarily satisfy the hierarchical constraints
\[\mathcal{C} \defeq \left\{ \, \U = \big( \u^1 \, \big| \, \dots \, \big| \, \u^{|\Gamma|} \big)  \, \big| \, \mathrm{Im} \big(\U^{\transp}\big) \subset \mathrm{Ker}\big(\K\big)  \right\}.\]
Note that, for any matrix $\U \in \mathcal{M}_{|\Gamma|}$, by definition of the orthogonal projection $\Pi_{\K}$, the forecast vector $\Pi_{\K} \U^{\transp} \x_t$ satisfies the hierarchical relationships so the set $\mathcal{C}$ contains the matrix $\U\Pi_{\K}^{\transp}$. This implies that the set $\mathcal{C}$ is not empty. 
To compare our methods to any constant strategy $\U \in \mathcal{C}$, we now introduce the common notion of regret.

\subsubsection{Aim: Regret Minimization}
\label{subsub:regret}
We want to compare the average prediction error $\tilde{L}_T$ to $L_T(\U)$, where $\U \in \mathcal{C}$ so the forecasts associated with $\U$ satisfy the hierarchical constraints -- otherwise, the two strategies would not be comparable because our predictions do respect the hierarchy.
Good algorithms should ensure that $\tilde{L}_T$ is not too far from the best $L_T(\U)$.
We thus define, for any $\U = (\u^\gamma)_{\gamma \in \Gamma} \in \mathcal{C}$, the cumulative prediction error of the associated constant linear combinations of features by
 \[L_T(\U)\defeq \frac{1}{T}\sum_{t=1}^T\frac{1}{|\Gamma|}\sum_{\gamma \in \Gamma} \big(y^\gamma_t-\u^\gamma \cdot \x_t \big)^2 =\frac{1}{T|\Gamma|} \sum_{t=1}^T \big\|\y_t-\U^{\transp} \x_t\big\|^2.\]
 In order to obtain a theoretical guarantee of the form of Equation~\eqref{eq:bound_form}, we decompose the average prediction error as
 \begin{equation}
	\label{eqn:decomp-loss}
	\tilde{L}_T = L_T(\U) + \frac{R_T(\U)}{T |\Gamma|}\,,
\end{equation}
where, the quantity $R_T(\U)$, commonly called regret is defined as the difference between the cumulative prediction error of our method and the one for weights $\U$:
\[
R_T(\U)\defeq T |\Gamma| \times \Big( \tilde{L}_T - L_T(\U) \Big) = \sum_{t=1}^T \big\| \y_t -\ty_t \big\|^2 - \sum_{t=1}^T \big\| \y_t - \U^{\transp}\mathbf{x}_t \big\|^2.
\]
In the light of Equation~\eqref{eqn:decomp-loss}, the average prediction error $\tilde{L}_T$ we attempt to minimize 
breaks down into an approximation error $L_T(\U)$ (the best prediction error we can hope for) and a sequential estimation error (dependent of how quickly the model estimate $\U$), proportional to the regret $R_T(\U)$. As stated before, the aim for algorithms is that $\tilde{L}_T$ is as close as possible to $\min_{\U \in \mathcal{C}} L_T(\U)$ (with  $\mathcal{C}$ the class of comparison defined above), which is equivalent to $\max_{\U \in \mathcal{C}} R_T(\U)$ being small. This point of view is very common for online forecasting methods (see, among others,~\citealp{devaine2013forecasting} and \citealp{mallet2009ozone}), and for an algorithm to be useful,  $\max_{\U \in \mathcal{C}} R_T(\U)$ need to be sub-linear in $T$ (otherwise the error remains constant -- or even worst: it increases with time).
Typical theoretical guaranties provide bounds of order $\sqrt{T}$ (see for example,~\citealp{deswarte:hal-01939813} and \citealp{amat2018fundamentals}).

\subsection{Technical Discussion: why we require the same features at each node.}
\label{subsec:samefeatures}
In this section, we explain why we consider the same features vector for each nodes. \textit{A priori}, we could have a different set of features at each node $\x_t^\gamma$, created with methods specific to this node. Also the size of feature vector $d^\gamma$ associated with the node $\gamma$ could vary.
Prediction of a time series $y_t^\gamma$ associated to a $d^\gamma$-vector $\u^\gamma$ is $\hat{y}^\gamma_t=\u^\gamma \cdot \x_t^\gamma$.
Therefore, a global constant strategy is a set $\big\{\u^1,\dots,\u^{|\Gamma|}\big\} \subset \mathbb{R}^{d^1 \times \dots \times d^{|\Gamma|}}$.
First is is a little less practical because unlike the previous setting, the vectors $\u^1,\dots,\u^{|\Gamma|}$ and $\x^1,\dots,\x^{|\Gamma|}$ are of different sizes, so it is less easy to use matrix notations.
Moreover, it becomes tricky to specify the class of constant strategies to compare to. As said before, the forecast vector $\hy_t=(\hat{y}^\gamma_t)_{\gamma \in \Gamma}$ satisfies the summation constraints if and only if it is in the kernel of $\K$. Thus, the following set, which contains the constant strategies fulfilling the hierarchical constraints for all $t>0$, 
\[
\bigg\{ \Big(\mathbf{u}^1,\dots,\mathbf{u}^{|\Gamma|} \Big)\in \mathbb{R}^{d^1\times \dots \times d^{|\Gamma|}} \, \Big| \,\forall t >0,\, \Big(\mathbf{u}^1\cdot\mathbf{x}_t^1,\dots,\mathbf{u}^{|\Gamma|}\cdot\mathbf{x}_t^{|\Gamma|}\Big)^{\transp} \in\mathrm{Ker}\big(\K\big) \bigg\},
\]
is not explicitly defined and may be empty because of the number of constraints on $(\mathbf{u}^1,\dots,\mathbf{u}^{|\Gamma|} )\in \mathbb{R}^{d^1\times \dots \times d^{|\Gamma|}}$ which increases at each time step. If there is no restrictions on the feature vectors, these constraints could be linearly independent, leading to an empty set.
Indeed, if we consider that the times series are connected by $K$ summations relationships, at each instance $t$, the $d^1 + \dots + d^{|\Gamma|}$ coefficients of vectors $\mathbf{u}^1,\dots,\mathbf{u}^{|\Gamma|}$ are linked by $K$ equations.
 As features are proper to each node, theses constraints have no reason to be dependent, so as soon as $T \times K > d^1 + \dots + d^{|\Gamma|}$, the above set may likely be empty.
 Because of that, it is not clear how to define the regret in this setting.
For this reason, we decided to use the same features vectors $\x_t$ for all nodes of $\Gamma$; which has also the benefit of allowing a simpler presentation.

\section{Main Theoretical Result}
\label{sec:mainresult}

From now on, let us introduce the following notation concerning the regret bound of Algorithm~$\mathcal{A}$.
\begin{notation}
We assume that, for any $\gamma \in \Gamma$ with the initialization parameter vector $\mathbf{s}_0^\gamma$, Algorithm $\mathcal{A}^\gamma$ ensures, for $T>0$ and for any $\u^\gamma \in \mathbb{R}^{|\Gamma|}$, any $\x_{1:T}=\x_1, \dots \x_T$ and any $y_{1:T}^\gamma=y_{1}^\gamma, \dots, y_{T}^\gamma$,
\begin{equation}
R_T^\gamma(\u^\gamma) \defeq \sum_{t=1}^T \big( y_t^\gamma - \hat{y}_t^\gamma   \big)^2- \sum_{t=1}^T \big( y_t^\gamma - \u^\gamma \cdot \mathbf{x}_t   \big)^2 \leq \mathrm{B}(\mathbf{x}_{1:T},y_{1:T}^\gamma,\mathbf{s}^\gamma_0,\u^\gamma).
\label{hyp::algo}
\end{equation}
\label{ass1}
\end{notation}
Details and examples of these regret bounds are provided in Section~\ref{sec:aggregalgo} that describes the aggregation algorithms considered in the experiments of Section~\ref{sec:experiments}. 
As getting a linear bound is trivial (by using the common assumption that prediction errors are bounded), these bounds have to be sub-linear to be of interest.
Referring to the average prediction error decomposition of Equation~\eqref{eqn:decomp-loss}, the sub-linearity ensures that the sequential estimation error $R_T^\gamma(\u^\gamma) / T$ tends to $0$.
This notation makes it possible to establish a bound of the cumulative regret.
\begin{theorem}
\label{gen-bound}
Under Notation~\ref{ass1}, for any matrix $\U \in \mathcal{C}$ and any $T\geq 1$, 
\[ R_T(\U)\leq \sum_{t=1}^T \big\| \y_t -\ty_t \big\|^2 - \sum_{t=1}^T \big\| \y_t - \U^{\transp}\mathbf{x}_t \big\|^2 = \sum_{\gamma \in \Gamma} \mathrm{B}\big(\mathbf{x}_{1:T},y_{1:T}^\gamma,\mathbf{s}_0^\gamma,\u^\gamma\big). \]
\end{theorem}
The regret $R_T(\U)$ is not just the sum over all the nodes of the regrets $R_T^\gamma(\u^\gamma)$ of Equation~\ref{hyp::algo}. Indeed, we do not evaluate here the forecasts $\hy_t$ but those obtained after the projection step: $\ty_t$.
The projection step provides a diminishing of the square prediction error and we just have to sum Equation~\ref{hyp::algo} on all nodes to get the bound.
\begin{proof}
This regret bound results from two main arguments: Pythagorean theorem, on the one hand, and Notation~\ref{ass1}, on the other hand.
For any $t\geq 1$, as $\y_t \in \mathrm{Ker}(\K)$, the Pythagorean theorem ensures 
\begin{equation}
\big\| \y_t - \ty_t\big\|^2 = \big\|\y_t - \Pi_\K (\hy_t)\big\|^2  \leq \big\| \y_t - \hy_t\big\|^2.
\label{pythagore}
\end{equation}
Let us fix a matrix $\U=\big(\u^1 | \, \dots | \, \u^{|\Gamma|}\big) \in \mathcal{C}$.
Firstly, the application of Pythagorean theorem ensures that the projection step reduces regret. 
Rewriting the regret as a sum over the nodes,
we then use Notation~\ref{ass1} independently for each node of $\Gamma$ to conclude the proof.
\begin{align*}
R_T(\U)&=\sum_{t=1}^T \big\| \y_t -\ty_t \big\|^2 - \sum_{t=1}^T \big\| \y_t - \U^{\transp} \mathbf{x}_t \big\|^2  \\
& \stackrel{\mbox{\eqref{pythagore}}}{\leq} \sum_{t=1}^T \big\| \y_t -\hy_t \big\|^2 - \sum_{t=1}^T \big\| \y_t - \U^{\transp} \mathbf{x}_t \big\|^2 \\
&= \sum_{\gamma \in \Gamma}  \sum_{t=1}^T \big( y^\gamma_t -\hat{y}^\gamma_t \big)^2 - \sum_{\gamma \in \Gamma}  \sum_{t=1}^T \Big( y^\gamma_t - \u^\gamma\cdot \mathbf{x}_t \Big)^2 =  \sum_{\gamma \in \Gamma} R_T^\gamma\big(\u^\gamma\big) \\ 
& \stackrel{\mbox{\eqref{hyp::algo}}}{\leq} \sum_{\gamma \in \Gamma} \mathrm{B}\big(\mathbf{x}_{1:T},y_{1:T}^\gamma,\mathbf{s}_0^\gamma,\u^\gamma\big).\qedhere
\end{align*}
\end{proof}

\begin{rem}
\label{rem:uniform_regret}
For an initialization parameter vector $\mathbf{s}_0^\gamma$, and a subset $\mathcal{D}\subset \mathbb{R}^{|\Gamma|}$, some aggregation algorithms provide a uniform regret bound of the following form:
\begin{equation*}
	R_T^\gamma (\mathcal{D}) \defeq \sum_{t=1}^T \big( y_t^\gamma - \hat{y}_t^\gamma   \big)^2- \min_{\u^\gamma \in \mathcal{D}}\sum_{t=1}^T \big( y_t^\gamma - \u^\gamma \cdot \x_t   \big)^2 \leq \mathrm{B}\big(\mathbf{x}_{1:T}^\gamma,y_{1:T}^\gamma,\mathbf{s}_0^\gamma\big).
\end{equation*}
In this case, let us introduce, for any subset $ \mathcal{B} \subset  \mathcal{M}_{|\Gamma|}$, the subset $\mathcal{B}_{| \mathcal{D}} \defeq \big\{ \U\in\mathcal{B}\, | \, \forall \gamma \in \Gamma, \, \u^\gamma \in \mathcal{D} \big\} $. Then, we bound the cumulative regret $R_T(\mathcal{D})$ defined just below with
\[R_T(\mathcal{D}) \defeq \max_{\U \in \mathcal{C}_{|\mathcal{D}}} R_T(\U).\]
With the same previous arguments we get the uniform regret bound
\[ R_T(\mathcal{D}) =   \sum_{t=1}^T \big\| \y_t -\ty_t \big\|^2 -\min_{\U \in \mathcal{C}_{|\mathcal{D}}} \sum_{t=1}^T \big\| \y_t - \U^{\transp}\mathbf{x} \big\|^2\leq \sum_{\gamma \in \Gamma} \mathrm{B}\big(\mathbf{x}_{1:T}^\gamma,y_{1:T}^\gamma,\mathbf{s}_0^\gamma\big).\] 
\end{rem}
\section{On one Operational Constraint: Half-Hourly Predictions with One-Day-Delayed Observations}
In this section, we highlight the differences between the previous theoretical setting and the practical setting of our experiments and how these changes affect the regret bound.
In Section~\ref{sec:experiments}, we aim to forecast power consumptions at half-hourly intervals. 
Meta-algorithm~\ref{prot:1} makes the implicit assumption that historical time series values are available and to forecast at an instance $t$, we can use $\y_{1:t-1}$. We thus assume that very recent past observations, up to half an hour ago, would be available -- and it is not realistic at all.
Indeed, there is some operational constraints on the power network and on meters that make it difficult to instantly access the data: it is common to obtain load records with a delay of a few hours or even a few days. 
Although this delay is becoming shorter with the deployment of smart meters and the evolution of grids, we cannot consider we have access to the consumption of the previous half-hour. 
To take into account these operational constraints and to carry out experiments under practical conditions, we make the classic assumption that we have access to consumptions with a delay of $24$ hours (see among others \citealp{fan2011short} and \citealp{GaillardAdditivemodelsrobust2016a}).
As now, only past observations $\y_{1:t-48}$ are available at an instance $t$,  we adapt the previous method a bit.\\

As we will see in Section~\ref{sec:features}, the half-hour of the day is a crucial variable for power consumption forecasting and to obtain relevant forecasts, we will consider the consumption of the previous day at the same half-hour (but never the one of the previous half-hour). 
Thus the delay in the access to consumption observation is not an issue for feature generation. 
But it becomes especially problematic for online learning (in our experiments, features are generated offline with models trained on historical data).
Indeed, in the aggregation step of our method, we assume to observe, for each node $\gamma$ and at each instance $t$, the consumption $y_{t-1}^\gamma$ -- that is not possible anymore.
To deal with this issue we initially considered two solutions.
In our first approach, for any $\gamma \in \Gamma$, the time series $(y_t^\gamma)$ is divided into $48$ time series with daily time steps. Then, $48$ aggregations are done in parallel and, as $t-1$ now refers to the previous day, there is no more delay issue. 
The $48$ series are then collected to reconstruct a time series at half-hour time step.
For a constant strategy $\u^\gamma$, the regret of the global aggregation $R^\gamma_T(\u^\gamma)$ is simply the sum of the $48$ regrets  -- that refer to the $48$ aggregation run in parallel on the $48$ daily time series -- denoted by $\big(R^{\gamma\, h}_{T/48}(\u^\gamma)\big)_{1 \leq h \leq 48}$, so we have
$$R^\gamma_T(\u^\gamma)= \sum_{h=1}^{48} R^{\gamma\, h}_{T/48}(\u^\gamma).$$
 If we consider an aggregation algorithm that ensures a bound of the form of Notation~\ref{ass1} where the bound $B$ depends only on the horizon time -- namely, $R \leq  B$ for all $h$ -- the regret associated with the half-hourly time series $(y_t^\gamma)$ satisfies:
$$R^\gamma_T(\u^\gamma)\leq 48 \times B(T/48).$$
\citet{joulani2013online} provide an overview of work on online learning under delayed feedback and for our framework, which refers to
full information setting with general feedback. The bound above matches their results. 
In a second approach, we ``ignore" the delay in a sense that we apply the 
aggregation algorithms as if the delayed observations $\y_{t-48}$ were $\y_{t}$. 
Thus, in Meta-algorithm~\ref{prot:1}, at each node $\gamma$ and any instance $t$, instead of outputting the forecast $y_t^\gamma=\u_t^\gamma\cdot \x_t$, algorithm~$\mathcal{A}^\gamma$ outputs $y_t^\gamma=\u_{t-48}^\gamma\cdot \x_t$.
For simplicity of notation, the aggregation algorithms of Section~\ref{sec:aggregalgo} are presented in their original version, namely assuming that observations at $t-1$ are available at an instance $t$.
Such adaptations have already been tested: Algorithm 15 of~\citet{Gaillard2015} gives a delayed version of Algorithm~\ref{algo:ml_pol} that we also use in Section~\ref{sec:experiments}.
After testing both approaches, we kept the second one, which achieves a much better performance. 
Our choice was also supported by Chapter~9 of \citet{Gaillard2015} experiments, which drew similar conclusions.

\section{Generation of the Features}
\label{sec:features}

Here we describe the forecasting methods we use in the experiments of Section~\ref{sec:experiments} to generate the benchmark predictions that will be used as features in the sequel.
We recall (see Section~\ref{sec:methodo}) that
throughout this work, we consider that, at each node $\gamma \in \Gamma$ and for any instance $t$, a forecaster provides a benchmark prediction $x_t^\gamma$ based on historical data of the time series $(y_t^\gamma)_{\gamma \in \Gamma}$ and on exogenous variables relative to the node $\gamma$. 
These $|\Gamma|$ forecasters independently generate the $|\Gamma|$ forecasts $(x_t^\gamma)$ in parallel
and the set of features $\x_t$ is made up of the above $|\Gamma|$ benchmark predictions.
Forecasts can be the output of any predictive model. 
In the experiments of Section~\ref{sec:experiments}, we consider three forecasting methods, that are described in the
following Subsections~\ref{subsec:autoreg},~\ref{subsec:gam} and~\ref{subsec:rf}. 
\paragraph{Notation} Subsections~\ref{subsec:autoreg} and ~\ref{subsec:gam} present parametric methods. For any parameter $a$ of the model, we will denote by $\hat{a}$ its estimation (no matter the method we use).

\subsection{Auto-Regressive Model}
\label{subsec:autoreg}
A simple approach consists in considering an auto-regressive model. 
Let us fix $\gamma \in \Gamma$ and assume that, to predict the time series $(y^\gamma_t)_{t>0}$, we have access to historical observations.
For an instance $t$, the  model specifies that the output variable $y^\gamma_t$ depends linearly on its own previous values. 
In Section~\ref{sec:experiments}, we consider the power consumption at half-hourly intervals. 
For an instance $t$, to forecast the time series  $y^\gamma_t$ we assume to have access to the power consumption at D-$1$ and D-$7$, which correspond to $y^\gamma_{t-48}$ and $y^\gamma_{t-7\times 48}$, respectively. 
We predict the consumption  half-hour by half-hour thanks to linear models taking as explanatory variables its values at  D-$1$ and D-$7$.
We assume that these $48$ auto-regressive models have the same coefficients.
Thus, for this modeling, the power consumption associated with the node $\gamma$ equals
 \begin{align*}
y_t^\gamma=a_1^\gamma y_{t-48}^\gamma+a^\gamma_7 y_{t-7\times 48}^\gamma+\mathrm{noise} \,.
 \end{align*}
For each $\gamma \in \Gamma$, we estimate the coefficients $a_1^\gamma$ and $a_7^\gamma$ 
using ordinary least squares regression on a training data set. 
Therefore, at an new instant $t$, we predict
 \begin{align*}
x_t^\gamma=\hat{a}_1^\gamma y_{t-48}^\gamma+\hat{a}^\gamma_7 y_{t-7\times 48}^\gamma \,.
 \end{align*}

\subsection{General Additive Model}
\label{subsec:gam}
Generalized additive models (see the monograph of~\citealp{wood2006generalized} an in-depth presentation)
are effective semi-parametric approaches to forecast electricity consumption (see, among others, \citealp{goude2014local} and \citealp{GaillardAdditivemodelsrobust2016a}).
They model the power demand as a sum of independent exogenous (possibly non-linear) variable effects. We describe this model using the specification we chose in our experiments.
In Section~\ref{sec:experiments}, for a node $\gamma \in \Gamma$,
we take into account some local meteorological variables at the half-hour time step: the temperature $\tau^\gamma$ and the smoothed temperature $\bar{\tau}^\gamma$, the visibility $\nu^\gamma$, and the humidity $\kappa^\gamma$. 
For an instant $t$, we also introduce calendar variables: the day of the week $d_t$ (equal to 1 for Monday, 2 for Tuesday, etc.), the half-hour of the day $h_t \in \{1,...,48\}$ and the position in the year $\rho_t \in [0,1]$, which takes linear values between  $\rho_t= 0$ on January 1st at 00:00 and $\rho_t = 1$ on December the 31st at 23:59.
As the effect of the half-hour $h_t$ is crucial to forecast load, it is often more efficient to consider a model per half-hour (see \citet{fan2011short} and \citet{goude2014local}). 
The global model is then the sum of 48 daily models, one for each half-hour of the day.
More precisely, we consider the following additive model for the load, which breaks down time by half hours:
 \begin{align*}
y_t^\gamma=\sum_{h=1}^{48} \mathbf{1}_{h_t=h} \Big[  a_h^\gamma y_{t-7\times 48}^\gamma + \, s_{1,h}^\gamma\big(y_{t-48}^\gamma\big) +  \, s_{\tau,h}^\gamma(\tau^\gamma_t) + \, s_{\bar{\tau},h}^\gamma(\bar{\tau}^\gamma_t) + \,  s_{\nu,h}^\gamma(\nu^\gamma_t) & \nonumber \\
+ \, s_{\kappa,h}^\gamma(\kappa^\gamma_t) + \sum_{d=1}^7w_{d,h}^\gamma\mathbf{1}_{d_t=d}  +  \, s_{\rho,h}^\gamma(\rho_t) \Big] + \mathrm{noise}.&
 \end{align*}
 The $s_{1,h}^\gamma$, $s_{\tau,h}^\gamma$, $s_{\bar{\tau},h}^\gamma$, $s_{\nu,h}^\gamma$, $s_{\kappa,h}^\gamma$ and  $s_{\rho,h}^\gamma$  functions catch the effect of the consumption lag, the meteorological variables and of the yearly seasonality.
They are cubic splines: $\mathcal{C}^2$-smooth functions made up of sections of cubic polynomials joined together at points of a grid.
The coefficients $a^\gamma_h$ and $w^\gamma_{d,h}$  model the influence of the consumption at D-$7$ and of the day of the week.
Indeed, we consider a linear effect for the consumption at D-$7$ (it achieved a better performance than a spline effect in our experiments) and as the day of the week takes only $7$ values, we write its effect as a sum of indicator functions, and thus $7$ coefficients $w^\gamma_{d,h}$ are considered. 
As we consider a model per half-hour, all the coefficients and splines are indexed by $h$.
To estimate each model, we use the Penalized Iterative Re-Weighted Least Square (P-IRLS) method~\citealp{wood2006generalized}, implemented in the \texttt{mgcv} R-package, on a training data set.
At any node $\gamma \in \Gamma$, for a new round $t$, we then output the forecast 
 \begin{align*}
x_t^\gamma=\sum_{h=1}^{48} \mathbf{1}_{h_t=h}\Big[  \hat{a}_h^\gamma y_{t-7\times 48}^\gamma +  \, \hat{s}_{1,h}^\gamma\big(y_{t-48}^\gamma\big) +  \, \hat{s}_{\tau,h}^\gamma(\tau^\gamma_t) + \,  \hat{s}_{\bar{\tau},h}^\gamma(\bar{\tau}^\gamma_t) + \,  \hat{s}_{\nu,h}^\gamma(\nu^\gamma_t)  &\nonumber \\
+ \,   \hat{s}_{\kappa,h}^\gamma(\kappa^\gamma_t) + \sum_{d=1}^7\hat{w}_{d,h}\mathbf{1}_{d_t=d}  +  \, \hat{s}_{\rho,h}^\gamma(\rho_t) \Big].&
 \end{align*}
 
\subsection{Random Forests}
\label{subsec:rf}

Random forests form a powerful learning method for classification and regression that constructs a collection of decision trees from training data and output, for each new data point, the mean prediction of the individual trees.
Introduced by~\citet{breiman2001random}, theses approaches operate well on many applications. 
Recent work demonstrates their efficiency in forecasting power consumption (see, among others~\citealp{goehry2019aggregation} and~\citealp{fan2011short}). 
A random forest is made up of a set $\big(T_k^\gamma\big)_{1 \leq k \leq K}$ of decision trees grown in the following way (see \citealp{BreimanFOS84} for further details).
For each  $k=1,\dots,K$, we first randomly draw, with replacement, $n$ points from the training data set and start at the root, that contains all the points of the sub-sample.
At each node $\mathcal{N}$ with more than $m$ data points, $V$ variables are randomly selected among the exogenous variables.
Given a variable $v \in V$ and a threshold $s$, each point of the node $\mathcal{N}$ is assigned to the left daughter node $\mathcal{N}_L$  if its value in $v$ is lower than $s$ or to the right daughter node $\mathcal{N}_R$ otherwise. 
Considering only these $V$ variables, the best split -- given by a pair $(v,s)$ of variable and an associated threshold -- to separate the points into two set $\mathcal{N}_L$ and $\mathcal{N}_R$ is determined by minimizing the variance criterion indicated below. 
 For any node $\mathcal{N}$ let us define the variance $\mathrm{Var} (\mathcal{N})$ by
 \[
 \mathrm{Var}(\mathcal{N})\defeq \frac{1}{|\mathcal{N}|} \sum_{i \in \mathcal{N}} \big(y^\gamma_i - \bar{y}^\gamma_{\mathcal{N}} \big)^2, \quad \mathrm{with} \quad \bar{y}^\gamma_{\mathcal{N}}\defeq  \frac{1}{|\mathcal{N}|} \sum_{i \in \mathcal{N}} y^\gamma_i.
 \]
Each node $\mathcal{N}$ is split in the two daughter nodes $\mathcal{N}^{^\star}_R$ and $\mathcal{N}_L^{^\star}$ (determined by the choice of $v$ and $s$) minimizing the following criterion
 \begin{equation}
\label{eq:variance_criterion}
\big(\mathcal{N}_R^{^\star} , \mathcal{N}_L^{^\star} \big) \in \argmin_{\mathcal{N}_R,\, \mathcal{N}_L}   \frac{|\mathcal{N}_R|}{n} \, \mathrm{Var}\big(\mathcal{N}_R\big)+  \frac{|\mathcal{N}_L|}{n} \, \mathrm{Var}\big(\mathcal{N}_L\big).
 \end{equation}
  Thus, we create a binary test to split the points of the node.
When all the leaves contain fewer than $m$ points, we associate with each leaf the mean of its data points.
For a new point, we look at the values of its variables. For each $k=1,\dots,K$, we browse the tree $T_k^\gamma$ and predict the value of the corresponding leaf.
The $K$ resulting forecasts are then averaged out.
Algorithm~\ref{rf} describes the above procedure and is implemented in the \texttt{ranger} R-package.
\begin{algorithm}[htb!]
\caption{\label{rf} Random Forest for Regression}
\begin{algorithmic}
\STATE \textbf{Parameters} 
\STATE \quad Number of trees $K$
\STATE \quad Sample size $n$
\STATE \quad Minimal node size $m$
\STATE \quad Number of variables to possibly split at in each node $V$
 \FOR{$k = 1, \dots,K$}
 \STATE Draw a sample (with replacement) of size $n$ from training data
 \STATE Construct the tree $T_k$ starting at the root with all the $n$ data points 
 \WHILE{a leaf contains more than $m$ data points} 
 \FOR{each leaf of more than $m$ data points}
 \STATE Select $V$ variables 
 \STATE Split the node into two nodes using the variance criterion~\eqref{eq:variance_criterion} among the chosen variables
 \ENDFOR
\ENDWHILE
\ENDFOR 
\STATE Output $\Big(T_k^\gamma \Big)_{1\leq k \leq K}$ \\[3pt]
\STATE  \textbf{Prediction at a new data point}
 \STATE  \quad Mean of the $K$ forecasts output by the trees $\big(T_k^\gamma \big)_{1\leq k \leq K}$
\end{algorithmic}
\end{algorithm}
In the experiments of Section~\ref{sec:experiments}, we take $n$ equal to the number of data points in the training set, $m=5$ and $K=500$ (default parameters of \texttt{ranger}). The number $V$ has been optimized by grid search; what we obtained is that, for each node, we keep two-thirds of the variables to split it (these variables are the same as the ones described in the previous section).
With $\big(T_k \big)_{1\leq k \leq K}$, the trees constructed by Algorithm~\ref{rf} run on a training data set, the forecast of any node $\gamma \in \Gamma$, at a new round $t$, is then
\begin{equation*}
x_t^\gamma=\frac{1}{K} \sum_{k=1}^K T_k^\gamma \big(  y_{t-7\times 48}^\gamma , y_{t-48}^\gamma,\tau^\gamma_t,\bar{\tau}^\gamma_t,\nu^\gamma_t, \kappa^\gamma_t,\rho_t,d_t,h_t \big).
\end{equation*}
\section{Aggregation Algorithms}
\label{sec:aggregalgo}
This section describes the three aggregation algorithms we use in the experiments of Section~\ref{sec:experiments}.
At an instance $t$, for a node $\gamma \in \Gamma$, a copy $\mathcal{A}^\gamma$ of an aggregation algorithm $\mathcal{A}$ takes the feature vector $\x_t$  (generated with one of methods of the previous section) as an input and outputs the forecast $\hat{y}_t^\gamma=\u_t^\gamma \cdot \x_t$.
Therefore, to forecast the node $\gamma$, we use  $\x_t$, which contains the predictions of all the nodes (including that of the considering node). 
We remind that the features $ \big(x^\gamma_t\big)_{\gamma \in \Gamma}$ are generated independently with possibly different exogenous variables but that the observations $ \big(y^\gamma_t\big)_{\gamma \in \Gamma}$ may be strongly correlated.
This is why we consider aggregation to refine some forecasts by combining the features.
Our experiments demonstrate that this aggregation step improves the forecasts.
Subsection~\ref{subsec:standardization} presents a trick to empirically standardize the features and the observations first. 
On the one hand, this preprocessing justifies boundedness assumptions~\eqref{eq:ass2} on observations and features, that ensure some theoretical guaranties of the form requested by Notation~\ref{ass1}. On the other hand, this preprocessing simplifies hyper-parameters search (for the aggregation step) as we can choose the same for every series since they have similar statistics (scale and variance).
Following Subsections~\ref{subsec:linear_aggreg} and~\ref{subsec:convex_aggreg} introduce the aggregation algorithms and some technical tricks implemented in the experiments of Section~\ref{sec:experiments}.

\subsection{Standardization}
\label{subsec:standardization}

In empirical machine learning, it is known that standardizing observations and features may significantly improve results, and sequential learning is no exception (see \citealp{pmlr-v98-gaillard19a}).
In addition, standardization makes the calibration of the parameters of the algorithm common to all the nodes, namely for each algorithm $\mathcal{A}^\gamma$, we choose the hyper-parameters $\mathbf{s}_0^\gamma = \breve{\mathbf{s}}_0$.
We can do so, because thanks to the preprocessing below, features and observations will be of the same order.
Let us fix $\gamma \in \Gamma$ and $t>0$. We consider the following transformations, relying on statistics $S^\gamma$ and $\breve{\E}$ computed on $T_0$ historical time steps:
\begin{align*}
y^\gamma_t & \rightarrow \breve{y}^\gamma_t \, \defeq \,\frac{y^\gamma_t - x_t^\gamma}{S^\gamma} & \text{Observations tranform}\\
\x_t &\,  \rightarrow \breve{\x}_t \, \, \, \defeq  \, \, \, \breve{\E} \x_t  & \text{Features transform}
\end{align*}
\[
\mathrm{with} \quad S^\gamma = \max_{1-T_0 \leq t \leq 0} |y_t^\gamma - x_t^\gamma|
\quad\quad\text{ and }\quad\quad
\displaystyle\breve{\E} \defeq \left(\frac{1}{T_0} \sum_{t=1-T_0}^0\x_t\,\x_t^{\transp} \right)^{-1/2}\,.
\]
We thus assume that the Gram matrix $\frac{1}{T_0} \sum_{t=1-T_0}^0\x_t \x_t^{\transp}$ is invertible, which is a reasonable assumption as soon as $T_0$ is large enough.
Our standardization process differs from the usual methods  (see details below) but it provides the theoretical guaranties set out below.
Furthermore, it makes sense for the following reasons. 
Fixing $\gamma \in \Gamma$, when features and observations are bounded, $S^\gamma$ is an estimation of a bound on $y_t^\gamma-x_t^\gamma$. The re-scaling of $(y_t^\gamma-x_t^\gamma)$ by $S^\gamma$ should provide transformed observations lying in $[-1,1]$ or a some neighboring range.
It also reduces and homogenizes the variances for all the nodes.
A simple example may illustrate this variance reduction.
For deterministic features, the variance of non-transformed observations satisfy $\mathrm{Var} \big( y_t^\gamma \big)=\mathrm{Var} \big( y_t^\gamma - x_t^\gamma\big)$. The variance of standardized observations is then divided by $\big(S^\gamma\big)^2$ and we have
$\mathrm{Var} \big( \breve{y}_t^\gamma \big) = \mathrm{Var} \big( y_t^\gamma \big) / \big(S^\gamma\big)^2.$
For $T_0$ large enough, the variance of transformed observations should be less than $1$. Indeed, with high probability, the maximum of the absolute values of the random variable $(y_t^\gamma - x_t^\gamma)$ on $t=1-T_0, \dots, 0$ (which is $S^\gamma$), is higher than its standard deviation $ \sqrt{ \mathrm{Var} \big( y_t^\gamma \big)}$ and thus $\big(S^\gamma\big)^2 > \mathrm{Var} \big( y_t^\gamma \big)$.
Moreover, the expectation of $(y_t^\gamma - x_t^\gamma)$ should be close to $0$ as soon as the features are correctly generated. Indeed, the more the benchmark forecast are relevant, the more the observations are re-centered. 
Concerning the features, our standardization is classic in the case of centered features.
The matrix $\breve{\E}^2$ would then be an estimation of the inverse of the co-variance matrix of vectors $\x_t$, and the multiplication of the features by $\breve{\E}$ would provide transformed features whose co-variance matrix is close to the identity matrix. 
Here, we do not recenter observations and features with some empirical mean as it is classically done (this would be unconvenient for our regret analysis). 
Anyway, Subsection~\ref{subsec:experiment_design} provides some experimental results which confirm that our preprocessing standardizes reasonably well observations and features. 
Moreover, we tested classical standardization (with re-centering) on features and obtained results similar to those presented in Section~\ref{sec:experiments} (but, as hinted at above, no theoretical guaranties would be associated with this classical standardization). \\

We run Algorithm $\mathcal{A}^\gamma$ on transformed features and observations with the initialization parameter vector $\breve{\mathbf{s}}_0$ (which does not depend on $\gamma$)
and obtain a standardized prediction at node $\gamma$, denoted by $\bar{y}^\gamma_t$.
Then, we transform this output to get the (non-standardized) forecast  
\[\hat{y}_t^\gamma \defeq S^\gamma\bar{y}_t^\gamma +x_t^\gamma.\]
For any vector $\breve{\mathbf{u}}^\gamma \in \mathbb{R}^{|\Gamma|}$, we introduce the standardized regret associated with transformed observations and features, denoted by $\breve{R}_T^\gamma(\breve{\mathbf{u}}^\gamma)$ as:
\begin{align*}
\breve{R}_T^\gamma(\breve{\mathbf{u}}^\gamma)& \defeq \sum_{t=1}^T \big( \breve{y}^\gamma_t - \bar{y}_t^\gamma \big)^2 - \sum_{t=1}^T \big( \breve{y}^\gamma_t -\breve{\mathbf{u}}^\gamma\cdot \breve{\x}_t \big)^2 \\
&=\sum_{t=1}^T \bigg( \frac{y_t^\gamma - x_t^\gamma}{S^\gamma } - \frac{\hat{y}_t^\gamma - x_t^\gamma}{S^\gamma }  \bigg)^2 - \sum_{t=1}^T \Bigg(\frac{y_t^\gamma - x_t^\gamma}{S^\gamma }  - \breve{\u}^\gamma  \cdot  \big( \breve{\E} \x_t \big) \Bigg)^2 \\
&= \frac{1}{\big(S^\gamma\big)^2}\sum_{t=1}^T \big(y_t^\gamma -\hat{y}_t^\gamma \big)^2 -
\frac{1}{\big(S^\gamma\big)^2} \sum_{t=1}^T \bigg( y_t^\gamma -  \underbrace{\Big( x_t^\gamma  +  S^\gamma \big( \breve{\E} \breve{\u}^\gamma \big)  \cdot   \x_t \Big)}_{\u^\gamma \cdot \x_t} \bigg)^2.
\end{align*}
In the equations above, we define $\u^\gamma \defeq \boldsymbol{\delta}^\gamma +  S^\gamma \big( \breve{\E} \breve{\u}^\gamma \big)$ where
$\boldsymbol{\delta}^\gamma \defeq \big( \mathbf{1}_{\{i=\gamma\}} \big)_{i \in \Gamma}\,$ denotes the standard basis vector that points in the $\gamma$ direction. Equivalently,
$\breve{\mathbf{u}}^\gamma= \breve{\E}^{-1}(\u^\gamma -~\boldsymbol{\delta}
^\gamma)/ S^\gamma$, so there is a bijective correspondence between the vectors $\u^\gamma$ and $\breve{\u}^\gamma$.
Therefore, by noticing that $x_t^\gamma = \boldsymbol{\delta}^\gamma \cdot \x_t$, the regret associated with original features and observations is related to the regret of transformed data by the following equation:
\[ \breve{R}_T^\gamma(\breve{\mathbf{u}}^\gamma)= \frac{1}{\big(S^\gamma\big)^2}\sum_{t=1}^T \big(y_t^\gamma -\hat{y}_t^\gamma \big)^2 -
\frac{1}{\big(S^\gamma\big)^2} \sum_{t=1}^T \big( y_t^\gamma -  \u^\gamma \cdot   \x_t \big)^2 = \frac{R_T^\gamma(\u^\gamma)}{\big(S^\gamma\big)^2}  \,. \]
Furthermore, as for any ${\breve{\mathbf{u}} \in \mathbb{R}^{|\Gamma|}} $, Notation~\ref{ass1} ensures
\begin{align*}
\breve{R}_T^\gamma(\breve{\mathbf{u}}^\gamma)&=\sum_{t=1}^T \big( \breve{y}^\gamma_t - \bar{y}_t^\gamma \big)^2 - \sum_{t=1}^T \big( \breve{y}^\gamma_t -{\breve{\mathbf{u}^\gamma}}^{\transp}\breve{\x}_t \big)^2  \leq  \mathrm{B}\big(\breve{\x}_{1:T},\breve{y}_{1:T}^\gamma,\breve{\mathbf{s}}_0,\breve{\mathbf{u}}^\gamma \big)\,,
\end{align*}
Combining the two previous equations yields the following proposition.
\begin{proposition}
\label{prop}
For any $\gamma \in \Gamma$ and any $\u^\gamma \in \mathbb{R}^{|\Gamma|}$, if Notation~\ref{ass1} holds for Algorithm $\mathcal{A}^\gamma$ run on transformed observations and features $\breve{y}_{1:T}^\gamma$ and $\breve{\x}_{1:T}$, with the initialization parameter vector $\breve{\mathbf{s}}_0$, we have, for $T>0$,
\[
R_T^\gamma\big(\u^\gamma\big) \leq  \big({S^\gamma}\big)^2 \, \mathrm{B}\big(\breve{\x}_{1:T},\breve{y}_{1:T}^\gamma,\breve{\mathbf{s}}_0,\breve{\mathbf{u}}^\gamma \big) \quad \mathrm{where} \quad \breve{\mathbf{u}}^\gamma= \breve{\E}^{-1}(\u^\gamma -~\boldsymbol{\delta}^\gamma)/ S^\gamma.
\]
\end{proposition}
Throughout the section, without loss of generality and to simplify the notation, we now replace the features and observations with the standardized ones. Thus, we will write $y^\gamma_t$ for $\breve{y}_t^\gamma$, $\x_t$ for $\breve{\x}_t$ and so on.
Moreover, we make the following assumption on the boundedness of features and observations.
\begin{hyp} Boundedness assumptions.
\label{ass2}
For any $t>0$ and  any $\gamma \in \Gamma$ we assume that there is a constant $C>0$ such that
\begin{equation}
|y_t^\gamma |\leq C \quad \mathrm{and} \quad |x_t^\gamma |\leq C.
\label{eq:ass2}
\end{equation}
\end{hyp}
Some boudedness assumptions on features and observations are frequently required to establish theoretical guaranties. Here, the constant is common to all the nodes.
Practically, this assumption makes sens because of the previous transformations. As explained above, it centers and normalizes observations and features. Subsection~\ref{subsec:experiment_design} presents statistics on features and observation before and after standardization and indicates possible values of the constant $C$.\\

In the two next subsections, we introduce the aggregation algorithms we implemented in Section~\ref{sec:experiments}. 
We recall that, for any $\gamma \in \Gamma$, at a round $t$, the algorithm $\mathcal{A}^\gamma$ provides a weight vector $\u^\gamma_t$ and thus forecasts $y_t^\gamma$ with $\u^\gamma_t \cdot \x_t$.
In Subsection~\ref{subsec:linear_aggreg}, we consider a linear aggregation algorithm: there is no restriction on the computed weight vectors. 
In Subsection~\ref{subsec:convex_aggreg}, the two algorithms output convex combinations of features: the weight vectors are in the $|\Gamma|$-simplex denoted by $\Delta_{|\Gamma| }$.
However, there is no reason to consider such a restriction and this is why the last paragraph of the subsection presents a trick to extend the previous algorithms to output linear combinations of features for which the weight vectors are in a $L_1$-ball. Thus, there are no longer restrictions on the sum or the sign of the weights.

\subsection{Linear Aggregation: Sequential Non-Linear  Ridge Regression}
\label{subsec:linear_aggreg}
The first aggregation algorithm that we consider is the sequential non-linear ridge regression of \citet{Vovk01} and \citet{AzouryWarmuth2001}.
So, for any $\gamma \in \Gamma$, Algorithm $\mathcal{A}^\gamma$ refers here to Algorithm~\ref{algo:ridge} run with regularization parameter $\mathbf{s}_0 = \lambda$.
For any instance $t\geq 2$, this algorithm, chooses vectors $\u_t^\gamma$ as follow:
\begin{equation}
\u_t^\gamma \in \argmin_{\u^\gamma \in \mathbb{R}^d} \sum_{s=1}^{t-1} \big (y^\gamma_s - \u^\gamma \cdot\x_s \big)^2
 + \big(\u^\gamma \cdot \x_t \big)^2 + \lambda \|\u^\gamma\|^2.
 \label{ridge}
\end{equation}
The solution of this minimization problem is given by:
\[
\u_t^\gamma = \Big( \lambda \big( \mathbf{1}_{\{i=j\}} \big)_{(i,j)  \in \Gamma^2}+\sum_{s=1}^{t} \x_s\x_s^{\transp}\Big)^{\dag} \, \sum_{s=1}^{t-1} y_s^\gamma \x_s,
\]
where $A^{\dag}$ denotes the pseudo-inverse of the matrix $A$. 
Algorithm~\ref{algo:ridge} provides a sequential implementation of the solution of this convex minimization problem.
 \begin{algorithm}[h]	
	\caption{Non-Linear Sequential Ridge Regression}
	\begin{algorithmic}
		\label{algo:ridge} 
		\STATE \textbf{aim} 
		\STATE \quad Predict the time series $\big(y^\gamma_t\big)_{1\leq t\leq T}$ 
		\STATE \textbf{parameter} Regularization parameter $\lambda>0$
		\STATE \textbf{initialization} $\A_0 = \lambda \big( \mathbf{1}_{\{i=j\}} \big)_{(i,j)  \in \Gamma^2}$ and $\mathbf{b}_0=(0,\dots,0)^{\transp}$
		\FOR{$t = 1, \dots,T$}
		\STATE Update matrix $\A_t=\A_{t-1}+\x_t\x_t^{\transp}$ 
		\STATE Compute the vector $\u^\gamma_t=\A_t^{-1}\mathbf{b}_{t-1}$
		\STATE Output prediction $\hat{y}_t^\gamma=\u^\gamma_t \cdot \mathbf{x}_t$
		\STATE Update vector $\mathbf{b}_t=\mathbf{b}_{t-1}+y^\gamma_t\x_t$ 
		\ENDFOR
	\end{algorithmic}
\end{algorithm}
The above non-linear ridge regression is a penalized ordinary least-squares regression.
Since the features may be strongly correlated, the least squares estimator, $\u_t^\gamma \in \argmin_{\u^\gamma \in \mathbb{R}^d} \sum_{s=1}^{t-1} \big (y^\gamma_s - \u^\gamma \cdot\x_s \big)^2$, could lead to very large prediction if a new features vector belongs to an eigenspace of the empirical gram matrix associated to a small value. The regularization term $\lambda \|\u^\gamma\|^2$  ensures that eigenvalues of the empirical gram matrix are not too small.
We then add the regularization term $\big(\u^\gamma \cdot \x_t  \big)^2$ which is the last term of the cumulative prediction error $(y^\gamma_t - \u^\gamma \cdot\x_t \big)^2$ where we have replaced unknown $y^\gamma_t$ by our best guess $0$. 
It is known to improve the regret bound (see \citealp{Vovk01} and \citealp{pmlr-v98-gaillard19a}). 
In our case (standardized targets), it particularly makes sense because it biases predictions towards $0$; which, because of the standardization, biases aggregated predictions towards benchmark predictions.\\

Under the boundedness assumptions~\eqref{eq:ass2}, for any vector $\mathbf{u}^\gamma \in \mathbb{R}^{|\Gamma|}$, with the algorithm $\mathcal{A}^\gamma$ set to the non-linear ridge regression~\eqref{ridge} run with regularization parameter $\lambda$, Theorem~11.8 of the monograph \textit{Prediction, Learning, and Games} by \citet{Cesa-Bianchi2006} or Theorem~2 of \citet{pmlr-v98-gaillard19a} provide the following theoretical guaranties:
\[
R_T^\gamma(\u^\gamma)
\leq \lambda \|\mathbf{u}^\gamma\|^2 +  |\Gamma| C^2 \, \ln \bigg( 1+\frac{ C^2 T}{\lambda}\bigg)\,.
\]
So, for any $\U=\big( \u^1 | \dots | \u^{|\Gamma|}  \big) \in \mathcal{C}$, as $\big\|\U\big\|_F^2=\sum_{\gamma \in \Gamma}  \|\mathbf{u}^\gamma\|^2$, Theorem~\ref{gen-bound} ensures
\begin{align*}
R_T(\U) = \sum_{\gamma \in \Gamma} R_T^\gamma(\u^\gamma)
&\leq\lambda \big\|\U\big\|_F^2 + |\Gamma|^2 C^2 \, \ln \bigg( 1+\frac{ C^2 T }{\lambda}\bigg) = \mathcal{O}\big(|\Gamma|^2 \ln T\big) \,.
\end{align*}
That is, since the sequential non-linear ridge regression provides a logarithmic regret bound, Meta-algorithm~\ref{prot:1} achieves a bound of the same order.

\subsection{Convex Aggregation}
\label{subsec:convex_aggreg}
We focus here on uniform bounds and use notation introduced in Remark~\ref{rem:uniform_regret}.
The following two algorithms were initially designed to compete against the best feature.
Namely, for a node $\gamma \in \Gamma$, the Bernstein online aggregation (BOA, see~\citealp{wintenberger2017optimal}) and  polynomially weighted average forecaster with multiple learning rates (ML-Pol, see~\citealp{Gaillard2015}) provide some bound on the difference between
the cumulative prediction error $L^\gamma_T \defeq \sum_{t=1}^T \big(y^\gamma_t-\hat{y}^\gamma_t\big)^2$ of the strategy and  $\min_{i \in \Gamma} \sum_{t=1}^T \big(y^\gamma_t-~x^i_t\big)^2$. 
At each instance $t$, both strategies compute weight vector $\u^\gamma_t= \big( u^{\gamma \, i}_t \big)_{i \in \Gamma}$ based on historical data. 
These vectors are in the $|\Gamma|$-simplex, which we denote by $\Delta_{|\Gamma|}$.
For each feature $i \in \Gamma$, the weight $u^{\gamma \, i}_t$ is,
for BOA, an exponential function of a regularized cumulative prediction error of the feature $x^i_t$ and, for ML-Pol, a polynomial function of the cumulative prediction error of $x^i_t$.
However, by using gradients of prediction errors instead of the original prediction errors the average error of these algorithms may come close to $\min_{\u^\gamma \in \Delta_{|\Gamma|}} \frac{1}{T} \sum_{t=1}^T \big(y^\gamma_t-~\u^\gamma\cdot \x_t\big)^2$. 
This ``gradient trick" (see \citealp{Cesa-Bianchi2006}, Section 2.5) is presented in the next paragraph and is already integrated in the statements of the algorithms below. 
Moreover, for both algorithms, the computed weight vectors are in $\Delta_{|\Gamma|}$.
As we do not necessarily want to impose such a restriction, we use another trick, introduced by~\citet{KivinenExponentiatedGradientGradient1997a} and presented in the last paragraph.
It extends the class of comparison from the $|\Gamma|$-simplex to an $L_1$-ball of radius $\alpha$ denoted by $\mathcal{B}_\alpha \defeq~\big\{ \u^\gamma \in \mathbb{R}^{|\Gamma|} \, \big| \, \|\u\|_1 = \sum_{i \in \Gamma} | u^{\gamma \, i} | \leq \alpha \big\}$.  The aim is then to come close to the cumulative error $\min_{\u^\gamma \in \mathcal{B}_\alpha} \sum_{t=1}^T \big(y^\gamma_t-~\u^\gamma\cdot~\x_t\big)^2$.  

\paragraph{Gradient Trick: from the best feature to the best convex combination of features}
We consider an aggregation algorithm that takes as input, at any time step $t+1$, the previous prediction errors of each feature $(y_t^\gamma - x^i_t)^2$, for any $i \in \Gamma$, and that of the forecast outputted at $t$: $(y_t^\gamma - \hat{y}^\gamma_t)^2$. 
Although this trick generalizes to various prediction errors, we focus here to its application in our case, namely the quadratic prediction error.
We assume that the algorithm provides a bound on the quantity (see notation of Remark~\ref{rem:uniform_regret})
\[ 
R^\gamma_T \big(\delta_{|\Gamma|} \big) \defeq \sum_{t=1}^T (y_t^\gamma - \hat{y}^\gamma_t)^2  - \min_{i \in \Gamma} \sum_{t=1}^T (y_t^\gamma - x_t^i)^2 \,.
\]
where $ \delta_{|\Gamma|} \defeq \big\{ ( \boldsymbol{\delta}^i )_{i\in \Gamma}  \big\}$ is the set of canonical basis vectors (so we have $\boldsymbol{\delta}^i  \cdot \x_t = x_t^i$).
The gradient trick consists in giving, instead of the
prediction errors $(y_t^\gamma - \hat{y}^\gamma_t)^2$ and $(y_t^\gamma - x_t^i)^2$, for any $i \in \Gamma$, the pseudo prediction errors  functions defined below as input to algorithm $\mathcal{A}^\gamma$. 
This will provide a bound on the pseudo regret denoted by $\tilde{R}^\gamma_T \big(\delta_{|\Gamma|}\big)$.
We will prove that the same bound is achieved for the minimum of $R_T^\gamma(\u^\gamma)$ taken over $\u^\gamma \,  \in \Delta_{|\Gamma|}$ (and not only $\delta_{|\Gamma|}$), namely $ R_T^\gamma \big(\Delta_{|\Gamma|} \big)$. We detail here how the trick works and gives:
$$ R_T^\gamma \big(\Delta_{|\Gamma|} \big) \leq \tilde{R}^\gamma_T \big(\delta_{|\Gamma|}\big).$$
Let us fix a vector
$\u^\gamma={\big(u^{\gamma \, i}\big)}_{i\in \Gamma} \in \Delta_{|\Gamma|}$, we have for each $t=1,\dots,T$
\begin{align}
\label{eq:convexity}
\big( y_t^\gamma -\hat{y}_t^\gamma\big)^2 - \big( y_t^\gamma - \u^\gamma \cdot \x_t\big)^2   &=\big(2 y^\gamma_t - \hat{y}_t^\gamma - \u^\gamma \cdot \x_t\big)\big(\u^\gamma \cdot \x_t - \hat{y}_t^\gamma   \big) \nonumber \\
&=  2 \big(\hat{y}_t^\gamma - y_t^\gamma \big) \big(\hat{y}_t^\gamma - \u^\gamma \cdot \x_t\big) - \big(\hat{y}_t^\gamma - \u^\gamma \cdot \x_t\big)^2 \nonumber\\
&\leq  2 \big(\hat{y}_t^\gamma - y_t^\gamma \big) \big(\hat{y}_t^\gamma - \u^\gamma \cdot \x_t\big)\,.  
\end{align}
By plugging this equation into the definition of the regret, we obtain
\begin{align*}
 R_T^\gamma \big( \u^\gamma \big)  
 &   \defeq \sum_{t=1}^T (y_t^\gamma - \hat{y}^\gamma_t)^2  - \sum_{t=1}^T (y_t^\gamma - \u^\gamma\cdot \x_t)^2  \\
  & \stackrel{\mbox{\eqref{eq:convexity}}}{\leq} \sum_{t=1}^T 2 \big(\hat{y}_t^\gamma - y_t^\gamma \big) \big(\hat{y}_t^\gamma - \u^\gamma \cdot \x_t\big)
  =  \sum_{t=1}^{T} 2 \big(\hat{y}_t^\gamma - y_t^\gamma \big) \hat{y}_t^\gamma   - \, \sum_{t=1}^{T} \sum_{i \in \Gamma}  u^{\gamma \, i}
  2 \big(\hat{y}_t^\gamma - y_t^\gamma \big) x_t^i. 
\end{align*}
As $\u^\gamma$ belongs to the $|\Gamma|$-simplex (so $\forall i \in \Gamma$, $u^{\gamma \, i} \geq 0$ and $\sum_{i \in \Gamma} u^{\gamma \, i}=1$), we get:
\[
\sum_{i \in \Gamma}  u^{\gamma \, i} x_t^i \, \geq \, \min_{j \in \Gamma} x_t^j \sum_{i \in \Gamma} u^{\gamma \, i} =  \min_{j \in \Gamma} x_t^j
\]
Therefore, for any vector $\u^\gamma \in \Delta_{|\Gamma|} $, the regret  $R_T^\gamma \big( \u^\gamma \big) $ is bounded by
\[
 R_T^\gamma \big( \u^\gamma \big)  \leq \sum_{t=1}^T  2 \big(\hat{y}_t^\gamma - y_t^\gamma \big) \hat{y}_t^\gamma  - \min_{j \in \Gamma} \sum_{t=1}^T  2 \big(\hat{y}_t^\gamma - y_t^\gamma \big) x_t^j  \defeq \tilde{R}^\gamma_T \big(\delta_{|\Gamma|}\big).
\]
Thus, we now give the pseudo prediction errors associated with each feature $2 \big(\hat{y}_t^\gamma - y_t^\gamma \big) x_t^i$, with $i \in \Gamma$, and with the outputted forecast $ 2 \big(\hat{y}_t^\gamma - y_t^\gamma \big) \hat{y}_t^\gamma$ as input to algorithm $\mathcal{A}^\gamma$.
It provides a bound on the pseudo regret defined above $\tilde{R}^\gamma_T \big(\delta_{|\Gamma|}\big)$; and we get the same bound on $ R^\gamma_T \big(\Delta_{|\Gamma|}\big)$.
As a final note, we emphasize that the boundedness assumptions~\eqref{eq:ass2} allow to establish that pseudo prediction errors $2 (\hat{y}_t^\gamma - y_t^\gamma) x_t^i $ are bounded by $4C^2$.
Indeed, for any $(\gamma,i) \in \Gamma^2$, they ensure $|y_t^\gamma| \leq C$ and $|x_t^i|\leq C$.  In addition, as $\u^\gamma_t \in \Delta_{\Gamma}$, the output forecasts  $\hat{y}_t^\gamma=\u^\gamma_t \cdot \x_t$ are also bounded by:
 \[
 |\hat{y}_t^\gamma| = \bigg|\sum_{j \in \Gamma} u_t^{\gamma \, j} x_t^j \bigg | \leq \sum_{j \in \Gamma} u_t^{\gamma \, j} \big| x_t^j \big|\leq \sum_{j \in \Gamma} u_t^{\gamma \, j} C = C.
 \]
Hence, for any $i \in \Gamma$, the pseudo prediction error associated with feature $i$ satisfies
\begin{equation}
\label{eq:pseudo_loss_bound}
\big|2 \big( \hat{y}_t^\gamma - y_t^\gamma \big)x_t^i\big|\leq 4 C^2.
\end{equation}

\paragraph{Bernstein Online Aggregation}
\begin{algorithm}[ht!]
	\caption{\label{algo:boa} Fully adaptive Bernstein Online Aggregation (BOA) with gradient trick}
	\begin{algorithmic}
		\STATE \textbf{aim} 
		\STATE \quad Predict the time series $\big(y^\gamma_t\big)_{1 \leq t \leq T}$ 
		\STATE \textbf{parameter} Bound on pseudo prediction errors $E$:
		\STATE \quad for any $t=1,\dots,T$ and any $i \in \Gamma$, $|2 (\hat{y}_t^\gamma - y_t^\gamma) x_t^i| \leq E$ 
		\STATE \textbf{initialization} 
		\STATE \quad $\u^\gamma_1=(1/{|\Gamma|},\dots,1/{|\Gamma| })$
		\STATE \quad $\hat{y}_1^\gamma=\u^\gamma_1 \cdot \mathbf{x}_1$
		\STATE \quad For $i \in \Gamma$, $\tilde{R}_0^{\gamma \, i}=0$ 
		\STATE \quad For $i \in \Gamma$, $\eta_0^{\gamma \, i}=0$
		\FOR{$t = 1, \dots,T-1$} 
		\STATE  For each $i \in \Gamma$, update the cumulative quantity $\tilde{Q}^{\gamma \, i}$ for feature $i$
		\[\tilde{Q}_t^{\gamma \, i}=
		 \tilde{Q}_{t-1}^{\gamma \, i} +  \tilde{r}^{\,\gamma \, i}_t \big( 1+ \eta_{t-1}^{\gamma \, i} \tilde{r}^{\,\gamma \, i}_t\big) \quad \mathrm{where} \quad  \tilde{r}^{\,\gamma \, i}_t \defeq 2 (\hat{y}_t^\gamma - y_t^\gamma) ( \hat{y}_t^\gamma  - x_t^i)
 \]
		\STATE  For each $i \in \Gamma$, compute the learning rate
		\[\eta_t^{\gamma \, i}=\min \Bigg\{ \frac{1}{2E}, \sqrt{\frac{\log |\Gamma| }{ \sum_{s=1}^t \big(  \tilde{r}^{\,\gamma \, i}_t  \big)^2}} \Bigg\}\]
		\STATE Compute the weight vector $\u^\gamma_{t+1}=(u^{\gamma \, i}_{t+1})_{i \in \Gamma }$ defined as 
		\[ u^{\gamma \, i}_{t+1}=\frac{\exp \big( \eta_t^{\gamma \, i} \tilde{Q}^{\gamma \, i}_t\big)}{ \sum_{j \in \Gamma} \exp \big( \eta_t^{\gamma \, j}  \tilde{Q}^{\gamma \, j}_t\big)}\]
		\STATE Output prediction $\hat{y}_{t+1}^\gamma= \u^\gamma_{t+1} \cdot \x_{t+1} =\sum_{i \in \Gamma}u^{\gamma \, i}_{t+1} x_{t+1}^i$
		\ENDFOR
	\end{algorithmic}
\end{algorithm}
\citet{wintenberger2017optimal} introduces an aggregation procedure called Bernstein Online Aggregation for which weights are exponential function of the cumulative prediction errors.
Algorithm~\ref{algo:boa} describes this strategy combined with this gradient trick.
Let us fix a node $\gamma \in \Gamma$ and set $\mathcal{A}^\gamma$ to Algorithm~\ref{algo:boa} which takes as input the bound $E$ on pseudo prediction errors ($E=4 C^2$ is a suitable choice): \[\forall t=1,\dots,T, \, \forall i \in \Gamma, \quad 2 (\hat{y}_t^\gamma - y_t^\gamma) x_t^i  \leq E.\]
Theorem 3.4 of \citealp{wintenberger2017optimal} ensures that
\begin{align}
R^\gamma_T(\Delta_{|\Gamma|} )& \leq  \sqrt{T+1} E  \Bigg( \frac{\sqrt{2 \ln |\Gamma| }}{\sqrt{2}-1} + \frac{\ln (1 +2^{-1}\ln T)}{\sqrt{\ln |\Gamma| }} \Bigg) \nonumber  + \, E \Big( 2 \ln |\Gamma|  + 2 \ln(1+2^{-1} \ln T) +1  \Big) \nonumber \\ 
&\lesssim \, \mathcal{O}\Big(\sqrt{T}\ln \ln T\Big).
\label{eq:boa_regret}
\end{align}
Thanks to Equation~\eqref{eq:pseudo_loss_bound}, we replace $E$ by $4 C^2$ in Equation~\eqref{eq:boa_regret} and we get, for each node $\gamma \in \Gamma$, an upper bound on  $R^\gamma_T(\Delta_{|\Gamma|})$. 
By applying Theorem~\ref{gen-bound}, we obtain the following uniform regret bound:
\[R_T(\Delta_{|\Gamma|} )\lesssim \, \mathcal{O}\Big(|\Gamma|\sqrt{T}\ln \ln T\Big),\]
which is of order  $\sqrt{T}$ (up to poly-logarithmic terms). 

\paragraph{Polynomially Weighted Average Forecaster}
\begin{algorithm}[ht!]
	\caption{\label{algo:ml_pol}  Polynomially weighted average forecaster with Multiple Learning rates (ML-Pol) and gradient trick}
	\begin{algorithmic}
		\STATE \textbf{aim} 
		\STATE \quad Predict the time series $\big(y^\gamma_t\big)_{1 \leq t \leq T }$ 
		\STATE \textbf{parameter} Bound on pseudo prediction errors $E$:
		\STATE \quad for any $t=1,\dots,T$ and any $i \in \Gamma$, $|2 (\hat{y}_t^\gamma - y_t^\gamma) x_t^i| \leq E$ 
		\STATE \textbf{initialization} 
		\STATE \quad $\u^\gamma_1=(1/{|\Gamma| },\dots,1/{|\Gamma| })$
		\STATE \quad $\hat{y}_1^\gamma=\u^\gamma_1 \cdot \mathbf{x}_1 $
		\STATE \quad For $i \in \Gamma$, $\tilde{R}_0^{\gamma \, i}=0$ 
		\STATE \quad For $i \in \Gamma$, $\eta_0^{\gamma \, i}=0$
		\FOR{$t = 1, \dots,T-1$}
		\STATE  For each $i \in \Gamma$, update the cumulative pseudo-regret of feature $i$
		\[\tilde{R}_t^{\gamma \, i}=\tilde{R}_{t-1}^{\gamma \, i} + \tilde{r}^{\, \gamma \, i}_t \quad \mathrm{where} \quad  \tilde{r}^{\, \gamma \, i}_t \defeq 2 (\hat{y}_t^\gamma - y_t^\gamma) ( \hat{y}_t^\gamma  - x_t^i)
		\]
		\STATE  For each $i \in \Gamma$, compute the learning rate
		\[ \eta_t^{\gamma \, i}=\bigg( E+ \sum_{s=1}^t \big(  \tilde{r}^{\,\gamma \, i}_t\big)^2 \bigg)^{-1}\]
		\STATE Compute the weight vector $\u^\gamma_{t+1}=(u^{\gamma \, i}_{t+1})_{i \in \Gamma}$  defined as
		\[u^{\gamma \, i}_{t+1}=\frac{\eta_t^{\gamma \, i} \big(\tilde{R}^{\gamma \, i}_t)_+}{ \sum_{j \in \Gamma} \eta_t^{\gamma \, j} ( \tilde{R}^{\gamma \, j}_t \big)_+}\]
		\STATE Output prediction $\hat{y}_{t+1}^\gamma= \u^\gamma_{t+1} \cdot \x_{t+1} =\sum_{i \in \Gamma}u^{\gamma \, i}_{t+1} x_{t+1}^i$
		\ENDFOR
	\end{algorithmic}
\end{algorithm}
\citet{pmlr-v35-gaillard14} consider an aggregation method based on weights that are polynomial functions of the cumulative prediction errors.
We use this procedure combined with the gradient trick and present it in Algorithm~\ref{algo:ml_pol}. In this description, $(\mathbf{x})_+$ denotes the vector of non-negative parts of the components of $\mathbf{x}$.
With the same notation as in the previous paragraph, for any node $\gamma \in \Gamma$, Theorem 5 of 
\citet{pmlr-v35-gaillard14} provides the following regret bound:
\begin{equation}
R^\gamma_T(\Delta_{|\Gamma|} )   \leq 
E \sqrt{|\Gamma| (T+1) \big(1 + \ln (1+T) \big)}.
\label{eq:mlpol_regret}
\end{equation}
With $E\leq 4C^2$ and by applying Theorem~\ref{gen-bound}, we obtain an upper bound on the uniform regret $R_T(\Delta_{|\Gamma|}) $, which is also of order  $\sqrt{T}$ (up to poly-logarithmic terms):
\begin{align*}
R_T(\Delta_{|\Gamma|}) &\leq  4C^2 |\Gamma| \sqrt{|\Gamma|  (T+1) \big(1 + \ln (1+T) \big)} \\
&\lesssim \mathcal{O}\Big(|\Gamma|^{3/2}\sqrt{ T \ln T}\Big).
\end{align*}

\subsection{A scheme to extend the class of comparison from the simplex to an $L_1$-ball}

For the previous two algorithms, we obtained an upper bound on $R_T^\gamma (\Delta_{|\Gamma|})$. However, there is no reason for the best linear combination of features to be convex.
Algorithm~\ref{algo:trick} presents a trick introduced by \citet{KivinenExponentiatedGradientGradient1997a} which extends the class of comparison from the $|\Gamma|$-simplex to an $L_1$-ball of radius $\alpha>0$ denoted by $\mathcal{B}_\alpha$ and provides a bound on $R_T^\gamma (\mathcal{B}_\alpha)$. Let us fix a node $\gamma \in \Gamma$. The trick consists in transforming, at each round $t$, the feature vector $\x_t$ into the $2|\Gamma|$-vector $\bar{\x}_t = (\alpha \x_t | -\alpha \x_t)$. The algorithm~$\mathcal{A}^\gamma$ is then run with these new features and it outputs the weight vector $\bar{\u}_t^\gamma \in \Delta_{2|\Gamma|}$. 
Finally, a $|\Gamma|$-vector $\u_t^\gamma \in \mathcal{B}_\alpha$ is computed from $\bar{\u}_t^\gamma$ to provide the forecast $\u_t^\gamma  \cdot \x_t = \bar{\u}_t^\gamma  \cdot \bar{\x}_t$.
We will actually see that we may associate any $|\Gamma|$-vector $\u \in \mathcal{B}_\alpha$ with a vector $\bar{\u} \in \Delta_{2|\Gamma|}$  such as $\bar{\u}  \cdot \bar{\x}_t = \u \cdot \x_t$; the trick actually defines a surjection from $\Delta_{2|\Gamma|}$ to $\mathcal{B}_\alpha$.
Thus, to compete against the best linear combination of features in $\mathcal{B}_{\alpha}$, it is enough to compete against the best convex combination of features $\bar{\x}_t$ in a lifted space (which we may achieve, thanks to algorithm~$\mathcal{A}^\gamma$). We now give all the details on how this trick works and indicate its impact on the stated regret bounds. The following lemma introduces the surjection from $\Delta_{2|\Gamma|}$ to $\mathcal{B}_\alpha$, which is used in Algorithm~\ref{algo:trick}.
\begin{lemma}
	\label{lem:surj}
For any real $\alpha>0$, the following function $\psi$ is a surjection from $\Delta_{2 |\Gamma|}$ to $\mathcal{B}_\alpha$:
	\begin{displaymath}
\psi:
  \begin{array}{lcl}
    \Delta_{2 |\Gamma|} & \longrightarrow & \mathcal{B}_\alpha \\
   \bar{\u} = \big( \bar{\u}^+ \, | \, \bar{\u}^- \big) & \longmapsto &  \alpha (\bar{\u}^+ - \bar{\u}^-),\\
  \end{array}
\end{displaymath}
where the vector $\bar{\u} \in \Delta_{2 |\Gamma|}$ is decomposed in the two $|\Gamma|$-vectors $\bar{\u}^+$ and $\bar{\u}^-$, which correspond to the $|\Gamma|$ first and the $|\Gamma|$ last coefficients of $\bar{\u}$, respectively. 
\end{lemma}
By running Algorithm~$\mathcal{A}^\gamma$ with transformed features $\bar{\x}_t \defeq \big(\alpha \x_t \, | \, - \alpha \x_t\big)$ and parameter $s^\gamma_0$ (which provides weight vectors $\bar{\u}_t^\gamma$), we get the bound
\begin{equation*}
R_T^\gamma(\Delta_{2|\Gamma|}) \defeq \sum_{t=1}^T \big( y_t^\gamma - \bar{\u}^\gamma_t \cdot \bar{\x}_t \big)^2- \min_{\bar{\u}^\gamma \in \Delta_{2 |\Gamma|}} \sum_{t=1}^T \big( y_t^\gamma - \bar{\u}^\gamma \cdot \bar{\mathbf{x}}_t   \big)^2 \leq \mathrm{B}(\bar{\x}_{1:T},y_{1:T}^\gamma,\mathbf{s}^\gamma_0,\bar{\u}^\gamma).
\end{equation*}
For any instance $t=1,\dots,T$, and for any $\bar{\u}^\gamma \in \Delta_{2 |\Gamma|}$, we obtain the equality of the two scalar products $\bar{\u}^\gamma \cdot \bar{\x}_t $ and $\psi(\u^\gamma) \cdot \x_t$:
\begin{align*}
	\bar{\u}^\gamma \cdot \bar{\x}_t =
	\big(  {\bar{\u}}^{\gamma+} \, | \, {\bar{\u}}^{\gamma-} \big) \cdot \big( \alpha \x_t \, | \, - \alpha \x_t \big) 
	= \alpha \big( {\bar{\u}}^{\gamma+}  -  {\bar{\u}}^{\gamma-}  \big) \cdot \x_t = \psi (\bar{\u}^\gamma) \cdot \x_t\,.
\end{align*}
Lemma~\ref{lem:surj} implies that for any $\u^\gamma \in \mathcal{B}_\alpha$, there is at least one vector $\bar{\u}^\gamma \in \Delta_{2 |\Gamma|}$ such that $\psi (\bar{\u}^\gamma)=\u^\gamma$ and we get the equality:
$$\min_{\u^\gamma \in \mathcal{B}_\alpha} \sum_{t=1}^T \big( y_t^\gamma - \u^\gamma \cdot \x_t   \big)^2 =
\min_{\bar{\u}^\gamma \in \Delta_{2 |\Gamma|}} \sum_{t=1}^T \big( y_t^\gamma - \psi(\bar{\u}^\gamma) \cdot {\x}_t   \big)^2 =
\min_{\bar{\u}^\gamma \in \Delta_{2 |\Gamma|}} \sum_{t=1}^T \big( y_t^\gamma - \bar{\u}^\gamma \cdot \bar{\x}_t   \big)^2 \,.
 $$
So with, for any instance $t=1,\dots,T$, $\u_t^\gamma \defeq \psi(\bar{\u}_t^\gamma)$, we obtain
\[
R_T^\gamma(\mathcal{B}_\alpha) \defeq \sum_{t=1}^T \big( y_t^\gamma - \u_t^\gamma \cdot \x_t \big)^2- \min_{\u^\gamma \in \mathcal{B}_\alpha} \sum_{t=1}^T \big( y_t^\gamma - \u^\gamma \cdot \mathbf{x}_t   \big)^2 = R_T^\gamma(\Delta_{2|\Gamma|})\,.
\]
This equality provides a bound on $R^\gamma_T(\mathcal{B}_{\alpha})$ when predictions are $\hat{y}_t^\gamma = \psi^{-1}(\bar{\u}^\gamma_t) \cdot \x_t = \u_t^\gamma \cdot \x_t$.
With this trick, the previous bounds~\eqref{eq:boa_regret} and~\eqref{eq:mlpol_regret} 
are still true by replacing $|\Gamma|$ (the dimension of the features $\x_t$) by $2|\Gamma| $ (the dimension of the new features $\bar{\x}_t$) and the bound $E$ (previously equals to $4 C^2$) by $2\alpha(\alpha+1)C^2$ (the bound on the new pseudo prediction errors are calculated below):
\begin{align*}
R_T(\mathcal{B}_{\alpha}) &\leq \,  |\Gamma| \Bigg(  2\,\alpha\,(\alpha+1)\,C^2 \,   \sqrt{T+1}  \bigg( \frac{\sqrt{2 \ln |\Gamma| }}{\sqrt{2}-1} + \frac{\ln (1 +2^{-1}\ln T)}{\sqrt{\ln |\Gamma| }} \bigg)\\ 
&\qquad  \qquad \qquad  + \, 2\,\alpha\,(\alpha+1)\,C^2 \,   \big( 2 \ln |\Gamma| + 2 \ln(1+2^{-1} \ln T) +1  \big) \Bigg)
& \textrm{for BOA} \\
&\leq  \, 2\,\alpha\,(\alpha+1)\,C^2 \,  |\Gamma|  \sqrt{|\Gamma| (T+1) \big(1 + \ln (1+T) \big)} & \textrm{for ML-Poly}.
\end{align*}
The complete online algorithm leading to these bounds is summarized in Algorithm~\ref{algo:trick}.

\paragraph{Bound on new pseudo prediction errors}
Since boundedness assumptions~\eqref{eq:ass2} hold, the transformed features $\bar{\x}_t^\gamma$ are bounded by $\alpha C$.
Moreover, $\bar{\u}_t^\gamma \in \Delta_{2|\Gamma|}$ implies $\big\| \bar{\u}_t^\gamma \big\|_1 = 1$, so we get
\[
 \big|\hat{y}_t^\gamma\big| 
 = \big|\bar{\u}_t^\gamma \cdot \bar{\x}_t\big|
 \leq \big\| \bar{\u}_t^\gamma \big\|_1 \big\| \bar{\x}_t \big\|_\infty
 =\alpha C\,.
\]
Moreover, as the observations are still bounded by $C$,  we have $|y^\gamma_t -  \hat{y}^\gamma_t| \leq |y^\gamma_t | +|  \hat{y}^\gamma_t| \leq \big(\alpha+1\big)C $ and we obtain a bound on the pseudo prediction errors:
\[
\big|\tilde{\ell}_t^\gamma(\bar{\x}_t)\big|
=\big\|2\big(\hat{y}_t^\gamma-y_t^\gamma\big) \bar{\x}_t\big\|_\infty
\leq 2 \alpha(1+\alpha)C^2\,.
\]

\begin{algorithm}[h!]
	\caption{\label{algo:trick} Scheme for on-line linear regression.}
	\begin{algorithmic}
		\STATE \textbf{input} Algorithm $\mathcal{A}^\gamma$ and bound on the weight vectors $\alpha > 0$
		\FOR{$t = 1, \dots,T$}
		\STATE Get the feature vector $\x_t$ and denote (where $|$ is the concatenation operator between vectors)
		\[
		\bar{\x}_t \defeq \big(\alpha \x_t \, | \, - \alpha \x_t\big) \in \mathbb{R}^{2|\Gamma|}
		\]
		\STATE Run algorithm $\mathcal{A}^\gamma$ on node $\gamma$  with $\bar{\x}_t$ and get the weight vector $\bar{\u}^\gamma_t = \big( \bar{\u}_t^{\gamma +} \, | \, \bar{\u}_t^{\gamma -} \big)$
		\STATE Output the weight vector 
		$\u^\gamma_t = \alpha \big( \bar{\u}_t^{\gamma +} - \bar{\u}_t^{\gamma -} \big)$
		and predicts $\hat{y}_t^\gamma={\u}^\gamma_t \cdot{\x}_t$
		\ENDFOR
	\end{algorithmic}
\end{algorithm}

\begin{proof}[Proof of Lemma~\ref{lem:surj}]
	Denoting respectively by $(\u)_+$ and $(\u)_-$ the non-negative and non-positive parts of any vector $\u$ and by $\textbf{1}_{|\Gamma|}$ the vector of size $|\Gamma|$ of which all coordinates are $1$, we introduce the inverse function $\psi^{-1}$:
	\begin{displaymath}
\psi^{-1}:
  \begin{array}{lcl}
   \mathcal{B}_\alpha   & \longrightarrow &  \Delta_{2 |\Gamma|} \\
 \u & \longmapsto & \frac{1}{\alpha} \Bigg( \, \frac{ \alpha - \|\u\|_1}{2  |\Gamma|} \textbf{1}_{|\Gamma|} +(\u)_+   \, \bigg| \,   \, \frac{ \alpha - \|\u\|_1}{2 |\Gamma|} \textbf{1}_{|\Gamma|}  + (\u)_- \Bigg)\,.
  \end{array}
\end{displaymath}
First we will show that function images are in the right sets, meaning that for any $\u \in \mathcal{B}_\alpha$, $\psi^{-1}(\u) \in~\Delta_{2|\Gamma|}$ and for any $\bar{\u} \in \Delta_{2|\Gamma|}$, $\psi(\bar{\u}) \in \mathcal{B}_\alpha$. Secondly, we obtain the surjectivity of $\psi$ by proving that for any $\u \in \mathcal{B}_\alpha, \psi(\psi^{-1}(\u)) = \u$.

\paragraph{Proof that for any $\u \in \mathcal{B}_\alpha$, $\psi^{-1}(\u) \in \Delta_{2|\Gamma|}$} We set $\u \in \mathcal{B}_\alpha$. 
By definition for any $i \in \Gamma$, $(u^i)_{\pm} \geq 0$ and as $\u \in \mathcal{B}_\alpha$, $(\alpha - \|\u\|_1)/(2  |\Gamma|) \geq 0$. So, all the coefficients of $\psi^{-1}(\u)$ are non-negative.  
Since $\sum_{i \in \Gamma}  {(u^{i})}_+ + {(u^{i})}_-  =~\sum_{i \in \Gamma} |u^{ i}|= \|\u\|_1$,  the sum of the coefficients of the vector  $\psi^{-1}(\u)$ equals $1$:
	\begin{equation*}
	\sum_{i \in \Gamma}  \big(\psi^{-1}(\u)\big)^{i+} +\, \big(\psi^{-1}(\u)\big)^{i-} 
	= \frac{1}{\alpha}  \sum_{i \in \Gamma}  \Bigg( \big(u^{i}\big)_+ + \big(u^{i}\big)_- + \frac{\alpha- \|\u\|_1}{ |\Gamma|} \Bigg)
	= \frac{1}{\alpha} \big( \|\u\|_1 + \alpha - \|\u\|_1  \big) =1.
	\end{equation*}
	and thus $\bar{\u} = \psi(\u) \in \Delta_{2|\Gamma|}$.

\paragraph{Proof that for any $\bar{\u} \in \Delta_{2|\Gamma|}$, $\psi(\bar{\u}) \in \mathcal{B}_\alpha$} With $\bar{\u} = \big( \bar{\u}^+ \, | \, \bar{\u}^- \big) \in \Delta_{2|\Gamma|}$,
	using that all the coefficients of $\bar{\u}$ are non-negative and that their sum equals $1$ that is $\big\| \bar{\u}\big\|_1=1$, we get
	\[ 
	\|\psi(\bar{\u})\|_1 \defeq \big\| \alpha \bar{\u}^{+} - \alpha \bar{\u}^{-}  \big\|_1 \leq \alpha \big\| \bar{\u}^{+} \big\|_1  +\alpha \big\| \bar{\u}^{-} \big\|_1  =\alpha \big\| \bar{\u} \big\|_1 = \alpha.
	\]
	\paragraph{Proof that for any $\u \in \mathcal{B}_\alpha$, $\psi\big(\psi^{-1}(\u)\big)=\u$} 
	\[
	\psi\big(\psi^{-1}(\u)\big)= \frac{ \alpha - \|\u\|_1}{2  |\Gamma|} \textbf{1}_{|\Gamma|} +(\u)_+ - \frac{ \alpha - \|\u\|_1}{2  |\Gamma|} \textbf{1}_{|\Gamma|} - (\u)_- = \u\,.\qedhere
	\]
\end{proof}

\section{Experiments}
\label{sec:experiments}

Our application relies on electricity consumption data of a large number of households to which we have added meteorological data (see Subsection~\ref{subsec:data}).
Non-temporal information (sociological type, region, type of heating fuel and type of electricity contract) on the households is also provided.
From these temporal and non-temporal data, we dispatch the households into clusters thanks to the methods presented in Subsection~\ref{subsec:clustering}.
We describe the experiments and analyze the results in Subsections~\ref{subsec:experiment_design} and~\ref{subsec:results}.

\subsection{The Underlying Real Data Set}
\label{subsec:data}
The project ``\textit{Energy Demand Research Project\footnote{{https://www.ofgem.gov.uk/gas/retail-market/metering/transition-smart-meters/energy-demand-research-project}}}'', managed by Ofgem on behalf of the UK Government, was launched in late $2007$ across Great Britain (see \citealp{AECOMEnergyDemandResearch2018} and \citealp{schellong2011energy}). Power consumptions of approximately 18,000 households with smart-type meters were collected at half-hourly intervals for about two years.
We detail below how we select only the consumption of 1,545 households over
the period from April $20$, $2009$ to July $31$, $2010$ -- \citet{taieb2017coherent}, who used the same data, performed similiar pre-processing in their experiments. 
Four non-temporal variables are associated with each household: the Region (the initial data set provides the level-$4$ NUTS\footnote{\textit{Nomenclature des Unit\'es Territoriales Statistiques} (nomenclature of territorial units for statistics)} codes but we consider larger subdivisions -- from 150,000 to 800,000 inhabitants -- and associate each household with its level-$3$ code), 
 the Acorn category value (an integer between $1$ and $6$ associated with an United Kingdom’s population demographic type -- this segmentation was developed by the company CACI Limited), the type of heating fuel (``electricity" or ``electricity and gas") and the contract type (``Standard" or ``Time of Use tariff" for households containing an electricity meter with a dynamic time of use tariff) for each household.  
In a first data cleaning step, we removed
households with more than $5$ missing consumption records over the period April $20$, $2009$ to July $31$, $2010$ (around $1,600$ households are thus kept) -- the remaining missing consumption data points are imputed by a linear interpolation. 
Among the various clusterings of the households we consider in our experiments, three of them rely on three qualitative variables: ``Region", ``Tariff" and 
 ``Fuel + Tariff'' (which is based on both the heating fuel type and the contract type).
If one of the values of these qualitative variables had fewer than 20 occurrences, we have removed from the data set the households associated with that value.
The final data set then contains the electrical consumption records of the 1,545 remaining households.
From now on, we will denote by $\mathcal{I}$ the set of households
and by $(y_{i\,t})_{1-T_0 \leq t \leq T}$ the time series of the half-hourly power consumption of the $i \in \mathcal{I}$ household.
Finally, we added the temperature, visibility and humidity for each region from the NOAA\footnote{National Oceanic and Atmospheric Administration, https://www.noaa.gov/} data: we selected a weather station (with records available over the considered period) in each region and linearly interpolated the meteorological data to get 48 measurements per day (compared to 8 initially).
Table~\ref{table:variables} sums up the available variables of our data set and gives their range.
\begin{table}[t]
\footnotesize
	\label{table:variables}
	\centering
	\begin{tabular}{lll}
		\midrule
		\textbf{Variable} & \textbf{Description} & \textbf{Range / Value} \\
		\midrule
		\midrule
		Acorn  & Acorn category value & From 1 to 6\\
		Region & UK NUTS of level $3$ & UK- H23, -J33, -L15, -L16, -L21,  -M21, or -M27 \\ 
		Fuel &  Type of heating fuel & Electricity (E) or Electricity and Gas (EG)\\
		Tariff &  Contract type & Standard (Std) or Time of Use tariff (ToU)\\
		Temperature &  Air temperature & From $-20\degree$ to $30\degree$\\
		Visibility &  Air visibility & From $0$ to $10$ (integer)\\
		Humidity &  Air humidity percentage & From $0\%$ to $100\%$\\
		Date &  Current time & From  April $20$, $2009$ to  July $31$, $2010$ (half-hourly)\\
		Consumption & Power consumption & From $0.001$ to $900$  kWh\\
		\hdashline
		Fuel + Tariff & Cross of Fuel and Tariff variables & ``E - Std'', ``EG - Std'', ``E - ToU'' or ``EG - ToU'' \\
		Half-hour & Half-hour of the day & From $1$ to $48$ (integer) \\
		Day & Day off the week & From $1$ (Monday) to $7$ (Sunday) (integer) \\
		Position in the year & Linear values & From $0$ (Jan 1, 00:00) to $1$ (Dec 31, 23:59)\\
		 Smoothed temperature & Smoothed air temperature &From $-20\degree$ to $30\degree$ \\
		\midrule
	\end{tabular}
	\caption{\normalsize Summary of the variables provided and created for each household of the data set.}
\end{table}

\subsection{Clustering of the Households}
\label{subsec:clustering}
We present, in Paragraphs~\ref{subsub:random} to~\ref{subsub:nmf}, three methods to cluster the households and we compare them in the last paragraph of this subsection. 
After choosing a segmentation (or two crossed segmentations), we only consider, for each cluster, the aggregated consumption of its households. 
Thus, for any subset $\gamma \subset \mathcal{I}$, we compute the time series $y^\gamma_t \defeq \sum_{i \in \gamma} y_{i \, t}$ that we want to forecast and once clusterings are chosen, we never consider individual power consumption.

\subsubsection{Random Clustering}
\label{subsub:random}
We first consider the simplest way to cluster households: the segmentation is built randomly. In the experiments of Subsection~\ref{subsec:results}, the number of clusters varies from $4$ to $64$.
As an example here, we consider $4$ clusters and we randomly assign a number between $1$ and $4$ to each household and obtain the weekly profiles plotted in Figure~\ref{fig:conso_week_random}. 
In the following, we will call ``Random ($k$)", a segmentation of $k$ clusters built randomly.
Naturally, the curves are similar and the clusters are therefore rather homogeneous.

\subsubsection{Segmentation Based on Qualitative Household Variables}
The second approach consists in grouping households according to the provided non-temporal information.
 We consider the natural segmentations  ``Region", ``Acorn" and ``Fuel + Tariff'' based on the corresponding qualitative variables and we plot the weekly profile of each cluster on Figures~\ref{fig:conso_week_nuts},~\ref{fig:conso_week_acorn} and~\ref{fig:conso_week_fueltout}. 
Regions have an impact on the consumption profile: the evening consumption peak time varies by location. Moreover, consumption of the Wales regions (UKL15, UKL16 and UKL21) is lower than that of the other regions (see Figure~\ref{fig:conso_week_nuts}).
In the Acorn classification, the lower the value, the richer the household, thus Figure~\ref{fig:conso_week_acorn} shows that wealthiest households
consume the most (as expected).
Finally, the type of heating fuel does not seem to have a significant impact on the weekly consumption profile (although we have observed that when the heating is partly gas, the consumption is slightly lower and in winter, it is less sensitive to the temperature drops). 
Similarly, it seems that the type of contract does not influence the consumption profiles. Peak consumption in the evening is however less important for a dynamic time of fuse tariff than for the standard tariff. It should be noted that since time slots of prices may change from day to day, it is difficult to quantify here the impact of the tariff, as we are only showing average consumption profiles. 

\subsubsection{Clustering Based on Non-Negative Matrix Factorization and k-Means Method}
\label{subsub:nmf}
The last method relies on an historical individual time series of household power consumption (April $20$, $2009$ to April $20$, $2010$).
We propose a method to extract from these time series a low number -- denoted by $r$ -- of combined household characteristics and to use them to build relevant clusterings. 
The diagram below sums up the steps of the procedure described here quickly. We then further detail them one by one.
The $|\mathcal{I}|$ historical times series $\big(y_{i \, t}\big)_{1-T_0\leq t \leq 0}$ are firstly re-scaled and  gathered into a matrix $\Y_0\in \mathcal{M}_{|\mathcal{I}|\times T_0}$. 
We then reduce the dimension of data with a non-negative matrix factorization (NMF): we approximate $\Y_0$ by $\hat{\W} \, \hat{\H}$, where $\hat{\W}$  and  $\hat{\H}$ are $|\mathcal{I}| \times r$ and $r \times T_0$-non-negative matrices, respectively. 
As soon as this approximation is good enough, line $i$ of the matrix $\hat{\W}$ is sufficient to reconstruct the historical time series of household $i$ (with the knowledge of matrix $\hat{\H}$ - which is not used for the clustering). 
Thus, we assign, to each household, $r$ characteristics: the lines of $\hat{\W}$. 
After a re-scaling step -- to give the same importance to each of those characteristics -- we get the $r$-vectors $(\w_i)_{i \in \mathcal{I}}$.
With this low-dimension representation of households in $\mathbb{R}^r$, we use $k$-means clustering algorithm in $\mathbb{R}^r$ to provide the $k$ clusters $C_1,\dots,C_k$ and we write ``NMF ($k$)" for such a clustering. 
\small
\begin{center}
\tikzstyle{init} = [pin edge={to-,thin,black}]
\begin{tikzpicture}
    \node[draw ,text width=3cm, text centered,minimum height=3em] (a) {Re-scaling and gathering time series in a matrix};
     \node[draw ,text width=3cm,text centered,minimum height=3em] (c) [right of=a,node distance=6cm] {Low rank approximation (NMF): $\Y_0 \approx \hat{\W} \, \hat{\H} $};
    \node (b) [left of=a,node distance=5.5cm, coordinate] {a};
    \node [coordinate] (g) [right of=c, node distance=2cm]{};
    \path[->] (b) edge node [below] {Historical time series} (a);
    \path[->] (b) edge node [above] {$\Big\{ \big(y_{i \, t}\big)_{1-T_0\leq t \leq 0} \, | \, i \in \mathcal{I} \Big\}$} (a);
    \path[->] (a) edge node [above]  {$\Y_0\in \mathcal{M}_{|\mathcal{I}|\times T_0}$} (c);
    \path[->] (a) edge node [above]  { } (c);
    \path[-] (c) edge node [above] {} (g);
    \path[-] (c) edge node [above] { } (g);
    \node (i) [below of=g,node distance=1.2cm, coordinate] {};
    \path[-] (g) edge node [above] { } (i);
    \path[-] (g) edge node [below] { } (i);
    \node (j) [left of=i,node distance=11cm, coordinate] {};
    \path[-] (i) edge node [above] { } (j);
    \path[-] (i) edge node [below] { } (j);
    \node (h) [below of=j,node distance=1.2cm, coordinate] {j};
    \path[-] (j) edge node [above] { } (h);
    \path[-] (j) edge node [below] { } (h);
    \node[draw ,text width=3cm, text centered,minimum height=3em] (d) [right of=h,node distance=4cm] {Extracting and re-scaling characteristic vectors};
    \path[->] (h) edge node [above] { $\hat{\W} \in \mathcal{M}_{|\mathcal{I}|\times r}$} (d);
    \path[->] (h) edge node [below] { } (d);
    \node[draw ,text width=3cm, text centered,minimum height=3em] (e) [right of=d,node distance=6cm] {$k$-means clustering};
    \draw[->] (d) edge node [above] {$\Big\{ \w_i \in \mathbb{R}^r | \, i \in \mathcal{I}\Big\}$} (e) ;
    \draw[->] (d) edge node [below] {} (e) ;
    \node [coordinate] (end) [right of=e, node distance=4cm]{};
    \draw[->] (d) edge node [above] {} (e) ;
    \draw[->] (d) edge node [below] {} (e) ;
    \draw[->] (e) edge node [above] {$C_1,\dots,C_k$} (end) ;
    \draw[->] (e) edge node [below] {$k$ clusters} (end) ;
\end{tikzpicture}
\end{center}
\normalsize
\paragraph{Re-scaling and Gathering Time Series in a Matrix}
For $T_0 > 0$, 
we consider the $ |\mathcal{I}| \times T_0 $-matrix $\Y_0$  which contains the re-scaled historical power consumption time series: for any $i \in \mathcal{I}$ and any $1-T_0 \leq t \leq 0$, 
\[ (\Y_0)_{i\,t} \defeq \frac{y_{i\,t}}{\bar{y}_i}, \quad \mathrm{with} \quad \bar{y}_i \defeq \frac{1}{T_0} \sum_{t=1-T_0}^0 y_{i\,t}.\]
\paragraph{Low Rank Approximation}
Since we are interested in power consumption, all the coefficients of $\Y_0$ are non-negative - we will write $\Y_0 \geq 0$ and say that this matrix is non-negative.
To reduce dimension of non-negative matrices,
\citet{paatero1994positive} and \citet{lee1999learning} propose a factorization method whose distinguishing feature is the use of non-negativity constraints.
Let us fix some integer $r \ll \min(|\mathcal{I}|,T_0)$, which will ensure a reduction of the dimension (we chose $r=10$ in the experiments of the next subsection). 
The non-negative matrix factorization (NMF) approximates matrix $\Y_0$ by $\Y_0 \approx~\W^{\star}\H^{\star},$ where $\W^{\star}$ and $\H^{\star}$ are $|\mathcal{I}| \times r$ and $r \times T_0$ non-negative matrices.
They are computed by solving:
\begin{equation*}
\big(\W^{\star}, \, \H^{\star}\big) \in \argmin_{\W,\,  \H \, \geq \, 0} \big\| \Y_0-\W\H \big\|^2_F =\argmin_{\W,\,  \H \, \geq \, 0}
\sum_{i,t} \Big( y_{i\,t}- \big(\W\H\big)_{i\,t}\Big)^2.
\end{equation*}
We use the function \texttt{NMF} of the
\texttt{Python}-library \texttt{sklearn.decomposition} to approach a local minimum with a coordinate descent solver and denote by $\hat{\W}$ the approximation of $\W^{\star}$.
Thanks to the NMF, for any $i \in \mathcal{I}$, $r$ characteristics (the $i^{\mathrm{th}}$ line of matrix $\hat{\W}$) are thus  computed.
\paragraph{Extracting and Re-scaling Characteristic Vectors}
To give the same impact to each of these characteristics, we re-scale the columns of $\hat{\W}$ and define, for each household $i$, the  vector
$$ \w_{i} = \Bigg( \frac{\hat{\W}_{i \, 1}}{\sum_{j \in \mathcal{I}} \hat{\W}_{ j\, 1}} \, , \, \dots \, , \, \frac{\hat{\W}_{i \, r}}{\sum_{j \in \mathcal{I}} \hat{\W}_{j \, r}}  \Bigg).$$
\paragraph{k-Means Clustering}
The k-means algorithm (introduced by \citet{macqueen1967some}) is then used on these  $r$-~vectors to cluster the households into a fixed number $k$ of groups (which varies from $4$ to $64$ in our experiments).
We recall below how this algorithm works.
With $\{C_1,\dots,C_k\}$ a $k$-clustering of set $\mathcal{I}$, for any $1 \leq \ell \leq k$, we define the center $\bar{\w}_{\ell} $ and the variance $\mathrm{Var}(C_\ell)$ of cluster $C_\ell$ by
\[\bar{\w}_{\ell} \defeq \frac{1}{|C_\ell|} \sum_{i \in C_\ell} \w_i\quad \mathrm{and} \quad \mathrm{Var}(C_\ell) \defeq \frac{1}{C_\ell} \sum_{i \in C_\ell} \| \w_i -\bar{\w}_{\ell} \|^2.\]
In $k$-means clustering, each household belongs to the cluster with the nearest center. The best set of clusters, denoted by $\big\{C^{\star}_1,\dots,C^{\star}_k\big\}$ -- namely the best set of centers -- is obtained by minimizing the following criterion:
\begin{equation*}
\big\{C^{\star}_1,\dots,C^{\star}_k\big\} \in \argmin_{\{C_1,\dots,C_k\}} \sum_{\ell=1}^k \sum_{\w \in C_{\ell}} \big\| \w-\bar{\w}_{\ell} \big\|^2 = \argmin_{\{C_1,\dots,C_k\}} \sum_{\ell=1}^k |C_\ell| \mathrm{Var}(C_\ell).
\end{equation*}
In practice, we use the use \texttt{KMeans} function of the \texttt{Python}-library \texttt{sklearn.cluster} to compute clusters. 

\begin{figure}[tp]
\centering
\includegraphics[width=.95\columnwidth]{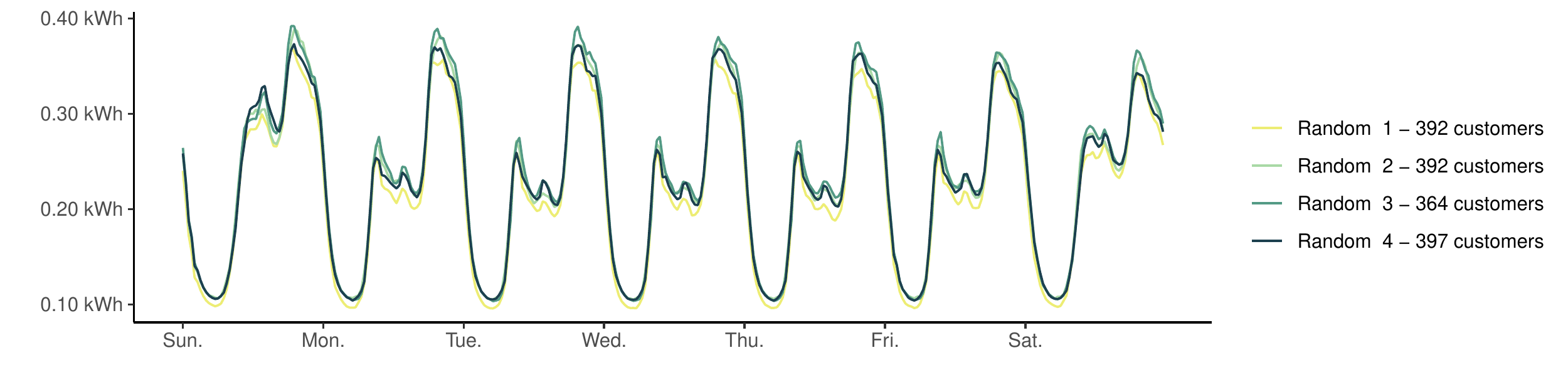} 
	\vspace{-1cm}
\caption{Mean consumption per week and per cluster, with households  randomly assigned to an integer from $1$ to $4$.}
\label{fig:conso_week_random}
\end{figure}
\begin{figure}[t!]
\centering
	\vspace{-0.17cm}
\includegraphics[width=.95\columnwidth]{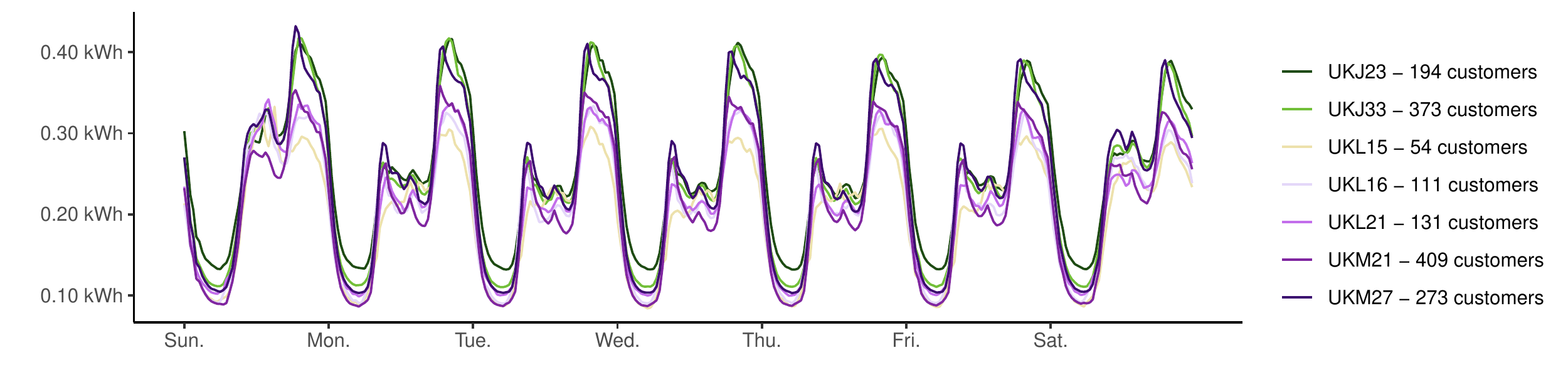} 
	\vspace{-1cm}
\caption{Mean consumption per week and per region (UK NUTS of level 3).}
\label{fig:conso_week_nuts}
\end{figure}
\begin{figure}[t!]
\centering
	\vspace{-0.17cm}
\includegraphics[width=.95\columnwidth]{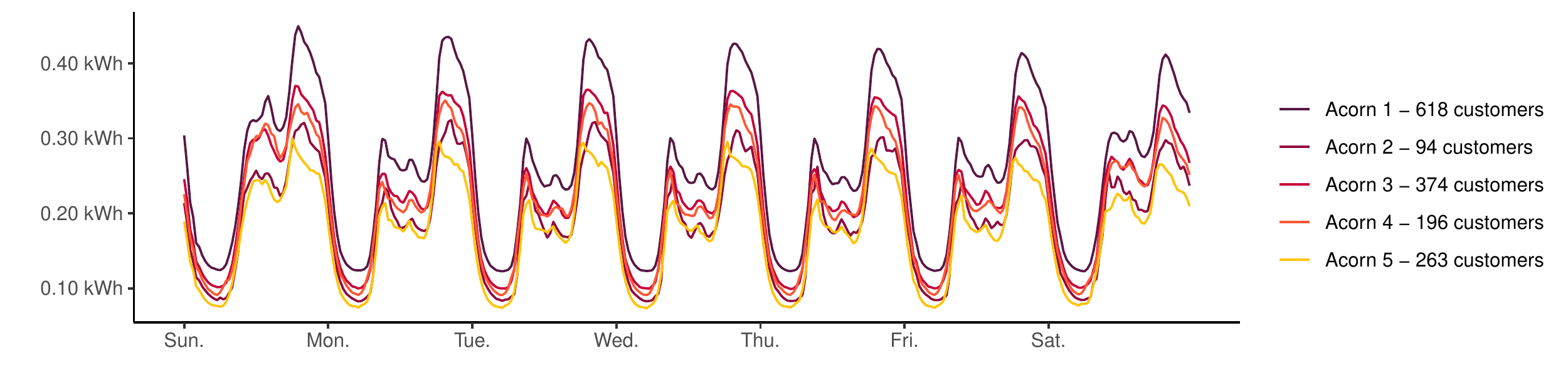} 
	\vspace{-1cm}
\caption{Mean consumption per week and per Acorn category value (from $1$ to $5$).}
\label{fig:conso_week_acorn}
\end{figure}
\begin{figure}[t!]
\centering
	\vspace{-0.17cm}
\includegraphics[width=.95\columnwidth]{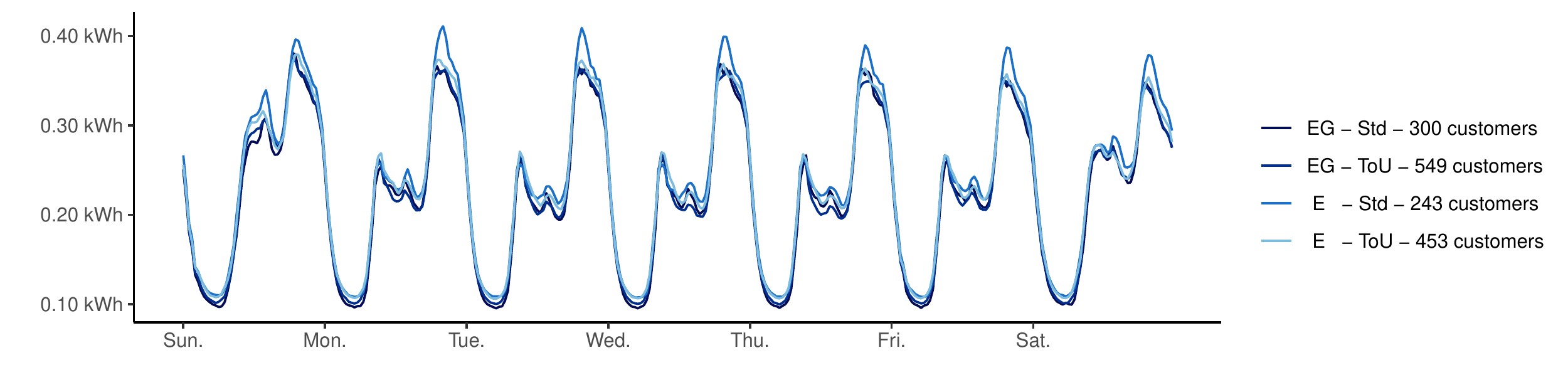} 
	\vspace{-1cm}
\caption{Mean consumption per week for the households clustered according ``Fuel + Tariff".}
\label{fig:conso_week_fueltout}
\end{figure}
\begin{figure}[t!]
\centering
	\vspace{-0.17cm}
	\includegraphics[width=.95\columnwidth]{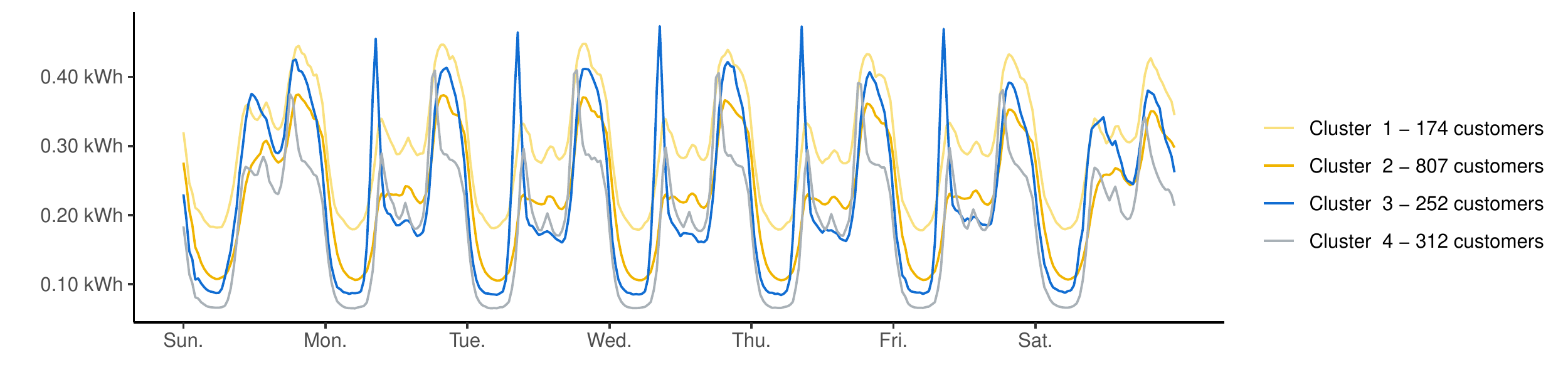} 
	\vspace{-1cm}
	\caption{Mean consumption per week and per cluster, with each household assigned to one of the four groups according to the NMF and $k$-means procedure (``NMF (4)" clustering).}
	\label{fig:conso_week_nmf}
\end{figure}

\paragraph{Description and Analysis of ``NMF (4)"}
For $k=4$, weekly profiles are plotted in Figure~\ref{fig:conso_week_nmf}.
This clustering seems to detect consumption behaviors much more specific than any of the previous ones.
 Indeed, Clusters 3 and 4  present a peak of consumption early in the morning on working days, while the  consumption of Cluster 2-- which includes the largest number of households -- remains almost flat throughout the morning.
Moreover, the evening peak for Cluster 4 arrives earlier than for the other clusters.  Finally, the consumption of Cluster 1 is generally the highest, while that of Cluster 2 is the lowest.

\subsubsection{Comparison of Clusterings}
To measure similarity between the clusterings above, we calculate the adjusted rand index (ARI) -- see~\citet{rand1971objective} -- for each segmentation pair and report the values thus obtained in Table~\ref{tab:ari}. 
Given a set elements $\mathcal{I}$ and two partitions to compare, for example the segmentation ``Region" $\{R_1,\dots R_{N}\}$ and another clustering $\{C_1,\dots C_{k}\}$, the ARI is defined by
\[
ARI \defeq \frac{\displaystyle \sum_{\ell=1}^k\sum_{n=1}^N \binom{|C_{\ell} \cap R_n |}{2} - \Bigg[ \sum_{\ell=1}^k \binom{|C_\ell|}{2}\sum_{n=1}^N \binom{|R_n|}{2} \Bigg] / \binom{|\mathcal{I}|}{2}}{\displaystyle \frac{1}{2} \Bigg[  \sum_{\ell=1}^k \binom{|C_\ell|}{2} + \sum_{n=1}^N \binom{|R_n|}{2} -\sum_{\ell=1}^k \binom{|C_\ell|}{2}\sum_{n=1}^N \binom{|R_n|}{2} \Bigg]   / \binom{|\mathcal{I}|}{2}}. 
\]
ARI lies in $[-1, 1]$ by construction, it is equal to 0 for a random matching between clusters of the two considered segmentations and to 1 for a perfect alignment.
Similarity between our  different household partitions is very low, only
``Region" is slightly correlated with all other clusterings, and ``NMF (4)" with ``ACORN". But these correlations remain low and
 the clustering ``NMF (4)" therefore seems to extract, from historical time series, some  households information that are not contained in other clusterings. Its use should improve forecasts -- this will be confirmed by the experiments below.
\begin{table}[t]
\centering
\label{tab:ari}
\begin{tabular}{lcccc}
\midrule
{} & \textbf{Region} & \textbf{NMF (4)} & \textbf{Acorn} & \textbf{Fuel $+$ Tariff} \\
\midrule
\midrule
 \textbf{Random (4) }& -0.000 & 0.000 & 0.003 & -0.000 \\
\textbf{Fuel $+$ Tariff} & 0.016 & -0.001 & 0.004 & \\
\textbf{Acorn} & 0.043 & 0.018 & & \\
\textbf{NMF (4)} & 0.011 & & & \\
\midrule
\end{tabular}
\caption{ARI (Adjusted Rand index) for each segmentation pair.}
\end{table}

\subsection{Experiment Design}
\label{subsec:experiment_design}

Thanks to the above methods, we established several partitions of the household set $\mathcal{I}$. 
As explained below, choosing one or two of them amounts to considering a two-level hierarchy (Example~\ref{ex:simplehierarchy}) or two crossed hierarchies (Example~\ref{ex:doublehierarchy}). We also detail the corresponding set of node $\Gamma$.
We then describe how we build meteorological data for each node $\gamma\in\Gamma$ and generate corresponding features.
Finally, we focus on standardization and online calibration of aggregation hyper-parameters.
We have divided the data set into training data: one-year of historical data (from April $20$, $2009$ to April $19$, $2010$) -- used for NMF clusterings, feature generation method training,  and standardization -- and testing data. 
As aggregation algorithms start from scratch, they work poorly during the first rounds. We therefore withdraw the first $10$  days of testing data from the performance evaluation period.
So, April $20$, $2010$ to April $30$, $2010$ is left for initializing aggregation algorithms and the hyper-parameters calibration and our methods are then tested during the last three months (from May $1$, $2010$ to  July $31$, $2010$).
We summarize  in Table~\ref{table:dates} the range of dates for each step of the procedure.

\begin{table}[t]
\label{table:dates}
	\centering
	\begin{tabular}{lrrr}
		\hline
		& Start date  & End date\\
		\hline
		\hline
		\begin{tabular}{@{}l@{}}\textbf{NMF Clusterings} \\ \textbf{Feature Generation Model Training} \\ \textbf{Features and Observations Standardization}  \end{tabular}  
		& April $20$, $2009$  & April $19$, $2010$\\
		\hline
		\textbf{Initialization of the Aggregation} 
		& April $20$, $2010$  & April $30$, $2010$\\
		\textbf{Model Evaluation} 
		& May $1$, $2010$ & July $31$, $2010$\\
		\hline
	\end{tabular}
	\caption{Date range for the steps of the proposed method}
	\label{train_test_dates}
\end{table}

\paragraph{Underlying Hierarchy}
As detailed in Section~\ref{sec:methodo}, we aim to forecast a set of power consumption time series  $\big\{(y_t^\gamma)_{t>0}, \gamma \in \Gamma\big\}$ connected to each other by some summation constraints.
These constraints are represented by one (or more) tree(s) and $\Gamma$ denotes the set of its (or their) nodes.
We refer to Example~\ref{ex:simplehierarchy} if we consider a single segmentation and to Example~\ref{ex:doublehierarchy} for two crossed clusterings. We detail below the set $\Gamma$, which will contain some subsets of households set $\mathcal{I}$, for these two configurations.
We recall that we denote the average power consumption of a group of households  $\gamma \subset \mathcal{I}$ by $y_t^\gamma \defeq \sum_{i \in \gamma} y_{i\,t}$.
Considering a single clustering $(C_1,\dots C_N)$ of $\mathcal{I}$, 
we want to forecast the consumption of each cluster $C_\ell$, and also the global consumption (namely, the one for $\gamma=\mathcal{I}$). Thus,
we set $\Gamma = \{C_\ell \}_{1 \leq \ell \leq N} \, \cup \, \{\mathcal{I}\}$ and the associated time series respect the hierarchy of Figure~\ref{fig:ex1} -- where  $y^\mathrm{\textsc tot}$ refers to the time series associated with $\mathcal{I}$ and $y^1,y^2,\dots ,y^N$ with the ones of clusters $C_1,C_2,\dots C_N$. 
We now consider two partitions. 
The first one $R_{1},\dots R_{N}$ refers to segmentation ``Region" and the second one, $C_{1},\dots C_{k}$ to any other clustering.
We would like to forecast the global consumption ($\gamma=\mathcal{I}$), the consumption associated with each region ($\gamma=R_n$, for $n=1,\dots,N$) and with each cluster ($\gamma=C_\ell$, for $\ell=1,\dots, k$) but also the power consumption of cluster $C_1$ in region $R_1$ ($\gamma=C_1 \cap R_1$), of cluster $C_1$  in region $R_2$ ($\gamma=C_1 \cap R_2$), and so on.
Thus,we consider the set of nodes 
\[\Gamma = \{C_{\ell} \cap R_{n}  \}_{1\leq \ell \leq k,\, 1 \leq n \leq N} \, \cup \, \{C_{\ell} \}_{1\leq \ell \leq k} \, \cup \, \{R_{n}\}_{1\leq n \leq N} \, \cup \, \{\mathcal{I}\}.\] 
The hierarchy associated with such crossed segmentations is represented in Figure~\ref{fig:ex4}  (with $N_1=k$ and $N_2=N$) -- where the global consumption,  associated with $\mathcal{I}$, is denoted by $y^\mathrm{\textsc tot}$, 
the one of cluster $C_{\ell}$ by $y^{\ell \, \cdot}$, the one of region $R_{n},$ by $y^{\cdot \, n}$ and where $y^{\ell \, n}$ refers to the local consumption of $C_{\ell} \cap R_{n}$.

\paragraph{Meteorological Data of any Set of Households}

Methods presented in Section~\ref{sec:features} for feature creation implicitly assume that meteorological data are available.
We recall that we collected meteorological data for each of the $N$ regions. Thus when $\gamma \in \Gamma$ refers to one of these regions, we can directly apply the feature generation methods.
However, if node $\gamma$ groups households from different regions, these data are not directly available and one may even wonder what they should correspond to. We take convex combinations of regional meteorological data, in proportions corresponding to the locations of the households. More precisely,  
for each meteorological variable (temperature, visibility or humidity), we built the meteorological variable of $\gamma$ as a convex combination of the $N$ meteorological variables of the $N$ regions. 
The weight associated with region $n$ corresponds to the proportion of this region in $\gamma$, in terms of contribution to the consumption -- this contribution is determined from historical data. 

\paragraph{Feature Creation}
For each node $\gamma$, we now have access to calendar and meteorological data. 
Considering an exponential smoothed temperature -- that models the thermal inertia of buildings~-- is likely to improve forecasts  (see among others, \citealp{doi:10.1057/palgrave.jors.2601589} and \citealp{goude2014local}), so we create the $a$-exponential smoothing of the temperature ${\bar{\tau}}_t^\gamma \defeq~a {\bar{\tau}}_{t-1}^\gamma+~(1-~a)\tau_t^\gamma $, where $a \in [0,1]$. 
After testing several values and evaluating their performance on the training set, we set $a=0.999$.
\begin{figure}[tp!]
\center
\includegraphics[width=0.4945\columnwidth]{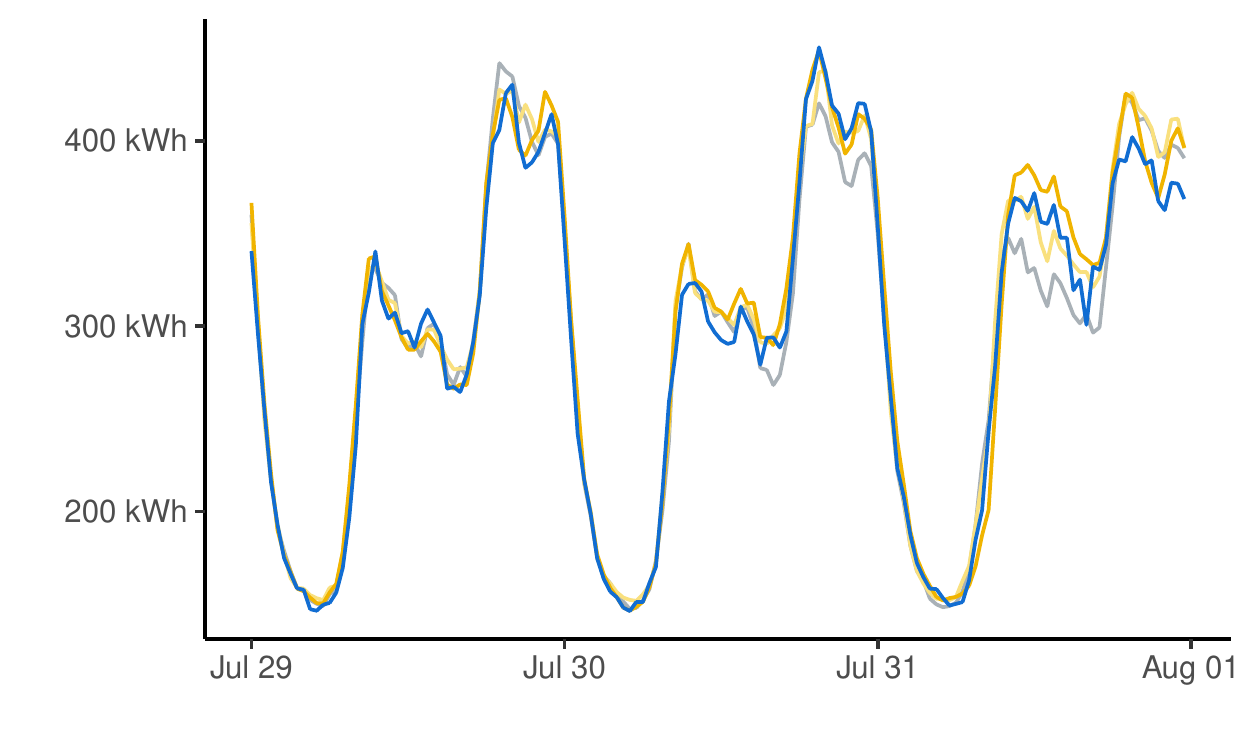}
\includegraphics[width=0.4945\columnwidth]{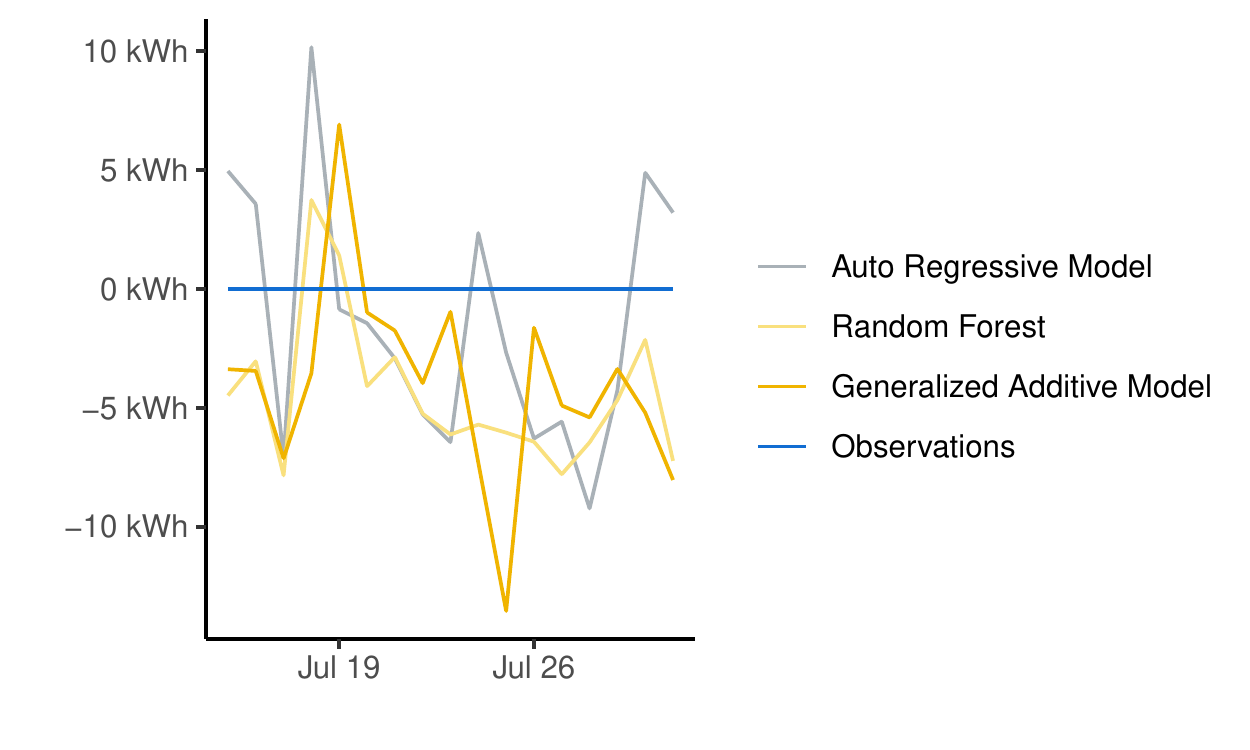}
\caption{Left picture: benchmark forecasts (auto-regressive model, generalized additive model, random forest) and observations of global consumption ($\gamma=\mathcal{I}$) at half-hour intervals on the last three days of the test period. Right picture: corresponding daily average signed errors on the last week of the test period.}
\label{fig:benchmark}
\end{figure}
\begin{figure}[t!]
\center 
\includegraphics[width=0.4945\columnwidth]{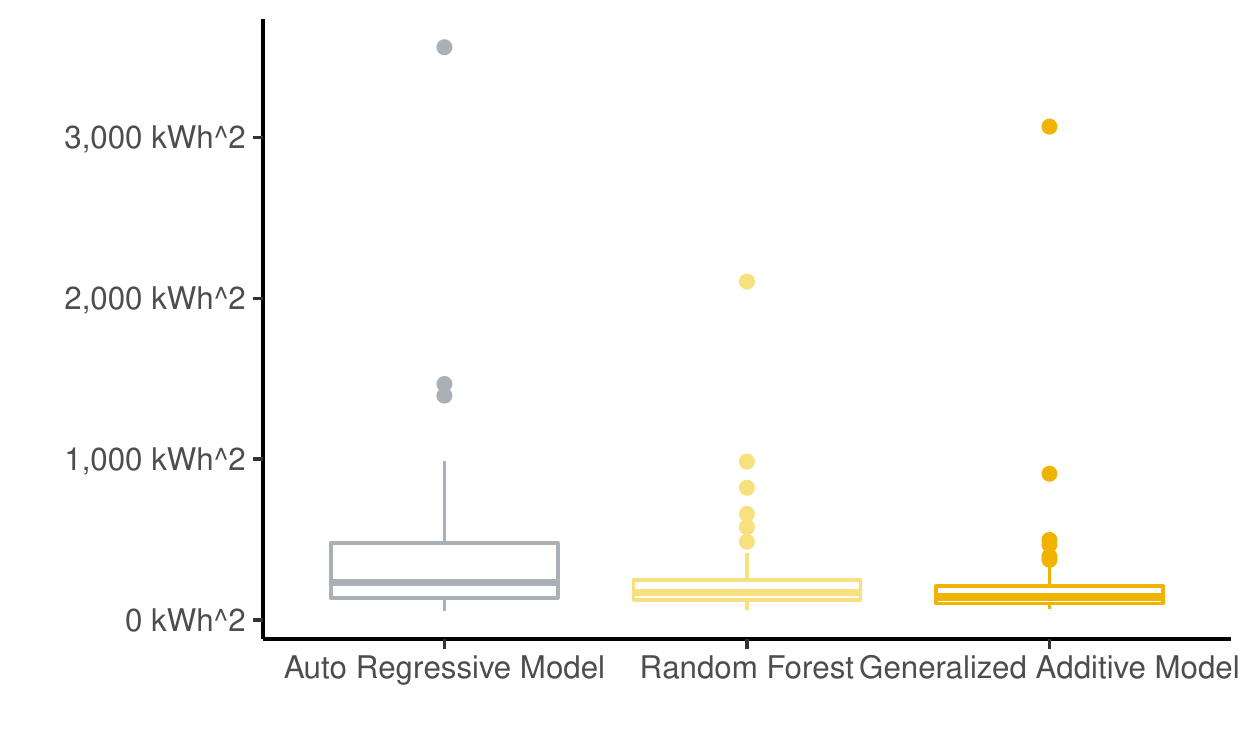}
\includegraphics[width=0.4945\columnwidth]{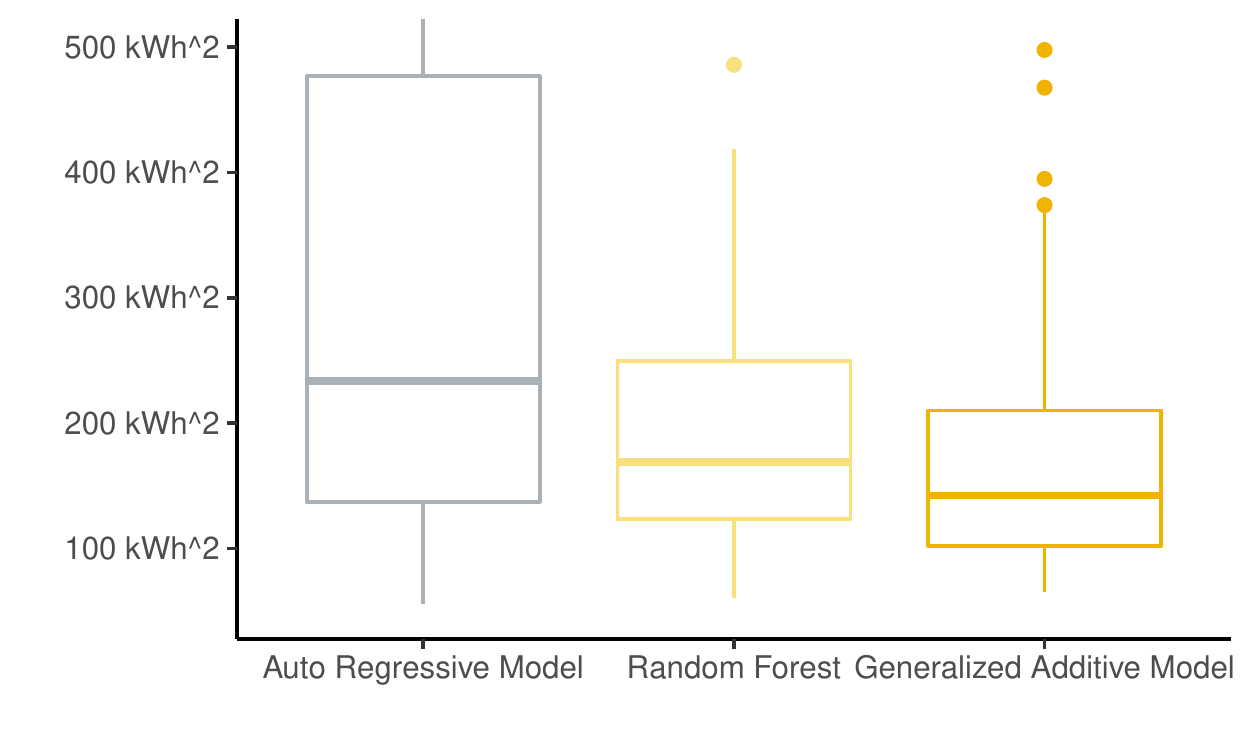}
\caption{Distribution over the test period of daily mean squared error of global consumption benchmark forecasts (auto-regressive model, generalized additive model, random forest). Left picture: original boxplots; Right picture: boxplots trimmed at 500 $\mathrm{kWh}^2$.}
\label{fig:benchmark_err}
\end{figure}
We then apply methods of Section~\ref{sec:features} using available explanatory variables to generate features $\x_t$. 
Each model (auto-regressive model, generalized additive model or random forest) is trained on a year of historical data (from April $20$, $2009$ to April $20$, $2010$). Then, forecasts are computed on the period April $20$, $2010$ to July $31$, $2010$. 
On the left of Figure~\ref{fig:benchmark}, we represent these benchmark predictions and the observations for the global consumption (namely $\gamma= \mathcal{I}$) over the last three days of the test period.
On the right, we plot daily signed errors, $\frac{1}{48}\sum_{s=t}^{t+48} \big(y_s^\gamma - x_s^\gamma\big)$, for  $\gamma= \mathcal{I}$ over the last week of the test period.
Finally, daily mean squared errors, $\frac{1}{48}\sum_{s=t}^{t+48}\big(y_s^\gamma - x_s^\gamma\big)^2$, are computed for each test period day and represented by box-plots on Figure~\ref{fig:benchmark_err}. 
The generalized additive model seems to perform the best (and the auto-regressive model the worst), this will be confirmed by the numerical results of the next subsection. 

\paragraph{Observations and Features Standardization}
Once above features computed, they are standardized using the protocol presented in Subsection~\ref{subsec:standardization}. We assess the quality of the standardization for one given configuration, namely ``Region + NMF (16)'', with features generated by the general additive model (this configuration, which refers to the two crossed clusterings ``Region" and ``NMF (16)", reaches the lower predictions errors -- see Table~\ref{tab3}).
As there are $7$ regions, the set $\Gamma$ consists of $16 \times 7 + 16 + 7 +1 =136$ nodes, but only $129$ are non-empty.  
For both standardized and non-standardized observations and features, we compute, for each node $\gamma\in \Gamma$, the empirical mean and empirical standard deviation over the test period. The distributions are plotted in Figures~\ref{fig:standardization_mean} and~\ref{fig:standardization_var}, respectively.
Since the abscissa for non-standardized data is in logarithmic scale, the mean and standard deviation of data differ a lot from a node to another. For example, the right-hand point is the global consumption ($\gamma=\mathcal{I}$), while points on the left correspond to the consumptions of small clusters. 
Thus, standardization centers data and decreases standard deviations of observations, as desired. 
In addition, standard deviations of features are close to $1$. 
 Figure~\ref{fig:cor_mat} represents correlation matrices of the $|\Gamma|$-vectors $(\x_t)_{1\leq t \leq T}$ and  $(\breve{\x}_t)_{1\leq t \leq T}$, that contain the non-standardized and standardized features over the test period. 
This shows that our standardization process  is centering, re-scaling and de-correlating features.
Finally, Table~\ref{tab_stand} gathers numerical values of the average, over $\gamma \in \Gamma$, of empirical means and standard deviations (these values are indicated by dashed vertical lines on Figures~\ref{fig:standardization_mean} and~\ref{fig:standardization_var}). We also compute the maximum of the absolute value of features and observations -- ``Bound" column of the table. This gives an empirical approximation of the boundedness constant $C$ -- see boundedness assumptions~\eqref{eq:ass2}. 
\begin{figure}[tp!]
	\vspace{1cm}
\center
\includegraphics[width=0.396\columnwidth]{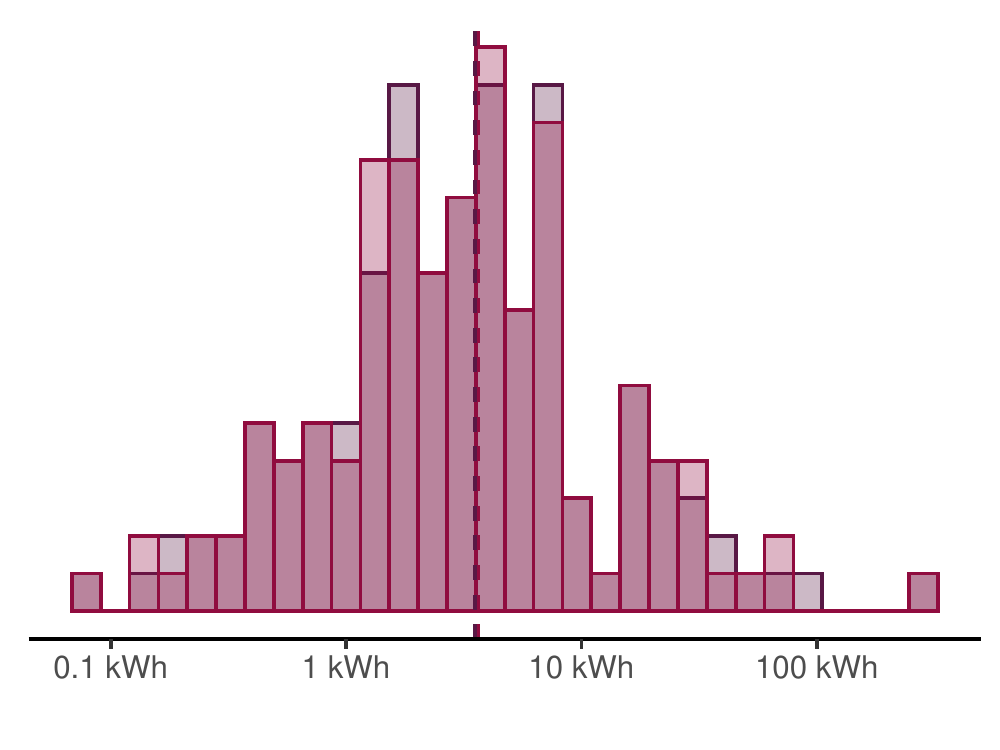} 
\includegraphics[width=0.495\columnwidth]{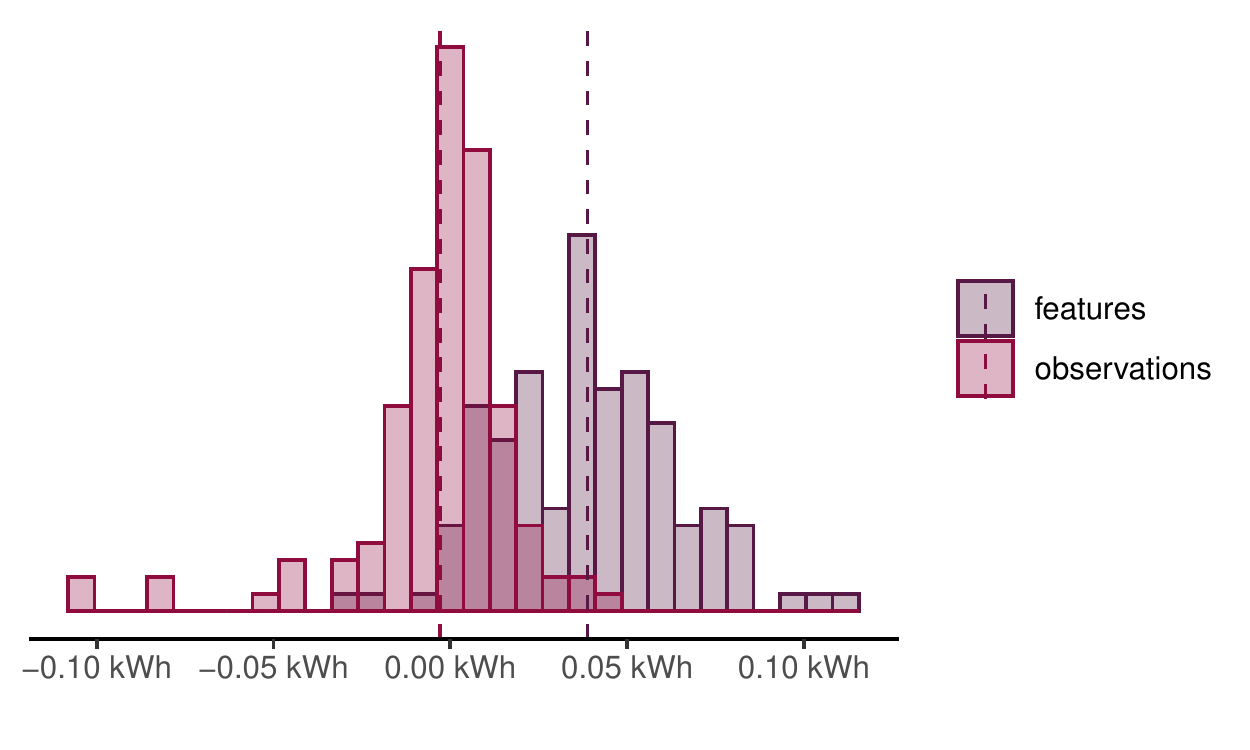} 
\caption{Distribution of empirical means per cluster, for non-standardized and standardized observations and features.}
\label{fig:standardization_mean}
\end{figure}
\begin{figure}[t!]
\center
\includegraphics[width=0.396\columnwidth]{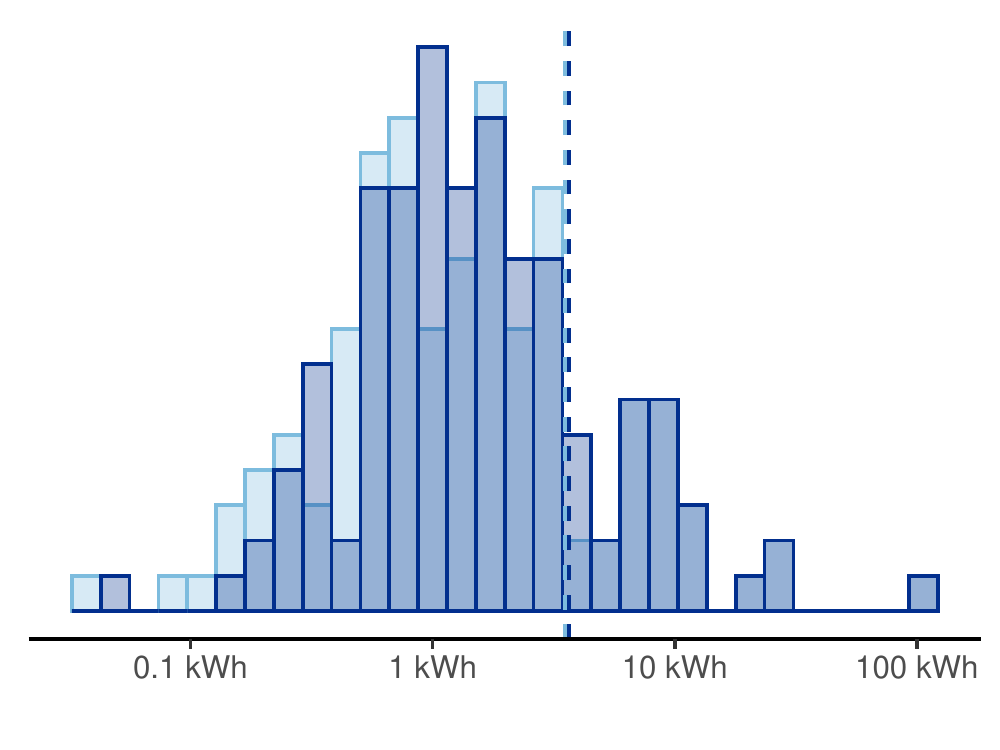} 
\includegraphics[width=0.495\columnwidth]{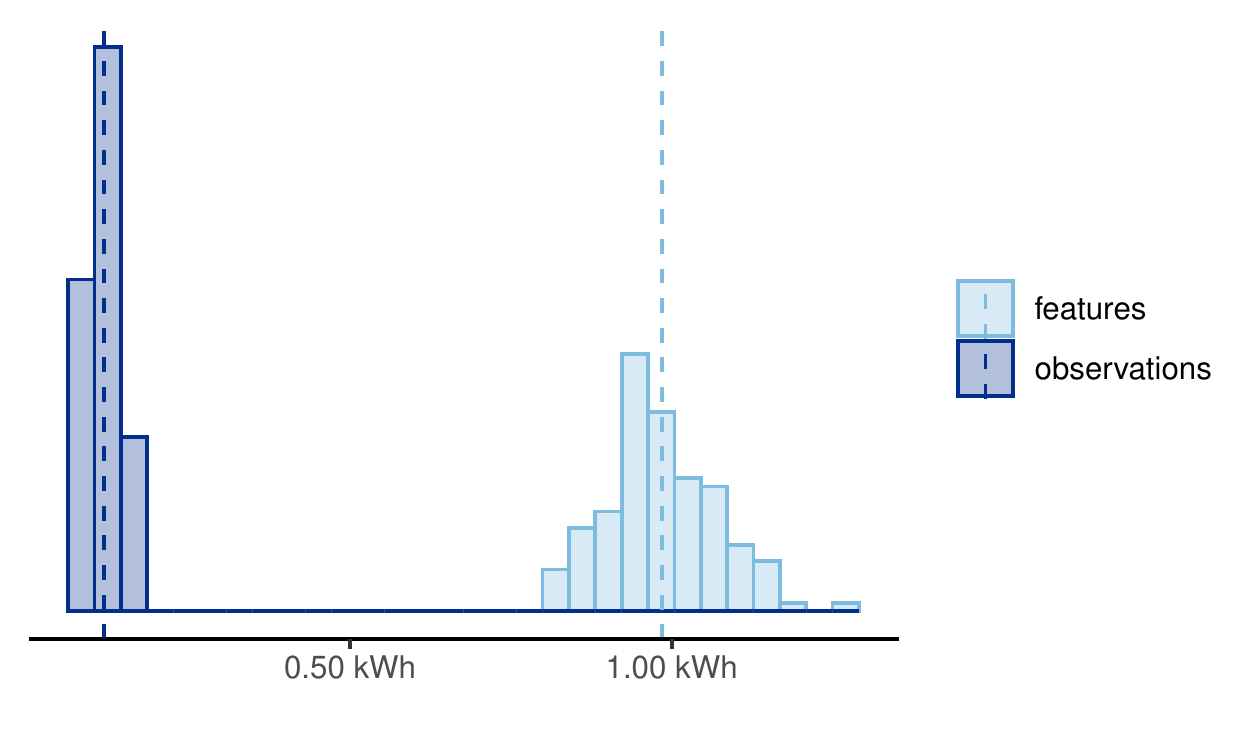} 
\caption{Distribution of empirical standard deviations per cluster, for non-standardized (left) and standardized (right) observations and features.}
\label{fig:standardization_var}
\end{figure}
 \begin{figure}[t!]
\center
\includegraphics[width=0.8\columnwidth]{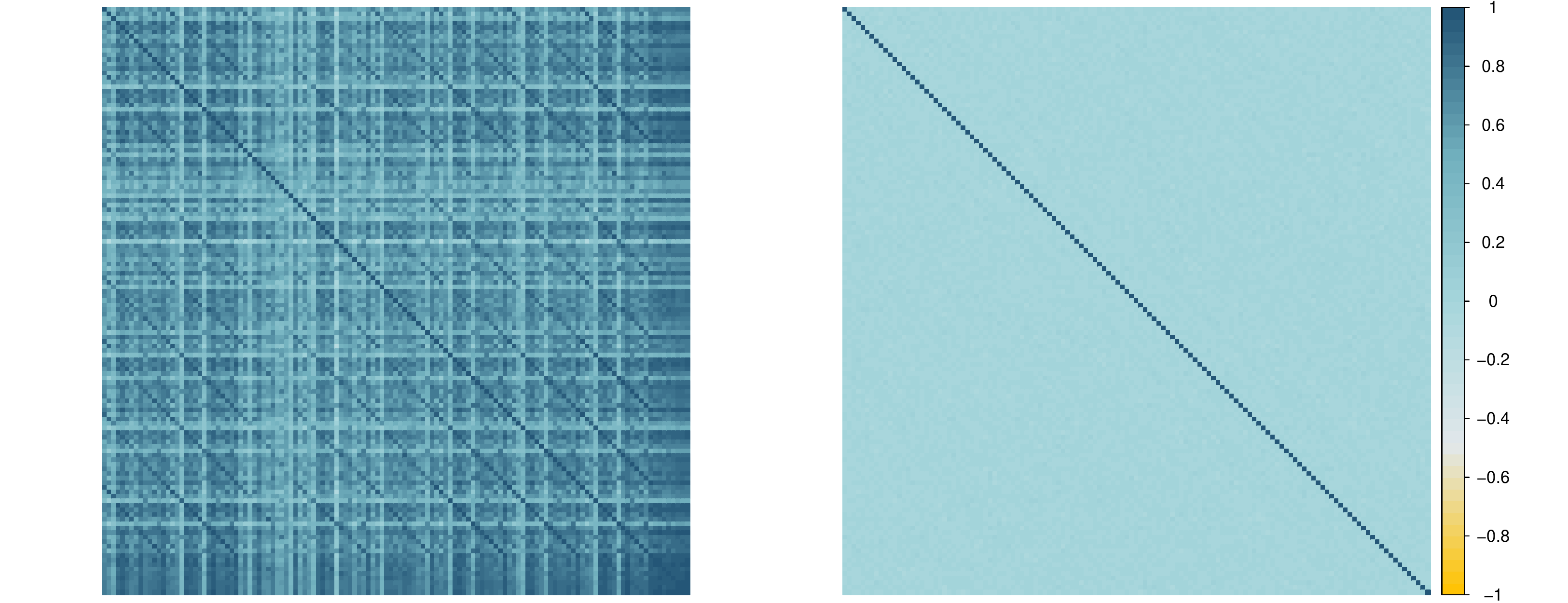} 
\caption{Correlation matrix of non-standardized (left) and standardized (right) feature vectors.}
\label{fig:cor_mat}
\end{figure}

\begin{table}[b]
\centering
\begin{tabular}{lrrr}
\hline
& \textbf{Mean}  & \textbf{Bound} & \textbf{Standard deviation} \\
\hline
\hline
 \textbf{Observations} & 9.53  & 570.02 & 3.65\\
 \textbf{Features} & 9.54  & 570.87& 3.53 \\
\hline
 \textbf{Standardized observations} & -0.003  &   1.27& 0.12\\
 \textbf{Standardized features} & 0.04 & 18.9& 0.98 \\
 \hline
\end{tabular}
\caption{Mean, and maximum of absolute value and standard deviation of observations and features before and after standardization}
\label{tab_stand}
\end{table}

\paragraph{Calibration of Hyper-Parameters}
Once features and observations are standardized, we choose one of the algorithms presented in Section~\ref{sec:aggregalgo} and run it, on the $|\Gamma|$ nodes, in parallel with the same hyper-parameter. 
For the sequential non-linear ridge regression (NL-Ridge),  we have to choose the regularization parameter $\lambda$ (see Equation~\ref{ridge}) and for BOA and ML-Pol algorithms, we need to set $\alpha$, the radius of the $L1$-ball  (see Algorithm~\ref{algo:trick}). 
Henceforth,, we denote by $\beta$ this hyper-parameter (which is equal to $\lambda$ for NL-Ridge and to $\alpha$ for BOA and ML-Pol).
We optimize the choice of $\beta$ by grid search, which is simply an exhaustive search in a specified finite subset $G$ of the hyper-parameter space. This optimization is performed sequentially. 
Indeed, for any node $\gamma$ and any instance $t>48$, we run $|G|$ algorithms in parallel and we chose the one -- denoted by $\beta_t$ -- which minimizes the average prediction error on past available data.
Thus, with $\hat{y}_s^\gamma(\beta)$ the output, at an instance $s$, of algorithm $\mathcal{A}^\gamma$ run with $\beta$,
we choose the parameter $\beta_t$ as follows:
\[
\beta_{t}  \in  \argmin_{\beta \in G}   \frac{1}{t- 48} \sum_{s=1}^{t- 48} \frac{1}{|\Gamma|} \sum_{\gamma \in \Gamma} \big( y_s^\gamma - \hat{y}_s^\gamma (\beta) \big)^2.
\]
In our experiments, to reduce the computational burden, we set $G= \big\{ 4^i\, | \, i =-5,-4,\dots,5 \big\}$, so (only) $11$ aggregations are run in parallel. At each new day, we check that we never reach the bounds $4^{-5}$ and $4^5$.
This kind of online calibration has shown good performance in load forecasting (see, for example,~\citealp{devaine2013forecasting}).

\subsection{Results}
\label{subsec:results}

In this subsection, we compare the four forecasting strategies detailed below  by evaluating them on the testing period (May $1$, $2010$ to July $31$, $2010$), for each forecasting method of Section~\ref{sec:features}, for each  aggregation algorithm of Section~\ref{sec:aggregalgo} and for various households clusterings.
To do so, we introduce some prediction error defined below as well as a confidence bound on this error.
We recall that we aim to forecast, at each instance $t$, a vector of time series $\y_t=(y_t^\gamma)_{\gamma \in \Gamma}$.
The first strategy, that we call ``Benchmark", consists simply in providing the features $\x_t$ as forecasts.
The second one considers only the projection step and thus skips the aggregation step (we will refer to it as the ``Projection'' strategy), the associated forecasts are thus the projected features $\Pi_\K( \x_t)$.
To measure the impact of the aggregation step, without projection, we also evaluate the forecasts $\hy_t$  (which do not necessary satisfy the hierarchical constraints) -- this strategy is called ``Aggregation''.
Finally, the strategy ``Aggregation + Projection" provides the predictions $\ty_t= \Pi_\K( \hy_t)$.
To allow for an evaluation of the accuracy of the prediction of some time series only, we define  the prediction error $E_T(\Lambda)$, for some subset of nodes $\Lambda \subset \Gamma$. 
In the results below, this subset can be equal to $\Gamma$ (to evaluate the strategies on all the nodes), to the singleton $\{ \mathcal{I} \}$ (to focus on the global consumption -- namely the consumption of all the households), or to the set of leaves of the tree associated with the considered segmentation(s), denoted by $\Gamma_0$ (to evaluate the performance of local forecasts only).
Note that $E_T (\Gamma)$ will correspond to $\tilde{L}_T\times |\Gamma|$ for the ``Aggregation + Projection'' strategy (see Equation~\ref{eq:loss}).
We now define, for any subset $\Lambda \subset \Gamma$, the prediction error $E_T ( \Lambda)$.
First of all, for a node $\gamma \in \Lambda$ and an instance $t$, let us denote by $\varepsilon_t^\gamma$ the instantaneous squared error.
It corresponds to $\big(y_t^\gamma-x_t^\gamma\big)^2$ for the  ``Benchmark" strategy, to $ \Big(y_t^\gamma-\big(\Pi_\K (\x_t) \big)^\gamma\Big)^2$ for ``Projection'', to $\big(y_t^\gamma-\hat{y}_t^\gamma\big)^2$ for ``Aggregation", and to $ \big(y_t^\gamma-\tilde{y}_t^\gamma\big)^2$ for the ``Aggregation + Projection'' strategy.
We then consider the average (over time) squared error (which is cumulated over $\Lambda$):
\[
E_T (\Lambda) \defeq \sum_{\gamma \in \Lambda} \frac{1}{T}\sum_{t=1}^T  \varepsilon_t^\gamma.
\]
We associate with this error a confidence bound and present our results (see Tables~\ref{tab2_all}--~\ref{tab3}) in the form:
\begin{equation}
\label{eq:error}
E_T (\Lambda) \pm \frac{\sigma_T(\Lambda)}{\sqrt{T}}, \quad \text{where} \quad \sigma_T(\Lambda)^2= \frac{1}{T} \sum_{t=1}^T \sum_{\gamma \in \Lambda} \big( \varepsilon_t^\gamma- E_T(\Lambda)\big)^2.
\end{equation}
We choose the quantity  $ \sigma_T(\Lambda)/\sqrt{T}$
as it is reminiscent of the error margin provided by asymptotic confidence intervals on the mean of independent and identically distributed random variables.
In the next paragraph, we consider the ``Region + NMF(16)"  configuration and, for each of the three benchmark forecasting methods of Section~\ref{sec:features} and for each of the three aggregation algorithms presented in Section~\ref{sec:aggregalgo},
we compute these errors and confidence bounds for the four above foresting strategies.
Finally, in the last paragraph, we set the benchmark forecasting method  (generalized additive model) and the aggregation algorithm (ML-Pol) to test various households clusterings.

\subsubsection{Impact of the Benchmark Forecasting Methods and of the Aggregation Algorithms}
\begin{table}[tp!]
\centering
\begin{tabular}{l R{2cm} R{2cm} R{2cm}}
\hline
& \textbf{NL-Ridge} & \textbf{ML-Pol} & \textbf{BOA} \\
\hline
\hline
\textbf{General Additive Model } & & & \\
\hline
 \vspace{-1.1cm} \\
\vspace{-.3cm}
& \multicolumn{3}{c}{
$\underbrace{ \qquad  \qquad \qquad \qquad \qquad  \qquad \qquad \qquad \qquad }_{}$}\\
Benchmark & \multicolumn{3}{c}{$ \,455.5 \pm 1.1$} \\
Projection &   \multicolumn{3}{c}{$ \,450.7  \pm 1.1$}\\ 
\hdashline
Aggregation & $407.6 \pm 1.1$ & $397.9 \pm 1.0$ & $406.0 \pm 1.0$ \\
Aggregation + Projection & $405.9 \pm 1.1$ &  \cellcolor{Gray2} $396.0 \pm 1.0$ & $403.5 \pm 1.0$ \\
\hline
\hline
\textbf{Random Forest} & & & \\
\hline
 \vspace{-1.1cm} \\
\vspace{-.3cm}
& \multicolumn{3}{c}{
$\underbrace{ \qquad  \qquad \qquad \qquad \qquad  \qquad \qquad \qquad \qquad }_{}$}\\
Benchmark &  \multicolumn{3}{c}{$ \,528.1 \pm 1.0$} \\
Projection &  \multicolumn{3}{c}{$  \,500.8 \pm 1.0$} \\
\hdashline
Aggregation & $459.3 \pm 1.0$ & $467.3 \pm 1.0$ & $470.9 \pm 1.0$ \\
Aggregation + Projection &  \cellcolor{Gray}  $451.1 \pm 1.0$ & $464.0 \pm 1.0$ & $468.1 \pm 1.0$ \\
\hline
\hline
\textbf{Auto-Regressive Model} & & & \\
\hline
 \vspace{-1.1cm} \\
\vspace{-.3cm}
& \multicolumn{3}{c}{
$\underbrace{ \qquad  \qquad \qquad \qquad \qquad  \qquad \qquad \qquad \qquad }_{}$}\\
Benchmark & \multicolumn{3}{c}{$ \,736.4  \pm  1.6$}  \\
Projection & \multicolumn{3}{c}{$ \,734.3 \pm 1.6$}  \\
\hdashline
Aggregation & $690.7 \pm 1.6$ & $690.1 \pm 1.6$ & $698.2 \pm 1.6$ \\
Aggregation + Projection & $689.8 \pm 1.6$ &  \cellcolor{Gray} $687.3 \pm 1.6$ & $693.1 \pm 1.6$ \\
\hline
\end{tabular}
\caption{$E_T (\Gamma) \pm \sigma_T(\Gamma)/\sqrt{T}$ (see Equation~\ref{eq:error})  where $\Gamma$ refers to the set of nodes associated with ``Region + NMF (16)'' clustering, for the three benchmark forecasting methods of Section~\ref{sec:features} (General Additive Model, Random Forest and Auto-Regressive Model), for the three aggregation algorithms of Section~\ref{sec:aggregalgo} (NL-Ridge, ML-Pol and BOA) and for the four strategies defined in Subsection~\ref{subsec:results} (``Benchmark", ``Projection", ``Agregation" and ``Aggregation + Projection"). $E_T (\Gamma)$ corresponds to $\tilde{L}_T\times |\Gamma|$ for the ``Aggregation + Projection'' strategy.
For strategies ``Benchmark" and ``Projection", the forecasts do not depend on the chosen aggregation algorithm, so the errors and the confidence bounds are the same for the three algorithms.
The dark gray area corresponds to the best prediction error of the table and the light gray area to the best one, for a given benchmark forecasting method.}
\label{tab2_all}
\end{table}
We consider here the two crossed hierarchies ``Region + NMF (16)'' and we vary the benchmark forecasting approaches and the aggregation algorithms.
Indeed we compute forecasts for the three methods of Section~\ref{sec:features} -- auto-regressive model, generalized additive model and random forest -- and for the three algorithms of Section~\ref{sec:aggregalgo} --  NL-Ridge and BOA and ML-Pol.
Table~\ref{tab2_all} sums up $E_T (\Gamma) \pm \sigma_T(\Gamma)/\sqrt{T}$, where $\Gamma$ refers to the set of nodes associated with ``Region + NMF (16)''.
Regarding forecasting methods, the general additive model provides the best benchmark predictions and the auto-regressive model, which is the most naive method, does not perform well. 
This was actually already illustrated in Figures~\ref{fig:benchmark} and~\ref{fig:benchmark_err}.
Moreover, as the theory guarantees, projection (with or without an aggregation step) always improves the forecasts. 
The projection step without aggregation leads to a decrease of prediction error of around $1\%$ for the general additive and auto-regressive models and of $5\%$ for random forest. 
Note that for parametric (or semi-parametric) methods, the model is assumed to be the same at all nodes. Forecasts are thus closely linked and seem to almost already satisfy the hierarchical constraints.
On the contrary, for random forest methods, the forecasts seem less correlated and thus projection improves significantly the predictions.
The impact of aggregation step is notable: the prediction error decreases by about $10\%$ for NL-Ridge and BOA and by about $15\%$ for ML-Pol.
Finally, our global strategy always gives the best forecasts, which, in addition, satisfy the hierarchical constraints. \\

\begin{table}[tp!]
\centering
\begin{tabular}{l R{2cm} R{2cm} R{2cm}}
\hline
& \textbf{NL-Ridge} & \textbf{ML-Pol} & \textbf{BOA} \\
\hline
\hline
\textbf{General Additive Model} & & & \\
\hline
\vspace{-1.1cm} \\
\vspace{-.3cm}
& \multicolumn{3}{c}{
$\underbrace{  \qquad \qquad \qquad \qquad \qquad  \qquad \qquad \qquad \qquad }_{}$}\\
Benchmark & \multicolumn{3}{c}{$ \, 205.8 \pm 9.3$}  \\
Projection & \multicolumn{3}{c}{$ \,200.8  \pm 9.2$} \\
\hdashline
Aggregation & $179.2 \pm 8.9$ & $172.0 \pm 8.6$ & $178.8 \pm 8.8$ \\
Aggregation + Projection & $177.6 \pm 8.8$ & \cellcolor{Gray2} $170.3 \pm 8.5$ & $176.3 \pm 8.7$ \\
\hline
\hline
\textbf{Random Forest} & & & \\
\hline
\vspace{-1.1cm} \\
\vspace{-.3cm}
& \multicolumn{3}{c}{
$\underbrace{ \qquad  \qquad \qquad \qquad \qquad  \qquad \qquad \qquad \qquad }_{}$}\\
Benchmark & \multicolumn{3}{c}{$  \,231.4 \pm 8.6$} \\
Projection & \multicolumn{3}{c}{$ \, 228.8 \pm 8.2$} \\
\hdashline
Aggregation  & $207.1 \pm 8.4$ & $214.8 \pm 8.4$ & $218.7 \pm 8.3$ \\
Aggregation + Projection & \cellcolor{Gray} $206.4 \pm 8.2$ & $212.4 \pm 8.1$ & $216.8 \pm 8.2$ \\
\hline
\textbf{Auto-Regressive Model} & & & \\
\hline
\vspace{-1.1cm} \\
\vspace{-.3cm}
& \multicolumn{3}{c}{
$\underbrace{ \qquad  \qquad \qquad \qquad \qquad  \qquad \qquad \qquad \qquad }_{}$}\\
Benchmark &\multicolumn{3}{c}{$ \,380.3 \pm 13.4$} \\
Projection & \multicolumn{3}{c}{$ \,380.4\pm  13.4$} \\
\hdashline
Aggregation &  $368.6 \pm 13.5$ & $370.8 \pm 13.6$ & $376.1 \pm 13.6$ \\
Aggregation + Projection &  \cellcolor{Gray}  $368.2 \pm 13.4$ & $369.4 \pm 13.5$ & $373.6 \pm 13.5$ \\
\hline
\end{tabular}
\caption{$E_T (\{\mathcal{I}\}) \pm \sigma_T(\{\mathcal{I}\})/\sqrt{T}$ (see Equation~\ref{eq:error}) for ``Region + NMF (16)'' clustering, for the three benchmark forecasting methods of Section~\ref{sec:features} (General Additive Model, Random Forest and Auto-Regressive Model), for the three aggregation algorithms of Section~\ref{sec:aggregalgo} (NL-Ridge, ML-Pol and BOA) and for the four strategies defined in Subsection~\ref{subsec:results} (``Benchmark", ``Projection", ``Agregation" and ``Aggregation + Projection"). The prediction error $E_T (\{\mathcal{I}\}) $ corresponds to the mean squared error (over the testing period) of the global consumption.
For strategies ``Benchmark" and ``Projection", the forecasts do not depend on the chosen aggregation algorithm, so the errors and the confidence bounds are the same for the three algorithms.
The dark gray area corresponds to the best prediction error of the table and the light gray area to the best one, for a given benchmark forecasting method.}
\label{tab2_total}
\end{table}
\begin{table}[tp!]
\centering
\begin{tabular}{l R{2cm} R{2cm} R{2cm}}
\hline
& \textbf{NL-Ridge} & \textbf{ML-Pol} & \textbf{BOA} \\
\hline
\hline
\textbf{General Additive Model} & & & \\
\hline
\vspace{-1.1cm} \\
\vspace{-.3cm}
& \multicolumn{3}{c}{
$\underbrace{ \qquad  \qquad \qquad \qquad \qquad  \qquad \qquad \qquad \qquad }_{}$}\\
Benchmark & \multicolumn{3}{c}{$ \,\,\,\, \, 66.3  \pm  0.1$} \\
Projection   &\multicolumn{3}{c}{$ \,\,\,\, \, 66.3   \pm  0.1$} \\
\hdashline
Aggregation & $61.6 \pm 0.1$ & $61.2 \pm 0.1$ &  \cellcolor{Gray2} $61.0 \pm 0.1$ \\
Aggregation + Projection & $61.5 \pm 0.1$ & $61.1 \pm 0.1$ &  \cellcolor{Gray2} $61.0 \pm 0.1$ \\
\hline
\hline
\textbf{Random Forest} & & & \\
\hline
\vspace{-1.1cm} \\
\vspace{-.3cm}
& \multicolumn{3}{c}{
$\underbrace{ \qquad  \qquad \qquad \qquad \qquad  \qquad \qquad \qquad \qquad }_{}$}\\
Benchmark & \multicolumn{3}{c}{$\,\,\,\, \, 78.7  \pm 0.1$} \\
Projection & \multicolumn{3}{c}{$ \, \,\,\, \,68.9  \pm  0.1$} \\
\hdashline
Aggregation &  $66.8 \pm 0.1$ & $65.7 \pm 0.1$ & $64.8 \pm 0.1$ \\
Aggregation + Projection &  \cellcolor{Gray}  $63.9 \pm 0.1$ & $65.7 \pm 0.1$ & $64.8 \pm 0.1$ \\
\hline
\hline
\textbf{Auto-Regressive Model} & & & \\
\hline
\vspace{-1.1cm} \\
\vspace{-.3cm}
& \multicolumn{3}{c}{
$\underbrace{ \qquad  \qquad \qquad \qquad \qquad  \qquad \qquad \qquad \qquad }_{}$}\\
Benchmark &\multicolumn{3}{c}{$ \,\,\,\,  \,84.4  \pm   0.1$} \\
Projection & \multicolumn{3}{c}{$\,\,\,\,  \,84.3  \pm 0.1$} \\
\hdashline
Aggregation &  $73.8 \pm 0.1$ & $72.8 \pm 0.1$ & $73.2 \pm 0.1$\\ 
Aggregation + Projection & $73.8 \pm 0.1$ &  \cellcolor{Gray} $72.0 \pm 0.1$ & \cellcolor{Gray}  $72.0 \pm 0.1$ \\
\hline
\end{tabular}
\caption{$E_T (\Gamma_0) \pm \sigma_T(\Gamma_0)/\sqrt{T}$ (see Equation~\ref{eq:error})  where $\Gamma_0$ refers to the set of leaves associated with ``Region + NMF (16)'' clustering, for the three benchmark forecasting methods of Section~\ref{sec:features} (General Additive Model, Random Forest and Auto-Regressive Model), for the three aggregation algorithms of Section~\ref{sec:aggregalgo} (NL-Ridge, ML-Pol and BOA) and for the four strategies defined in Subsection~\ref{subsec:results} (``Benchmark", ``Projection", ``Agregation" and ``Aggregation + Projection"). $E_T (\Gamma_0)$ corresponds to a prediction errors associated with local consumptions forecasts.
For strategies ``Benchmark" and ``Projection", the forecasts do not depend on the chosen aggregation algorithm, so the errors and the confidence bounds are the same for the three algorithms.
The dark gray area corresponds to the best prediction error of the table and the light gray area to the best one, for a given benchmark forecasting method.}
\label{tab2_leaves}
\end{table}

Even though theoretical guarantees (see Theorem~\ref{gen-bound}) are only ensured for errors summed over all nodes, we investigate the impact of our methods on global consumption predictions and on most local predictions (\textit{i.e.}, predictions at leaves).
Thus, Tables~\ref{tab2_total} and~\ref{tab2_leaves} contain $E_T \big(\{\mathcal{I}\}\big) \pm \sigma_T(\{\mathcal{I}\})/\sqrt{T}$  and 
 $E_T (\Gamma_0) \pm \sigma_T(\Gamma_0)/\sqrt{T}$ (where $\Gamma_0$ is the set of leaves), respectively.
By denoting by $R_1,\dots,R_N$, the $N$ regions and by $C_1, \dots C_{16}$, the $16$ clusters provided by ``NMF (16)", we have, in this ``Region + NMF~(16)'' configuration, $\Gamma_0 \defeq  \big\{C_{\ell} \cap R_n \big\}_{1\leq \ell \leq 16,\, 1 \leq n \leq N}$.
Concerning global consumption, a mere projection
 improves the forecasts,  except in the case of auto-regressive model and, in all cases, our strategy ``Aggregation + Projection" outperforms the three strategies ``Benchmark", ``Aggregation" and ``Projection".
The prediction error associated with $\Gamma_0$ also decreases thanks to our procedure.
Therefore, our method improves the forecasting of both global and local power consumptions.
Finally, Figure~\ref{fig:res} represents the global power consumption on the three last day of the testing period and the daily average signed error on the last week for the four forecasts obtained with features generated with general additive model and aggregated with ML-Pol algorithm.
The distributions of the daily mean squared errors for these strategies are represented in Figure~\ref{fig:res_err}.
We draw the same conclusions for the daily prediction errors as for the average error on the entire test period (three months):  aggregation greatly improves the forecasts, projection does too, but to a lesser extent. The box plots show that the variance of the error also decreases after the aggregation step.

\begin{figure}[tp!]
\center
\includegraphics[width=0.4945\columnwidth]{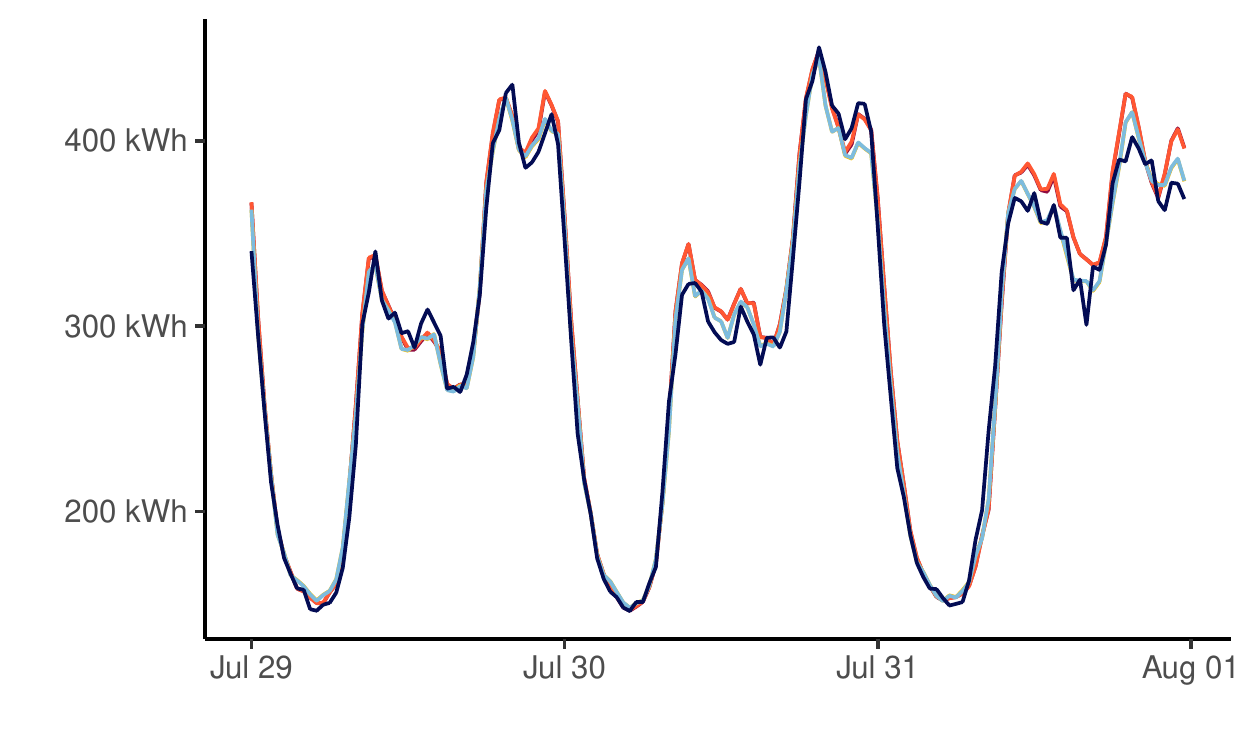} 
\includegraphics[width=0.4945\columnwidth]{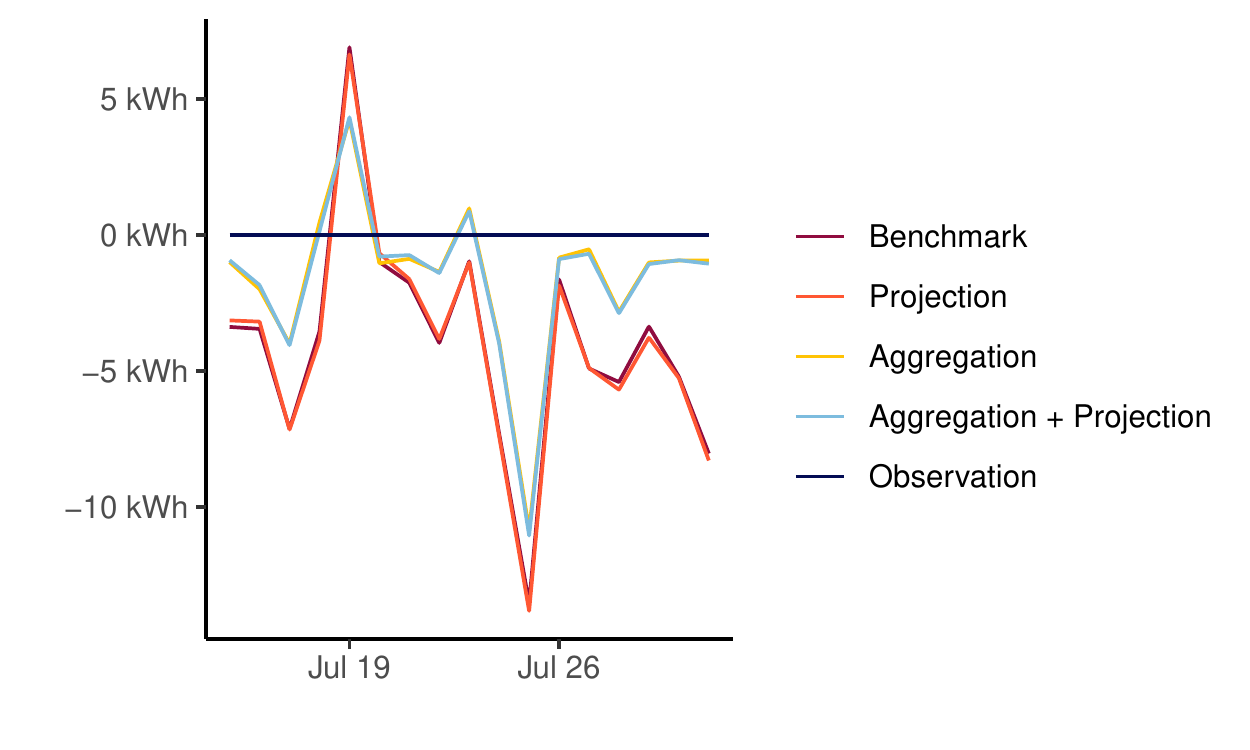} 
\caption{
Left picture: forecasts associated with the four strategies defined in Subsection~\ref{subsec:results} (``Benchmark", ``Projection", ``Agregation" and ``Aggregation + Projection"), with benchmark forecasts generated with the generalized additive model and aggregated with ML-Pol algorithm in the ``Region + NMF(16)" configuration,  and observations of global consumption ($\gamma=\mathcal{I}$) at half-hour intervals on the last three days of the test period. Right picture: corresponding daily average signed errors on the last week of the test period.}
\label{fig:res}
\end{figure}
\begin{figure}[t!]
\center 
\includegraphics[width=0.4945\columnwidth]{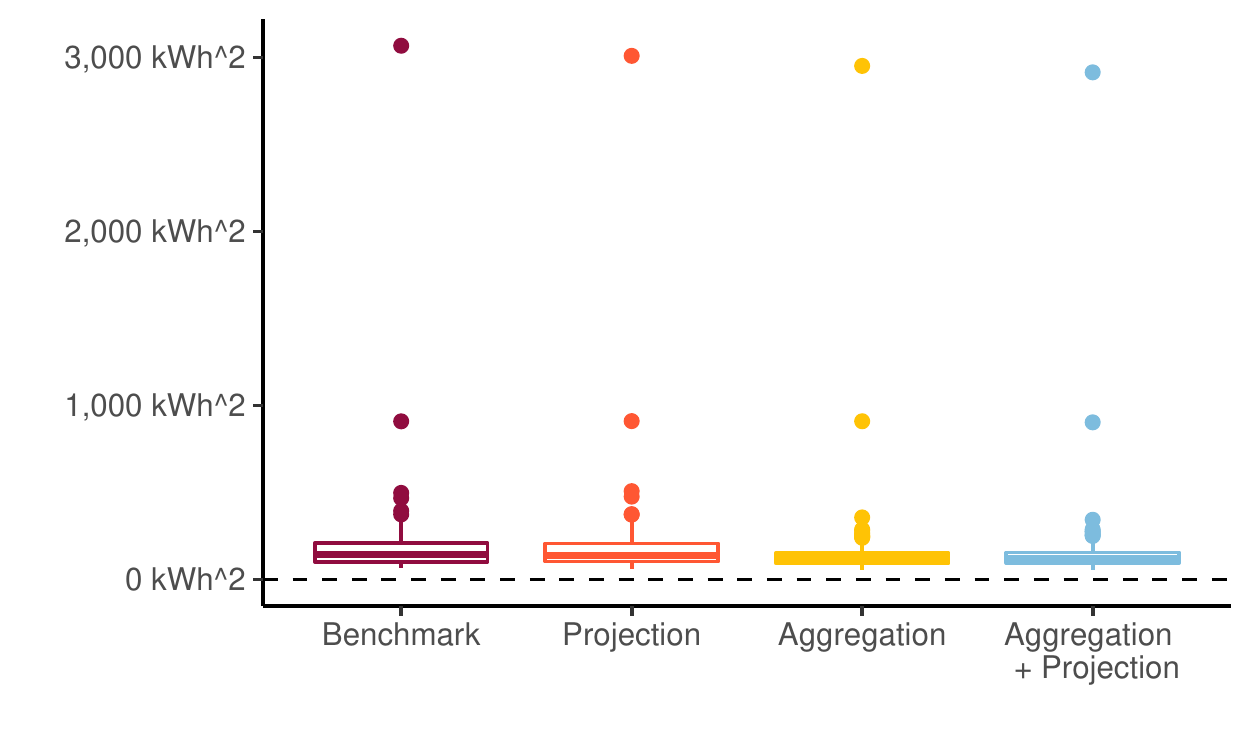} 
\includegraphics[width=0.4945\columnwidth]{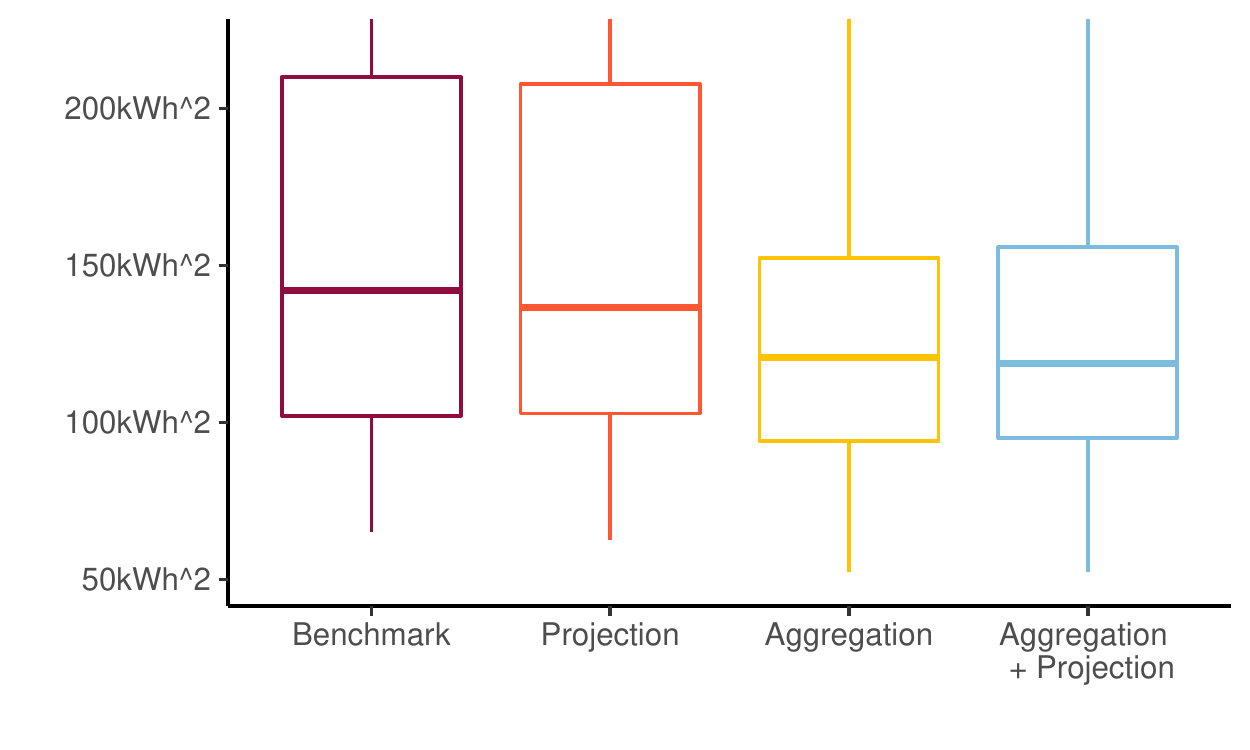} 
\caption{Distribution over the test period of daily mean squared error of global consumption for the four strategies defined in Subsection~\ref{subsec:results} (``Benchmark", ``Projection", ``Agregation" and ``Aggregation + Projection"), with benchmark forecasts generated with the generalized additive model and aggregated with ML-Pol algorithm in the ``Region + NMF(16)" configuration. Left picture: original boxplots. Right picture: boxplots trimmed at 220 $\mathrm{kWh}^2$.}
\label{fig:res_err}
\end{figure}

\subsubsection{Impact of the Clustering}
\begin{table}[tp!]
\centering
  \begin{adjustbox}{width=\columnwidth}
\begin{tabular}{l crrrr}
\hline
Clustering & Benchmark & Bottom-up & Projection & Aggregation &  Aggregation \\
&  & &  & & + Projection \\
\midrule
Region & $205.8 \pm 9.3$ & \cellcolor{Gray} $189.9 \pm 8.3$ & $201.3 \pm 9.1$ &  $187.8 \pm 8.4$ & $186.7 \pm 8.4$ \\
\textbf{Region + Acorn} & \textbf{---} & $\mathbf{194.2 \pm 8.4}$ & $\mathbf{200.8 \pm 9.2}$ &  $\mathbf{182.5 \pm 8.3}$ & $\mathbf{181.2 \pm 8.3}$ \\
Acorn & ---  & $205.7 \pm 9.5$ & $205.0 \pm 9.3$ &  $203.3 \pm 9.3$ & $202.9 \pm 9.3$ \\
\textbf{Region + Fuel + Tariff} & \textbf{---}  & $\mathbf{199.1 \pm 8.7}$ & $\mathbf{201.2 \pm 9.2}$ &  $\mathbf{185.4 \pm 8.6}$ & $\mathbf{184.1 \pm 8.6}$ \\
Fuel + Tariff & ---  & $207.1 \pm 9.7$ & $205.5 \pm 9.4$ &  $201.5 \pm 9.4$ & $201.4 \pm 9.5$ \\
\textbf{Region + Random (4)} & \textbf{---}  & $\mathbf{198.4 \pm 8.7}$ & $\mathbf{201.3 \pm 9.2}$ &  $\mathbf{186.1 \pm 8.6}$ & $\mathbf{184.6 \pm 8.6}$ \\
Random (4) & ---   &  $208.0 \pm 9.7$ & $205.7 \pm 9.4$ &  $199.5 \pm 9.4$ & $199.7 \pm 9.4$ \\
\textbf{Region + Random (8)} & \textbf{---} & $\mathbf{202.3 \pm 8.7}$ & $\mathbf{201.3 \pm 9.2}$ &  $\mathbf{182.4 \pm 8.7}$ & $\mathbf{181.0 \pm 8.7}$ \\
Random (8) & ---  & $212.9 \pm 9.8$ & $205.7 \pm 9.3$ &  $194.4 \pm 9.1$ & $194.4 \pm 9.1$ \\
\textbf{Region + Random (16)} & \textbf{---}   & $\mathbf{205.1 \pm 8.7}$ & $\mathbf{201.3 \pm 9.2}$ &  $\mathbf{180.5 \pm 8.7}$ & $\mathbf{178.8 \pm 8.7}$ \\
Random (16) & --- & $218.4 \pm 10.0$ & $205.7 \pm 9.3$ &  $188.6 \pm 8.7$ & $188.5 \pm 8.7$ \\
\textbf{Region + Random (32)} & \textbf{---} & $\mathbf{205.3 \pm 8.5}$ & $\mathbf{201.2 \pm 9.2}$ &  $\mathbf{180.4 \pm 8.8}$ & $\mathbf{178.9 \pm 8.7}$ \\
Random (32) & --- & $222.9 \pm 10.1$ & $205.6 \pm 9.3$ &  $189.6 \pm 8.7$ & $189.5 \pm 8.7$ \\
Random (64) & ---  & $222.9 \pm 9.8$ & $205.6 \pm 9.3$ &  $185.7 \pm 8.8$ & $185.5 \pm 8.8$ \\
\textbf{Region + NMF (4)} & \textbf{---}   & $\mathbf{196.0 \pm 8.6}$ & $\mathbf{200.8 \pm 9.2}$ &  $\mathbf{187.4 \pm 9.1}$ & $\mathbf{185.5 \pm 8.9}$ \\
NMF (4) & ---   & $205.7 \pm 9.5$ & $205.0 \pm 9.3$ &  $197.0 \pm 8.8$ & $196.8 \pm 8.9$ \\
\textbf{Region + NMF (8)} & \textbf{---} & $\mathbf{197.2 \pm 8.5}$ & \cellcolor{Gray} $\mathbf{200.7 \pm 9.2}$ &  $\mathbf{176.4 \pm 8.9}$ & $\mathbf{174.1 \pm 8.8}$ \\
NMF (8) &---  & $206.7 \pm 9.6$ & $205.0 \pm 9.3$ &  $186.1 \pm 8.9$ & $185.7 \pm 8.9$ \\
\textbf{Region + NMF (16)} & \textbf{---}  & $\mathbf{201.0 \pm 8.5}$ & $\mathbf{200.8 \pm 9.2}$ & \cellcolor{Gray}   $\mathbf{172.0 \pm 8.6}$ & \cellcolor{Gray2} $\mathbf{170.3 \pm 8.5}$ \\
NMF (16) & --- & $208.4 \pm 9.6$ & $205.2 \pm 9.3$ &  $179.3 \pm 8.4$ & $179.3 \pm 8.4$ \\
\textbf{Region + NMF (32)} & \textbf{---}   & $\mathbf{204.1 \pm 8.5}$ & $\mathbf{201.0 \pm 9.1}$ &  $\mathbf{173.2 \pm 8.7}$ & $\mathbf{171.5 \pm 8.6}$ \\
NMF (32) & --- & $211.1 \pm 9.6$ & $205.4 \pm 9.3$ &  $179.7 \pm 8.8$ & $179.5 \pm 8.8$ \\
NMF (64) & --- & $214.9 \pm 9.4$ & $205.6 \pm 9.3$ &  $181.9 \pm 8.6$ & $181.7 \pm 8.6$ \\
\hline
\end{tabular}
\end{adjustbox}
\caption{$E_T (\{\mathcal{I}\}) \pm \sigma_T(\{\mathcal{I}\})/\sqrt{T}$ (see Equation~\ref{eq:error}) for the fives strategies defined in Subsection~\ref{subsec:results} (``Benchmark", ``Bottom-up, ``Projection", ``Agregation" and ``Aggregation + Projection"), with benchmark predictions ($x^{\{\mathcal{I}\}}_t$ that are the same for all clusterings) made with General Additive Models and  aggregated with ML-Pol algorithm, for many segmentations (defined in Subsection~\ref{subsec:clustering}). The prediction error $E_T (\{\mathcal{I}\}) $ corresponds to the mean squared error (over the testing period) of the global consumption.
The dark gray area corresponds to the best prediction error of the table and the light gray area to the best one, for a given strategy.}
\label{tab3}
\end{table}
We now assess the impact of household segmentation on the quality of our predictions.
In view of the foregoing, we set the aggregation algorithm to ML-Pol and the benchmark forecasting method to the general additive model. 
As clusters change from a segmentation to another, the associated sets of nodes $\Gamma$ also change.  Errors related to $\Gamma$ or $\Gamma_0$ can therefore not be  compared from a segmentation to another.
We thus focus here on the global consumption (namely, we compute errors related to $\{\mathcal{I}\}$).
We compare our methods to a naive bottom-up strategy:
at each instance $t$, we forecast the global consumption $y^{\{\mathcal{I}\}}_t$ with the sum of local consumptions $\sum_{\gamma \in \Gamma_0} x_t^\gamma$ -- in lieu of the benchmark predictions $x^{\{\mathcal{I}\}}_t$.
Table~\ref{tab3} contains the prediction errors and the confidence bounds for the five strategies and for several household segmentations.
For the ``Bottom-up" strategy, the geographical clustering ``Region" provides the lowest prediction error, that are much better than the one of benchmark forecasts.  While when a single clustering based on household profiles or generated randomly is considered,  the benchmark forecasts $x^{\{\mathcal{I}\}}_t$ are more relevant -- in terms of mean squared error.
Thus, taking into account regional consumptions, which depend on local meteorological variables, improves prediction. 
In the same way, projection significantly improves the forecasts when the regions are taken into account.
Moreover,
for a fixed number of clusters -- for example, we compare  ``Fuel+Tariff", ``Random (4) and``NMF (4)" --
 the aggregation step seems more efficient when clusters present 
 different consumption profiles (see Figures~\ref{fig:conso_week_random}~-~\ref{fig:conso_week_nmf}).
Indeed, aggregation provides much better performance for ``NMF $(4)$" than for ``Random $(4)$". 
As we had anticipated, contrary to ``NMF" and ``Region",  clusterings ``Acorn" and ``Fuel + Tariff", that do not seem to detect consumption profiles, perform as well as ``Random".
When the number of clusters becomes too large, the performance of the strategy stagnates or even decreases. Typically for ``Random" or ``NMF", a number of clusters equals to $32$ or $64$ does not seem to improve the results compared to smaller numbers $4$, $8$ or $16$.
Another result is that aggregation and projection are robust to large number of clusters. Indeed, the performance are good for a sufficiently large number of clusters but does not decrease too much with the number of clusters -- either for `` Random" or ``NMF" clusterings.
Finally, our strategy ``Aggregation + Projection"
always outperforms the other four (``Bottom-up", ``Benchmark", ``Projection" and ``Aggregation") 
  and the  ``Region + NMF $(16)$" clustering  reaches the lowest prediction error.


\bibliography{ref}

\end{document}